\documentclass[english]{article}
\usepackage[T1]{fontenc}
\usepackage[latin9]{inputenc}
\usepackage{color}
\usepackage{babel}
\usepackage{refstyle}
\usepackage{graphicx}
\usepackage{caption}
\usepackage{subcaption}
\usepackage{amsmath}
\usepackage{amsthm}
\usepackage{amssymb}
\usepackage{xcolor}
\PassOptionsToPackage{sort&compress, numbers}{natbib}
\usepackage[unicode=true,pdfusetitle,
 bookmarks=true,bookmarksnumbered=false,bookmarksopen=false,
 breaklinks=false,pdfborder={0 0 1},backref=false,colorlinks=true]
 {hyperref}
\hypersetup{
 citecolor=blue,linkcolor=blue,urlcolor=blue}

\makeatletter

\AtBeginDocument{\providecommand\lemref[1]{\ref{lem:#1}}}
\AtBeginDocument{\providecommand\secref[1]{\ref{sec:#1}}}
\AtBeginDocument{\providecommand\thmref[1]{\ref{thm:#1}}}
\RS@ifundefined{subsecref}
  {\newref{subsec}{name = \RSsectxt}}
  {}
\RS@ifundefined{thmref}
  {\def\RSthmtxt{theorem~}\newref{thm}{name = \RSthmtxt}}
  {}
\RS@ifundefined{lemref}
  {\def\RSlemtxt{lemma~}\newref{lem}{name = \RSlemtxt}}
  {}

\theoremstyle{plain}
\newtheorem{thm}{\protect\theoremname}
\theoremstyle{plain}
\newtheorem{lem}[thm]{\protect\lemmaname}

\newref{thm}{name=Theorem~,Name=Theorem~}
\newref{lem}{name=Lemma~,Name=Lemma~}
\newref{sec}{name=Section~,Name=Section~}
\usepackage[final]{neurips_2022}
\makeatother
\providecommand{\lemmaname}{Lemma}
\providecommand{\theoremname}{Theorem}

\begin{document}
\title{Accelerating SGD for Highly Ill-Conditioned Huge-Scale Online Matrix
Completion}
\author{Gavin Zhang\\
University of Illinois at Urbana--Champaign\\
\texttt{jialun2@illinois.edu}~\\
\And Hong-Ming Chiu\\
University of Illinois at Urbana--Champaign\\
\texttt{hmchiu2@illinois.edu}\\
\And Richard Y. Zhang\\
University of Illinois at Urbana--Champaign\\
\texttt{ryz@illinois.edu}}

\maketitle
\global\long\def\inner#1#2{\langle#1,#2\rangle}%
\global\long\def\R{\mathbb{R}}%
\global\long\def\E{\mathbb{E}}%
\global\long\def\eqdef{\overset{\mathrm{def}}{=}}%
\global\long\def\rank{\operatorname{rank}}%
\begin{abstract}
The matrix completion problem seeks to recover a $d\times d$ ground
truth matrix of low rank $r\ll d$ from observations of its individual
elements. Real-world matrix completion is often a huge-scale optimization
problem, with $d$ so large that even the simplest full-dimension
vector operations with $O(d)$ time complexity become prohibitively
expensive. Stochastic gradient descent (SGD) is one of the few algorithms
capable of solving matrix completion on a huge scale, and can also
naturally handle streaming data over an evolving ground truth. Unfortunately,
SGD experiences a dramatic slow-down when the underlying ground truth
is ill-conditioned; it requires at least $O(\kappa\log(1/\epsilon))$
iterations to get $\epsilon$-close to ground truth matrix with condition
number $\kappa$. In this paper, we propose a preconditioned version
of SGD that preserves all the favorable practical qualities of SGD
for huge-scale online optimization while also making it agnostic to
$\kappa$. For a symmetric ground truth and the Root Mean Square Error
(RMSE) loss, we prove that the preconditioned SGD converges to $\epsilon$-accuracy
in $O(\log(1/\epsilon))$ iterations, with a rapid linear convergence
rate as if the ground truth were perfectly conditioned with $\kappa=1$.
In our experiments, we observe a similar acceleration for item-item 
collaborative filtering on the MovieLens25M dataset via a pair-wise ranking loss, 
with 100 million training pairs and 10 million testing pairs. 
{[}See supporting code at \url{https://github.com/Hong-Ming/ScaledSGD}.{]}
\end{abstract}
\section{Introduction}
The matrix completion problem seeks to recover an underlying $d\times d$
ground truth matrix $M$ of low rank $r\ll d$ from observations of
its individual matrix elements $M_{i,j}$. The problem appears most
prominently in the context of collaborative filtering and recommendation
system, but also numerous other applications. In this paper, we focus
on the \emph{symmetric} and \emph{positive semidefinite} variant of
the problem, in which the underlying matrix $M$ can be factored as
$M=ZZ^{T}$ where the factor matrix $Z$ is $d\times r$, though that
our methods have natural extensions to the nonsymmetric case. We note
that the symmetric positive semidefinite variant is actually far more
common in collaborative filtering, due to the prevalence of \emph{item-item}
models, which enjoy better data (most platforms contain several orders
of magnitude more users than items) and more stable recommendations
(the similarity between items tends to change slowly over time) than
user-user and user-item models. 

For the full-scale, online instances of matrix completion that arise
in real-world collaborative filtering, stochastic gradient descent
or SGD is the only viable algorithm for learning the underlying matrix
$M$. The basic idea is to formulate a candidate matrix of the form
$XX^{T}$ with respect to a learned factor matrix $X\in\mathbb{R}^{d\times r}$,
and to minimize a cost function of the form $\phi(XX^{T}-M)$. Earlier
work used the root mean square error (RMSE) loss $\|XX^{T}-M\|_{F}^{2}=\sum_{i,j}(XX^{T}-M)_{i,j}^{2}$,
though later work have focused on pairwise losses like the BPR~\citep{rendle2012bpr}
that optimize for ordering and therefore give better recommendations.
For the RMSE loss, the corresponding SGD iterations with (rescaled) learning rate $\alpha > 0$ reads
\begin{align}\label{eq:sgdupdate}
x_{i,+}=x_{i} - \alpha\cdot\left(x_{i}^{T}x_{j}-M_{ij}\right)x_{j},\quad x_{j,+}=x_{j} - \alpha\cdot\left(x_{i}^{T}x_{j}-M_{ij}\right)x_{i},
\end{align}
where $M_{ij}$ is the sampled $(i,j)$-th element of the ground truth
matrix $M$, and $x_{i},x_{j}$ and $x_{i,+},x_{j,+}$ denote the
$i$-th and $j$-th rows of the current iterate $X_{t}$ and new iterate
$X_{t+1}$. Pairwise losses like the BPR can be shown to have a similar
update equation over three rows of $X$~\citep{rendle2012bpr}. 
Given that only two or three rows of $X$ are accessed and updated
at any time, SGD is readily accessible to massive parallelization
and distributed computing. For very large values of $d$, the update
equation (\ref{eq:sgdupdate}) can be run by multiple workers in parallel
without locks, with vanishing probability of collision~\citep{recht2011hogwild}.
The blocks of $X$ that are more frequently accessed together can
be stored on the same node in a distributed memory system. 

Unfortunately, the convergence rate of SGD can sometimes be extremely slow.
One possible explanation, as many recent authors have pointed out~\citep{zheng2016convergence,tong2022scaling,zhang2021preconditioned,zhuo2021computational}, is that
matrix factorization models are very sensitive to ill-conditioning of the ground truth matrix $M$.
The number of SGD iterations grows at least linearly the condition number $\kappa$, which here is defined 
as the ratio between the largest and the $r$-th largest singular values of $M$.  Ill-conditioning causes particular concern because most real-world data are
ill-conditioned. In one widely cited study~\cite{kosinski2013private}, it was found that the dominant singular value accounts for only $\approx$80\% prediction
accuracy, with diversity of individual preferences making up the remainder ill-conditioned singular values. \citet{cloninger2014solving} notes that there are certain applications
of matrix completion that have condition numbers as high as $\kappa = 10^{15}$.

This paper is inspired by a recent full-batch gradient method called
ScaledGD~\citep{tong2021accelerating,tong2022scaling} and a closely
related algorithm PrecGD~\citep{zhang2021preconditioned} in which
 gradient descent is made immune to ill-conditioning in the ground truth by right-rescaling the full-batch gradient by the matrix $(X^{T}X)^{-1}$. Applying this
same strategy to the SGD update equation (\ref{eq:sgdupdate}) yields
the row-wise updates \begin{subequations} \label{eq:update}
\begin{equation}
x_{i,+}=x_{i}-\alpha\cdot\left(x_{i}^{T}x_{j}-M_{ij}\right)Px_{j},\quad x_{j,+}=x_{j}-\alpha\cdot\left(x_{i}^{T}x_{j}-M_{ij}\right)Px_{i},
\end{equation}
in which we \emph{precompute and cache} the preconditioner $P=(X^{T}X)^{-1}$
ahead of time\footnote{For an initialization, if the $d$ rows of $X_{0}$ are selected from
the unit Gaussian as in $x_{1},\dots,x_{d}\sim\mathcal{N}(0,\sigma^{2}I_{r})$,
then we can simply set $P_{0}=\sigma^{2}I$ without incurring the
$O(d)$ time needed in explicitly computing $P_{0}=(X_{0}^{T}X_{0})^{-1}$.}, and update it after the iteration as
\begin{equation}
P_{+}=(P^{-1}+x_{i,+}x_{i,+}^{T}+x_{j,+}x_{j,+}^{T}-x_{i}x_{i}^{T}-x_{j}x_{j}^{T})^{-1}
\end{equation}
\end{subequations}by making four calls to the Sherman--Morrison
rank-1 update formula 
\[
(P^{-1}+uu^{T})^{-1}=P-\frac{Puu^{T}P}{1+u^{T}Pu},\qquad(P^{-1}-uu^{T})^{-1}=P+\frac{Puu^{T}P}{1-u^{T}Pu}.
\]
This way, the rescaled update equations use just $O(r^{2})$ arithmetic
operations, which for modest values of $r$ is only marginally more
than the $O(r)$ cost of the unscaled update equations (\ref{eq:sgdupdate}).
Indeed, the nearest-neighbor algorithms inside most collaborative filters have exponential
complexity with respect to the latent dimensionality $r$, and so are often implemented with $r$ small enough for 
(\ref{eq:sgdupdate}) and (\ref{eq:update}) to have essentially the same runtime. 
Here, we observe that the rescaled update equations (\ref{eq:update})
preserve essentially all of the practical advantages of SGD for huge-scale,
online optimization: it can also be run by multiple workers in parallel
without locks, and it can also be easily implemented over distributed memory. The only minor difference is that separate copies of
$P$ should be maintained by each worker, and resynchronized once
differences grow large.  

\paragraph{Contributions}

In this paper, we provide a rigorous proof that the rescaled update
equations (\ref{eq:sgdupdate}), which we name ScaledSGD, become \emph{immune}
to the effects of ill-conditioning in the underlying ground truth
matrix. For symmetric matrix completion under the root mean squared
error (RMSE) loss function, regular SGD is known to have an iteration
count of $O(\kappa^{4}\cdot dr\log(d/\epsilon))$ within a local neighborhood
of the ground truth~\citep{jin2016provable}. This figure is optimal
in the dimension $d$, the rank $r$, and the final accuracy $\epsilon$,
but suboptimal by four exponents with respect to condition number
$\kappa$. In contrast, we prove for the same setting that ScaledSGD
attains an optimal convergence rate, converging to $\epsilon$-accuracy
in $O(dr\log(d/\epsilon))$ iterations for all values of the condition
number $\kappa$. In fact, our theoretical result predicts that ScaledSGD
converges as if the ground truth matrix is perfectly conditioned,
with a condition number of $\kappa=1$.

At first sight, it appears quite natural that applying the ScaledGD
preconditioner to SGD should result in accelerated convergence. However,
the core challenge of stochastic algorithms like SGD is that each
iteration can have substantial \emph{variance} that ``drown out''
the expected progress made in the iteration. In the case of ScaledSGD,
a rough analysis would suggest that the highly ill-conditioned preconditioner
should improve convergence in expectation, but at the cost of dramatically
worsening the variance. 

Surprisingly, we find in this paper that the specific scaling $(X^{T}X)^{-1}$
used in ScaledSGD not only does not worsen the variance, but in fact
\emph{improves} it. Our key insight and main theoretical contribution
is \lemref{gdecr}, which shows that the same mechanism that allows
ScaledGD to converge faster (compared to regular GD) also allows ScaledSGD
to enjoy reduced variance (compared to regular SGD). In fact, it is
this effect of variance reduction that is responsible for most ($\kappa^{3}$
out of $\kappa^{4}$) of our improvement over the previous state-of-the-art.
It turns out that a careful choice of preconditioner can be used as
a mechanism for \emph{variance reduction}, while at the same time
also fulfilling its usual, classical purpose, which is to accelerate
convergence in expectation. 

\paragraph{Related work}
Earlier work on matrix completion analyzed a convex relaxation of the original problem, showing that nuclear norm minimization can recover the ground truth from a few incoherent measurements \citep{candes2010matrix, candes2010power, recht2010guaranteed, srebro2005rank, negahban2012restricted}. This approach enjoys a near optimal sample complexity but incurs an $O(d^3)$ per-iteration computational cost, which is prohibitive for a even moderately large $d$. More recent work has focused more on a nonconvex formulation based on \citet{burer2003nonlinear}, which factors the optimization variable as $M=XX^T$ where $X\in \mathbb{R}^{d\times r}$ and applies a local search method such as alternating-minimization \cite{jain2013low, hardt2014fast, hardt2014understanding, sun2016guaranteed}, projected gradient descent \cite{chen2015fast, jain2015fast} and regular gradient descent \cite{tu2016low, bhojanapalli2016dropping, candes2015phase, ma2021implicit}. A separate line of work \citep{bhojanapalli2016global,li2019non,sun2018geometric,ge2016matrix,ge2017no,chen2017memory,zhang2019sharp,zhang2021sharp, josz2021nonsmooth} focused on global properties of nonconvex matrix recovery problems, showing that the problem has no spurious local minima if sampling operator satisfies certain regularity conditions such as incoherence or restricted isometry. 

The convergence rate of SGD has been well-studied for general classes of functions \cite{bassily2018exponential, vaswani2019fast, gower2019sgd, xie2020linear}. For matrix completion in particular, 
    \citet{jin2016provable} proved that SGD converges towards an $\epsilon$-accurate solution in $O(d\kappa^{4}\log(1/\epsilon))$ iterations where $\kappa$ is the condition number of $M$. Unfortunately, this quartic dependence on $\kappa$ makes SGD extremely slow and impractical for huge-scale applications. 
    
This dramatic slow down of gradient descent and its variants caused by ill-conditioning has become well-known in recent years. Several recent papers have proposed full-batch algorithms to overcome this issue \cite{tong2021accelerating, tong2021low,kummerle2021scalable}, but these methods cannot be used in the huge-scale optimization setting where $d$ is so large that even full-vector operations with $O(d)$ time complexity are too expensive. As a deterministic full-batch method, ScaledGD \cite{tong2021accelerating} requires a projection onto the set of incoherent matrices at every iteration in order to maintain rapid convergence. Instead our key finding here is that the stochasticity of SGD alone is enough to keep the iterates as incoherent as the ground truth, which allows for rapid progress to be made. The second-order method proposed in \cite{kummerle2021scalable} costs at least $O(d)$ per-iteration and has no straightforward stochastic analog. PrecGD~\cite{zhang2021preconditioned} only applies to matrices that satisfies matrices satisfying the restricted isometry property, which does not hold for matrix completion.

\section{Background: Linear convergence of SGD}

In our theoretical analysis, we restrict our attention to symmetric
matrix completion under the root mean squared error (RMSE) loss function.
Our goal is to solve the following nonconvex optimization 
\begin{equation}
\min_{X\in\R^{d\times r}}f(X)\eqdef\|XX^{T}-ZZ^{T}\|_{F}^{2}\qquad\text{ where }Z=[z_{1},z_{2},\dots,z_{n}]^{T}\in\R^{d\times r}\label{eq:prob}
\end{equation}
in which we assume that the $d\times d$ ground truth $ZZ^{T}\succeq0$
matrix is exactly rank-$r$, with a finite \emph{condition number}
\begin{equation}
\kappa\eqdef\lambda_{\max}(ZZ^{T})/\lambda_{r}(ZZ^{T})=\lambda_{\max}(Z^{T}Z)/\lambda_{\min}(Z^{T}Z)<\infty.\label{eq:cond}
\end{equation}
In order to be able to reconstruct $ZZ^{T}$ from a small number of
measurements, we will also need to assume that the ground truth has
small \emph{coherence}~\citep{candes2009exact}
\begin{equation}
\mu\eqdef\frac{d}{r}\cdot\max_{1\le i\le d}\|e_{i}^{T}Z(Z^{T}Z)^{-1/2}\|^{2}.\label{eq:incoh}
\end{equation}
Recall that $\mu$ takes on a value from $1$ to $d/r$, with the
smallest achieved by dense, orthonormal choices of $Z$ whose rows
all have magnitudes of $1/\sqrt{d}$, and the largest achieved by
a ground truth $ZZ^{T}$ containing a single nonzero element. Assuming
incoherence $\mu=O(1)$ with respect to $d$, it is a well-known result
that all $d^{2}$ matrix elements of $ZZ^{T}$ can be perfectly reconstructed
from just $O(dr\log d)$ random samples of its matrix elements~\citep{candes2010power,recht2011simpler}.

This paper considers solving (\ref{eq:prob}) in the huge-scale, online
optimization setting, in which individual matrix elements of the ground
truth $(ZZ^{T})_{i,j}=z_{i}^{T}z_{j}$ are revealed one-at-a-time,
uniformly at random with replacement, and that a current iterate $X$
is continuously updated to streaming data. We note that this is
a reasonably accurate model for how recommendation engines are tuned
to user preferences in practice, although the uniformity of random
sampling is admittedly an assumption made to ease theoretical analysis.
Define the stochastic gradient operator as 
\[
SG(X)=2d^{2}\cdot(x_{i}^{T}x_{j}-z_{i}^{T}z_{j})(e_{i}x_{j}^{T}+e_{j}x_{i}^{T})\quad\text{where }(i,j)\sim\mathrm{Unif}([d]\times[d]),
\]
where $x_{i},x_{j}\in\R^{r}$ are the $i$-th and $j$-th rows of
$X$, and the scaling $d^{2}$ is chosen that, over the randomness
of the sampled index $(i,j)$, we have exactly $\E[SG(X)]=\nabla f(X)$.
Then, the classical online SGD algorithm can be written as
\begin{equation}
X_{t+1}=X_{t}-\alpha SG(X_{t})\qquad\text{ where }\alpha>0.\tag{SGD}\label{sgd}
\end{equation}
Here, we observe that a single iteration of \ref{sgd} coincides with
full-batch gradient descent in expectation, as in $\E[X_{t+1}|X_{t}]=X_{t}-\alpha\nabla f(X_{t})$.
Therefore, assuming that bounded deviations and bounded variances,
it follows from standard arguments that the behavior of many iterations of SGD should
concentrate about that of full-batch gradient descent $X_{t+1}=X_{t}-\alpha\nabla f(X_{t})$.

Within a region sufficiently close to the ground truth, full-batch
gradient descent is well-known to converge at a linear rate to the
ground truth~\citep{zheng2015convergent,tu2016low}. Within this same region, \citet{jin2016provable} proved
that SGD also converges linearly. For an incoherent ground
truth with $\mu=O(1)$, they proved that SGD with an aggressive choice of step-size is able to
recover the ground truth to $\epsilon$-accuracy $O(\kappa^{4}dr\log(d/\epsilon))$
iterations, with each iteration costing $O(r)$ arithmetic operations
and selecting 1 random sample. This iteration count is optimal with respect to $d,$ $r,$ and $1/\epsilon,$
although its dependence on $\kappa$ is a cubic factor (i.e. a factor of $\kappa^{3}$)
worse than full-batch gradient descent's figure of $O(\kappa\log(1/\epsilon))$,
which is itself already quite bad, given that $\kappa$ in practice
can readily take on values of $10^{3}$ to $10^{6}$.

\begin{thm}[\citet*{jin2016provable}]
\label{thm:jin}For $Z\in\R^{d\times r}$ with $\sigma_{\max}(Z)=1$
and $f(X)=\|XX^{T}-ZZ^{T}\|_{F}^{2}$ and $h_{i}(X)=\|e_{i}^{T}X\|^{2}$,
define the following
\[
f_{\max}\eqdef\left(\frac{1}{10\kappa}\right)^{2},\qquad h_{\max}\eqdef20\cdot\kappa^{2}\cdot\frac{\mu r}{d}.
\]
For an initial point $X_{0}\in\R^{d\times r}$ that satifies $f(X_{0})\le\frac{1}{2}f_{\max}$
and $\max_{i}h_{i}(X_{0})\le\frac{1}{2}h_{\max}$, there exists some
constant $c$ such that for any learning rate $\alpha<c\cdot(\kappa\cdot h_{\max}\cdot d^{2}\log d)^{-1}$,
with probability at least $1-T/d^{10}$, we will have for all $t\le T$
iterations of \ref{sgd} that
\[
f(X_{t})\le\left(1-\frac{\alpha}{2\cdot\kappa}\right)^{t}\cdot f_{\max},\qquad\max_{i}h_{i}(X_{t})\le h_{\max}.
\]
\end{thm}

The reason for \thmref{jin}'s additional $\kappa^{3}$ dependence
beyond full-batch gradient descent is due to its need to maintain
\emph{incoherence} in its iterates. Using standard techniques on martingale
concentration, one can readily show that SGD replicates a single
iteration of full-batch gradient descent over an epoch of $d^{2}$
iterations.  This results in an iteration count $O(\kappa\cdot d^{2}\log(1/\epsilon))$
with an optimal dependence on $\kappa$, but the entire matrix is
already fully observed after collecting $d^{2}$ samples. Instead,
\citet{jin2016provable} noted that the \emph{variance} of SGD
iterations is controlled by the step-size $\alpha$ times the maximum
coherence $\mu_{X}=\frac{d}{r}\cdot\max_{i,t}\|e_{i}^{T}X_{t}\|^{2}$
over the iterates $X_{t},X_{t-1},\dots,X_{0}$. If the iterates can
be kept incoherent with $\mu_{X}=O(1)$, then SGD with a more
aggressive step-size will reproduce an
iteration of full-batch gradient descent after an epoch of just $O(dr\log d)$
iterations. 

The main finding in \citet{jin2016provable}'s proof of \thmref{jin}
is that the stochasticity of SGD is enough to keep the iterates
incoherent. This contrasts with full-batch methods at the time, which required an added regularizer~\citep{sun2016guaranteed,ge2016matrix,chi2019nonconvex}
or an explicit projection step~\citep{tong2021accelerating}. 
(As pointed out by a reviewer, it was later shown by \citet{ma2018implicit} that full-batch gradient descent 
is also able to maintain incoherence without a regularizer nor a projection.)
Unfortunately,
maintaining incoherence requires shrinking the step-size by a factor
of $\kappa$, and the actual value of $\mu_{X}$ that results is also
a factor of $\kappa^{2}$ worse than the original coherence $\mu$
of the ground truth $Z$. The resulting iteration count $O(\kappa^{4}\cdot dr\log(d/\epsilon))$
is made optimal with respect to $d,$ $r,$ and $1/\epsilon$, but
only at the cost of worsening its the dependence on the condition
number $\kappa$ by another \emph{three exponents}. 

\global\long\def\spur{\mathrm{spur}}%
Finally, the quality of the initial point $X_{0}$ also has a dependence
on the condition number $\kappa$. In order to guarantee linear convergence,
\thmref{jin} requires $X_{0}$ to lie in the neighborhood $\|X_{0}X_{0}^{T}-ZZ^{T}\|_{F}<\lambda_{\min}(Z^{T}Z)=O(\kappa^{-1})$.
This dependence on $\kappa$ is optimal, because full-batch gradient
descent must lose its ability to converge linearly in the limit $\kappa\to\infty$~\citep{zhuo2021computational,zhang2021preconditioned}.
However, the leading constant can be very pessmistic, because the
theorem must formally exclude spurious critical points $X_{\spur}$
that have $\nabla f(X_{\spur})=0$ but $f(X_{\spur})>0$ in order
to be provably correct. In practice, it is commonly observed that
SGD converges globally, starting from an arbitrary, possibly
random initialization~\citep{ge2016matrix}, at a linear rate that
is consistent with local convergence theorems like \thmref{jin}.
It is now commonly argued that gradient methods can escape saddle
points with high probability~\citep{lee2019first}, and so their
performance is primarily dictated by local convergence behavior~\citep{jin2017escape,jin2021nonconvex}.

\section{\label{sec:main} Proposed algorithm and main result}

Inspired by a recent full-batch gradient method called
ScaledGD~\citep{tong2021accelerating,tong2022scaling} and a closely
related algorithm PrecGD~\citep{zhang2021preconditioned}, we proposed the following algorithm
\begin{equation}
X_{t+1}=X_{t}-\alpha SG(X_{t})(X_{t}^{T}X_{t})^{-1}\qquad\text{ where }\alpha>0.\tag{ScaledSGD}\label{scaledsgd}
\end{equation}
As we mentioned in the introduction, the preconditioner $P=(X^T X)^{-1}$ can be precomputed and cached in a practical 
implementation, and afterwards efficiently updated using the Sherman--Morrison formula. The per-iteration cost of \ref{scaledsgd}
is $O(r^{2})$ arithmetic operations and 1 random sample, which for
modest values of $r$ is only marginally more than the cost of \ref{sgd}.

Our main result in this paper is that, with a region sufficiently
close to the ground truth, this simple rescaling allows ScaledSGD
to converge linearly to $\epsilon$-accuracy $O(dr\log(d/\epsilon))$
iterations, with no further dependence on the condition number $\kappa$.
This iteration count is optimal with respect to $d,$ $r,$ $1/\epsilon,$
and $\kappa$, and in fact matches SGD with a\emph{ perfectly
conditioned} ground truth $\kappa=1$. In our numerical experiments,
we observe that ScaledSGD converges globally from a random
initialization at the same rate as SGD as if $\kappa=1$.
\begin{thm}[Main]
\label{thm:main}For $Z\in\R^{d\times r}$ with $\sigma_{\max}(Z)=1$
and $f(X)=\|XX^{T}-ZZ^{T}\|_{F}^{2}$ and $g_{i}(X)=e_{i}^{T}X(X^{T}X)^{-1}X^{T}e_{i}$,
select a radius $\rho<1/2$ and set
\[
f_{\max}\eqdef\left(\frac{\rho}{\kappa}\right)^{2},\qquad g_{\max}\eqdef\frac{2^{4}}{(1-2\rho)^{2}}\cdot\frac{\mu r}{d}.
\]
For an initial point $X_{0}\in\R^{d\times r}$ that satifies $f(X_{0})\le\frac{1}{2}f_{\max}$
and $\max_{i}g_{i}(X_{0})\le\frac{1}{2}g_{\max}$, there exists some
constant $c$ such that for any learning rate $\alpha<c\cdot[(g_{\max}+\rho)\cdot d^{2}\log d]^{-1}$,
with probability at least $1-T/d^{10}$, we will have for all $t\le T$
iterations of \ref{scaledsgd} that:
\[
f(X_{t})\le\left(1-\frac{\alpha}{2}\right)^{t}\cdot f_{\max},\qquad\max_{i}g_{i}(X_{t})\le g_{\max}.
\]
\end{thm}
\thmref{main} eliminates all dependencies on the condition number
$\kappa$ in \thmref{jin} except for the quality of the initial point,
which we had already noted earlier as being optimal. Our main finding
is that it is possible to maintain incoherence while making aggressive
step-sizes towards a highly ill-conditioned ground truth $ZZ^{T}$.
In fact, \thmref{main} says that, with high probability, the maximum
coherence $\mu_{X}$ over of any iterate $X_{t}$ will only be a \emph{mild
constant factor} of $\approx 16$ times worse than the
coherence $\mu$ of the ground truth $ZZ^{T}$. This is particularly
surprising in view of the fact that every iteration of ScaledSGD
involves inverting a potentially highly ill-conditioned matrix $(X^{T}X)^{-1}$.
In contrast, even without inverting matrices, \thmref{jin} says
that \ref{sgd} is only able to keep   $\mu_{X}$
within a factor of $\kappa^{2}$ of $\mu$, and only by shrinking
the step-size $\alpha$ by another factor of $\kappa$. 

However, the price we pay for maintaining incoherence is that the
quality of the initial point $X_{0}$ now gains a dependence on dimension
$d$, in addition to the condition number $\kappa$. In order to guarantee fast
linear convergence independent of $\kappa$, \thmref{main} requires
$X_{0}$ to lie in the neighborhood $\|X_{0}X_{0}^{T}-ZZ^{T}\|_{F}<\mu r\lambda_{\min}(Z^{T}Z)/d=(\kappa d)^{-1}$,
so that $\rho$ can be set to be the same order of magnitude as $g_{\max}$.
In essence, the ``effective'' condition number of the ground truth
has been worsened by another factor of $d$. This shrinks the size
of our local neighborhood by a factor of $d$, but has no impact on
the convergence rate of the resulting iterations.

In the limit that $\kappa\to\infty$ and the search rank $r$ becomes
\emph{overparameterized} with respect to the true rank $r^{\star}<r$
of $ZZ^{T}$, both full-batch gradient descent and SGD slows
down to a \emph{sublinear} convergence rate, in theory and in practice~\citep{zhuo2021computational,zhang2021preconditioned}.
While \thmref{main} is no longer applicable, we observe in our numerical
experiments that ScaledSGD nevertheless maintains its fast
linear convergence rate as if $\kappa=1$. Following PrecGD~\citep{zhang2021preconditioned},
we believe that introducing a small identity perturbation to the scaling
matrix of ScaledSGD, as in $(X^{T}X+\eta I)^{-1}$ for some
$\eta\approx\sqrt{f(X)}$, should be enough to rigorously extend \thmref{main}
to the overparameterized regime. We leave this extension as future
work. 

\section{Key ideas for the proof}

We begin by explaining the mechanism by which \ref{sgd} slows down
when converging towards an ill-conditioned ground truth. Recall that
\[
\E[SG(X)]=\E[2d^{2}\cdot(XX^{T}-ZZ^{T})_{i,j}\cdot(e_{i}e_{j}^{T}+e_{j}e_{i}^{T})X]=4(XX^{T}-ZZ^{T})X=\nabla f(X).
\]
As $XX^{T}$ converges towards an ill-conditioned ground truth $ZZ^{T}$,
the factor matrix $X$ must become progressively ill-conditioned,
with
\[
\lambda_{\min}(X^{T}X)=\lambda_{r}(XX^{T})\le\lambda_{r}(ZZ^{T})+\|XX^{T}-ZZ^{T}\|_{F}\le\frac{1+\rho}{\kappa}.
\]
Therefore, it is possible for components of the error vector $XX^{T}-ZZ^{T}$
to become ``invisible'' by aligning within the ill-conditioned subspaces
of $X$. As SGD progresses towards the solution, these ill-conditioned
subspaces of $X$ become the slowest components of the error vector
to converge to zero. On the other hand, the maximum step-size that
can be taken is controlled by the most well-conditioned subspaces
of $X$. A simple idea, therefore, is to rescale the ill-conditioned components
of the gradient $\nabla f(X)$ in order to make the ill-conditioned
subspaces of $X$ more ``visible''.

More concretely, define the local norm of the gradient as $\|\nabla f(X)\|_X = \|\nabla f(X) (X^TX)^{1/2}\|_F$ and its corresponding dual norm as $\|\nabla f(X)\|^*_X = \|\nabla f(X) (X^TX)^{-1/2}\|_F$. It has long been known (see e.g. \cite{zheng2015convergent,tu2016low}) that rescaling the gradient yields 
\[
\|\nabla f(X)\|_{X}^{*} \eqdef\|4(XX^{T}-ZZ^{T})X(X^{T}X)^{-1/2}\|_{F}=4\cos\theta\cdot\|XX^{T}-ZZ^{T}\|_{F},
\]
where $\theta$ is the angle between the error vector $XX^T-ZZ^T$ and the linear subspace $\{XY^T+YX^T: Y\in\mathbb{R}^{d\times r}\}$.
This insight immediately suggests an iteration like $X_{+}=X-\alpha\nabla f(X)(X^{T}X)^{-1}$. In fact, the gradients of $f$ have some Lipschitz constant $L$, so 
\begin{align*}
f(X_{+}) & \le f(X)-\alpha\inner{\nabla f(X)}{\nabla f(X)(X^{T}X)^{-1}}+\frac{L}{2}\alpha^{2}\|\nabla f(X)(X^{T}X)^{-1}\|_{F}^{2},\\
 & \le f(X)-\alpha(\|\nabla f(X)\|_{X}^{*})^{2}+\frac{L_{X}}{2}\alpha^{2}(\|\nabla f(X)\|_{X}^{*})^{2},\\
 & \le\left[1-\alpha\cdot8\cos^{2}\theta\right]f(X)\qquad\text{ for }\alpha\le1/L_{X}.
\end{align*}
However, a naive analysis finds that $L_{X}=L/\lambda_{\min}(X^{T}X)\approx L\cdot\kappa$,
and this causes the step-size to shrink by a factor of $\kappa$.
The main motivating insight behind ScaledGD~\citep{tong2021accelerating,tong2022scaling} and later PrecGD~\citep{zhang2021preconditioned} is that,
with a finer analysis, it is possible to prove Lipschitz continuity
under a local change of norm.
\begin{lem}[Function descent]
\label{lem:fdecr}Let $X,Z\in\R^{n\times r}$ satisfy $\|XX^{T}-ZZ^{T}\|_{F}\le\rho\lambda_{\min}(Z^{T}Z)$
where $\rho<1/2$. Then, the function $f(X)=\|XX^{T}-ZZ^{T}\|_{F}^{2}$
satisfies
\[
f(X+V)\le f(X)+\inner{\nabla f(X)}V+\frac{L_{X}}{2}\|V\|_{X}^{2},\qquad (\|\nabla f(X)\|_{X}^{*})^{2}\ge 13 \cdot f(X)
\]
for all $\|V\|_{X}\le C\cdot\sqrt{f(X)}$ with $L_{X}=6+8C+2C^{2}=O(1+C^{2})$. 
\end{lem}

This same idea can be ``stochastified'' in a straightforward manner.
Conditioning on the current iterate $X$, then the new iterate $X_{+}=X-\alpha SG(X)(X^{T}X)^{-1}$
has expectation
\[
\E[f(X_{+})]\le f(X)-\alpha\inner{\nabla f(X)}{\E[SG(X)(X^{T}X)^{-1}]}+\alpha\frac{L_{X}}{2}\E[(\|SG(X)\|_{X}^{*})^{2}].
\]
The linear term evaluates  as $\E[SG(X)(X^{T}X)^{-1}]=\nabla f(X)(X^{T}X)^{-1}$,
while the quadratic term is
\begin{align*}
\E[(\|SG(X)\|_{X}^{*})^{2}] \le \sum_{i,j}4d^{2}\cdot(XX^{T}-ZZ^{T})_{i,j}^{2}\cdot4\max_{i}(\|e_{i}^{T}X\|_{X}^{*})^{2}=16\cdot f(X)\cdot \max_i g_i(X),
\end{align*}
where $g_i(X)=e_i^T X(X^T X)^{-1} X^{T} e_i = (\|e_i ^T X\|_X^*)^2.$ Combined, we obtain geometric convergence
\begin{equation}
\label{fdescent}
\E[f(X_{+})]\le\left(1-\alpha\cdot8\cos^{2}\theta\right)f(X)\qquad\text{ for }\alpha=O(g_{\max}^{-1} \cdot d^{-2}).
\end{equation}
We see that the step-size depends crucially on the incoherence $g_i(X) \le g_{\max}$
of the current iterate. If the current iterate $X$ is incoherent
with $g_{\max}=O(1/d)$, then a step-size of $\alpha=O(1/d)$ is possible,
resulting in convergence in $O(dr\log(d/\epsilon))$ iterations, which can be shown using standard martingale techniques \cite{jin2016provable}. But
if the current iterate is $g_{\max}=O(1)$, then only a step-size of
$\alpha=O(1/d^{2})$ is possible, which forces us to compute $d^{2}$
iterations, thereby obviating the need to complete the matrix in the first place. 

Therefore, in order for prove rapid linear
convergence, we need to additionally show that with high
probability, the coherence $g_{k}(X){=}(\|e_{k}^{T}X\|_{X}^{*})^{2}$
remains $O(1)$ throughout ScaledGD iterations. This is the most challenging
part of our proof. Previous methods that applied a similar scaling
to full-batch GD~\citep{tong2021accelerating} required an explicit
projection onto the set of incoherent matrices at each iteration. Applying
a similar projection to ScaledSGD will take $O(d)$ time, which destroys
the scalability of our method. On the other hand, \citet{jin2016provable}
showed that the randomness in SGD is enough to keep the coherence
of the iterates within a factor of $\kappa^{2}$ times worse than
the coherence of the ground truth, and only by a step-size of at
most $\alpha=O(\kappa^{-1})$. 

Surprisingly, here we show that the randomness in ScaledSGD is enough
to keep the coherence of the iterates with a \emph{constant factor}
of the coherence the ground truth, using a step-size with no dependence
on $\kappa$. The following key lemma is the crucial insight of our
proof. First, it says that function $g_{k}(X)$ satisfies a ``descent
lemma'' with respect to the local norm $\|\cdot\|_{X}^{*}$. Second,
and much more importantly, it says that descending $g_{k}(X)$ along
the scaled gradient direction $\nabla f(X)(X^{T}X)^{-1}$ incurs a
linear decrement $\frac{1-2\rho}{1-\rho}g_{k}(X)$ with no dependence
of the condition number $\kappa$. This is in direct analogy to the
function value decrement in (\ref{fdescent}), which has no dependence
on $\kappa$, and in direct contrast to the proof of \citet{jin2016provable},
which is only able to achieve a decrement of $(8/ \kappa) g_{k}(X)$ due to the lack of rescaling by $(X^T X)^{-1}$. 

\begin{lem}[Coherence descent]
\label{lem:gdecr} Let $g_{k}(X)=e_{k}^{T}X(X^{T}X)^{-1}X^Te_{k}$. Under the same conditions as \lemref{fdecr}, we have
\begin{gather*}
g_{k}(X+V)\le g_{k}(X)+\inner V{\nabla g_{k}(X)}+\frac{5(\|V\|_{X}^{*})^{2}}{1-2\|V\|_{X}^{*}},\\
\inner{\nabla g_{k}(X)}{\nabla f(X)(X^{T}X)^{-1}}\ge\left[\frac{1-2\rho}{1-\rho}g_{k}(X)-\frac{1}{1-\rho}\sqrt{g_{k}(X)g_{k}(Z)}\right].
\end{gather*}
\end{lem}

Conditioning on $X$, we have for the search direction $V=SG(X)(X^{T}X)^{-1}$ and $X_{+}=X+V$
\begin{align}
\E[g_{k}(X_{+})] & \le g_{k}(X)-\alpha\inner{\nabla g_{k}(X)}{\E[V]}+\alpha^{2}\cdot\E\left[\frac{(\|V\|_{X}^{*})^{2}}{1-2\|V\|_{X}^{*}}\right]\nonumber \\
 & \le\left(1-\frac{1-2\rho}{1-\rho}\alpha\right)g_{k}(X)+\alpha\cdot\frac{1}{1-\rho}\cdot\sqrt{g_{k}(X)g_{k}(Z)}+\alpha^{2}\cdot\frac{\E\left[(\|V\|_{X}^{*})^{2}\right]}{1-2\|V\|_{X}^{*}}\nonumber \\
 & \le\left(1-\frac{1-2\rho}{1-\rho}\alpha\right)g_{k}(X)+\alpha\cdot\frac{\sqrt{\mu/g_{\max}}}{1-\rho}\cdot g_{\max}+\alpha^{2}\cdot\frac{O(d^{2}\cdot g_{\max}\cdot\rho^{2})}{1-O(g_{\max}^{1/2}\cdot\rho)}\nonumber \\
 & \le\left(1-\zeta\alpha\right)g_{k}(X)+\alpha\cdot\frac{\zeta}{2}g_{\max}\qquad\text{ for }\alpha=O(\rho^{-1}d^{-2}).\label{eq:gdecr}
\end{align}
It then follows that $g_{k}(X_{+})$ converges geometrically towards
$\frac{1}{2}g_{\max}$ in expectation, with a convergence rate $(1-\zeta\alpha)$
that is independent of the condition number $\kappa$:
\[
\E[g_{k}(X_{+})-\frac{1}{2}g_{\max}]\le\left[\left(1-\zeta\alpha\right)g_{k}(X)+\alpha\cdot\frac{\zeta}{2}g_{\max}\right]-\frac{1}{2}g_{\max}\le\left(1-\zeta\alpha\right)\left[g_{k}(X)-\frac{1}{2}g_{\max}\right].
\]
The proof of \thmref{main} then follows from standard techniques, by making the two decrement conditions (\ref{fdescent}) and (\ref{eq:gdecr}) into supermartingales and applying a standard concentration inequality. We defer the rigorous proof to appendix \ref{app:proof}.

\section{Experimental validation}
In this section we compare the practical performance of \ref{scaledsgd} and \ref{sgd} for the RMSE loss function in \thmref{main} and two real-world loss functions: the pairwise RMSE loss used to complete Euclidean Distance Matrices (EDM) in wireless communication networks; and the Bayesian Personalized Ranking (BRP) loss used to generate personalized item recommendation in collaborative filtering. In each case, ScaledSGD remains highly efficient since it only updates two or three rows at a time, and the preconditioner $P$ can be computed through low-rank updates, for a per-iteration cost of $O(r^2).$ All of our experiments use random Gaussian initializations and an initial $P=\sigma^2I$. To be able to accurately measure and report the effects of ill-conditioning on ScaledSGD and SGD, we focus on small-scale synthetic datasets in the first two experiments, for which the ground truth is explicitly known, and where the condition numbers can be finely controlled. In addition, to gauge the scalability of ScaledSGD on huge-scale real-world datasets, in the third experiment, we apply ScaledSGD to generate personalized item recommendation using MovieLens25M dataset \cite{movielens}, for which the underlying item-item matrix has more than 62,000 items and 100 million pairwise samples are used during training. (Due to space constraints, we defer the details on the experimental setup, mathematical formulations, and the actual update equations to Appendix \ref{app:exp}.) The code for all experiments are available at \url{https://github.com/Hong-Ming/ScaledSGD}.

\paragraph{Matrix completion with RMSE loss.} The problem formulation is discussed in Section~\ref{sec:main}. Figure~\ref{fig:rmse} plots the error $f(X) = \|XX^T-M\|_F^2$ as the number of epochs increases. As expected, in the well-conditioned case, both ScaledSGD and SGD converges to machine error at roughly the same linear rate. However, in the ill-conditioned case, SGD slows down significantly while ScaledSGD converges at almost exactly the same rate as in the well-conditioned case. 
\begin{figure}[h!]
    \centering
    \begin{subfigure}{0.5\textwidth}
      \centering
      \includegraphics[width=\linewidth]{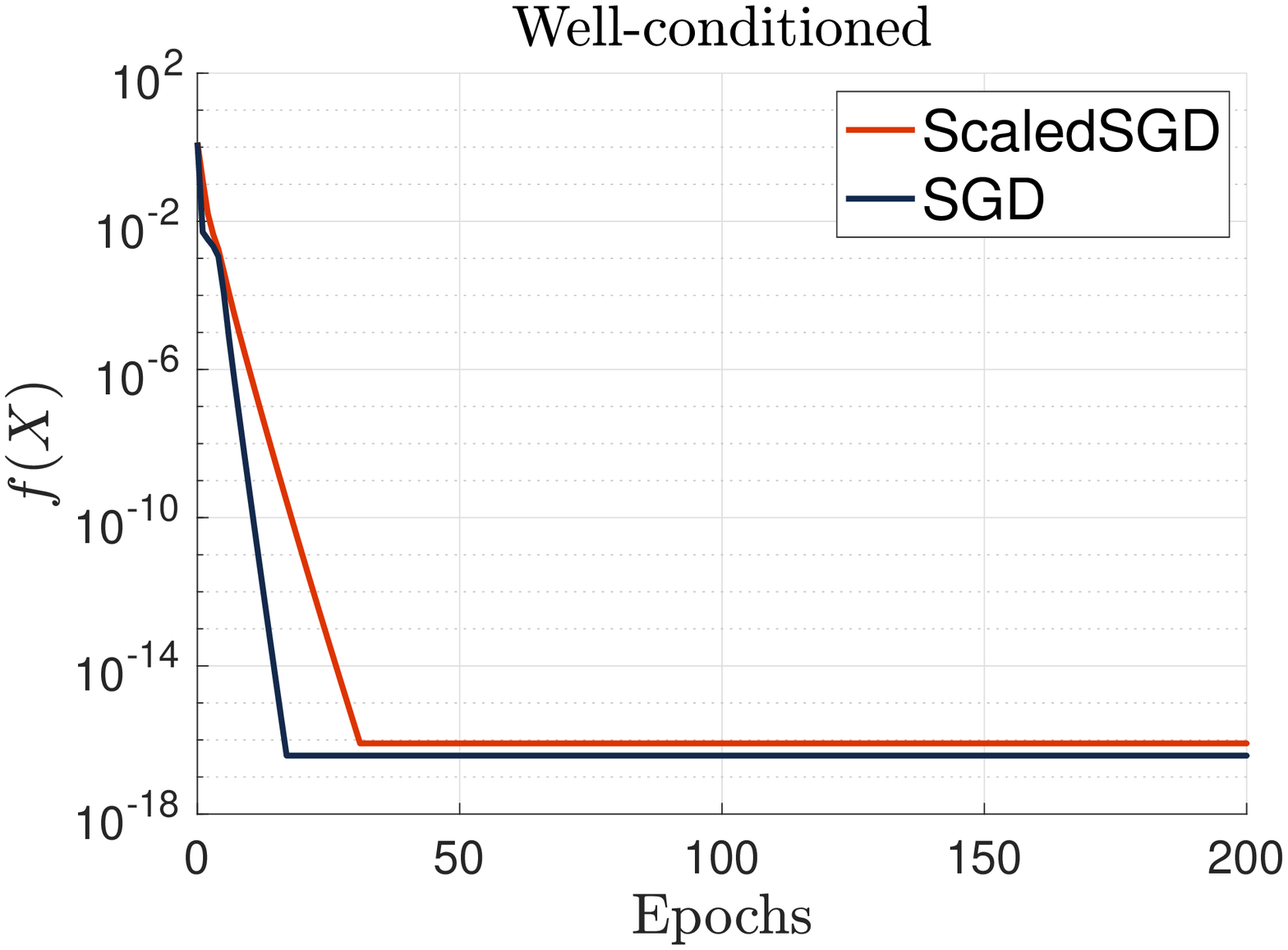}
    \end{subfigure}%
    \begin{subfigure}{0.5\textwidth}
      \centering
      \includegraphics[width=\linewidth]{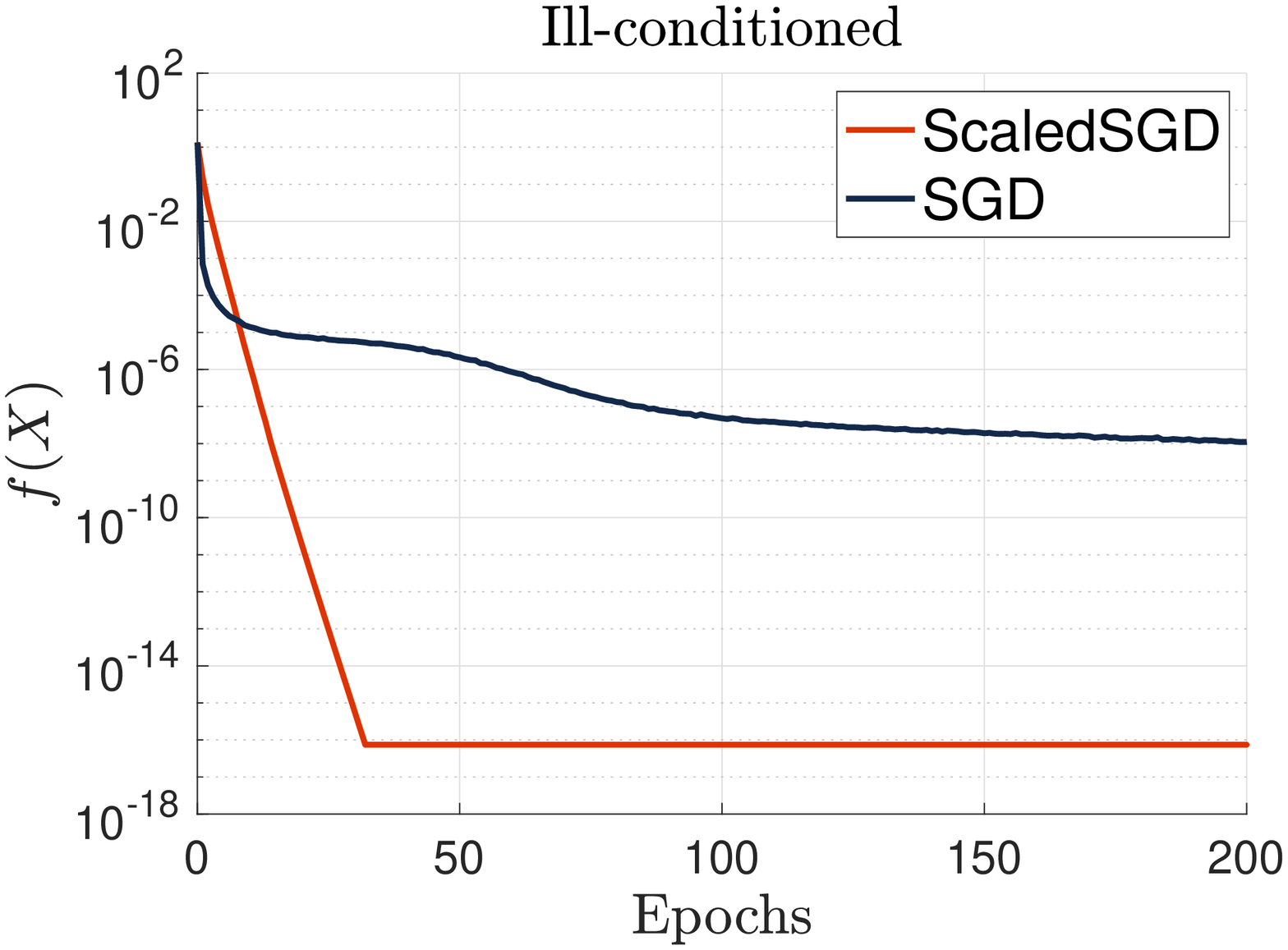}
    \end{subfigure}
    \caption{\textbf{Matrix Completion with RMSE loss.} We compare the convergence rate of \ref{scaledsgd} and \ref{sgd} for a well-conditioned and ill-conditioned ground truth matrix of size $30\times 30$ and rank 3. (\textbf{Left}) Well-conditoned $M$, $\kappa(M) = 1$. Step-size $\alpha=0.3$. Both ScaledSGD and SGD converges quickly to the ground truth. (\textbf{Right}) Ill-conditoned $M$, $\kappa(M) = 10^4$. Step-size $\alpha=0.3$. SGD stagnates while ScaledSGD retains the same convergence rate as the well-conditioned case.}
    \label{fig:rmse}
\end{figure}

\paragraph{Euclidean distance matrix (EDM) completion.} The Euclidean distance matrix (EDM) is a matrix of pairwise distance between $d$ points in Euclidean space \citep{dokmanic2015euclidean}. In applications such as wireless sensor networks, estimation of unknown distances, i.e., completing the EDM is often required. We emphasize that this loss function is a pairwise loss, meaning that each measurement indexes multiple elements of the ground truth matrix.

To demonstrate the efficacy of ScaledSGD, we conduct two experiments where $D$ is well-conditioned and ill-conditioned respectively: \textbf{Experiment 1.} We uniformly sample $30$ points in a cube center at origin with side length $2$, and use them to compute the ground truth EDM $D$. In this case, each row $x_i\in\mathbb R^3$ corresponds to the coordinates of the $i$-th sample. The corresponding matrix $X\in\mathbb R^{30\times 3}$ is well-conditioned because of the uniform sampling. \textbf{Experiment 2.} The ground truth EDM is generated with $25$ samples lie in the same cube in experiment 1, and $5$ samples lie far away from the the cube. These five outliers make the corresponding $X$ become ill-conditioned.

 \begin{figure}[h!]
    \centering
    \begin{subfigure}{0.5\textwidth}
      \centering
      \includegraphics[width=\linewidth]{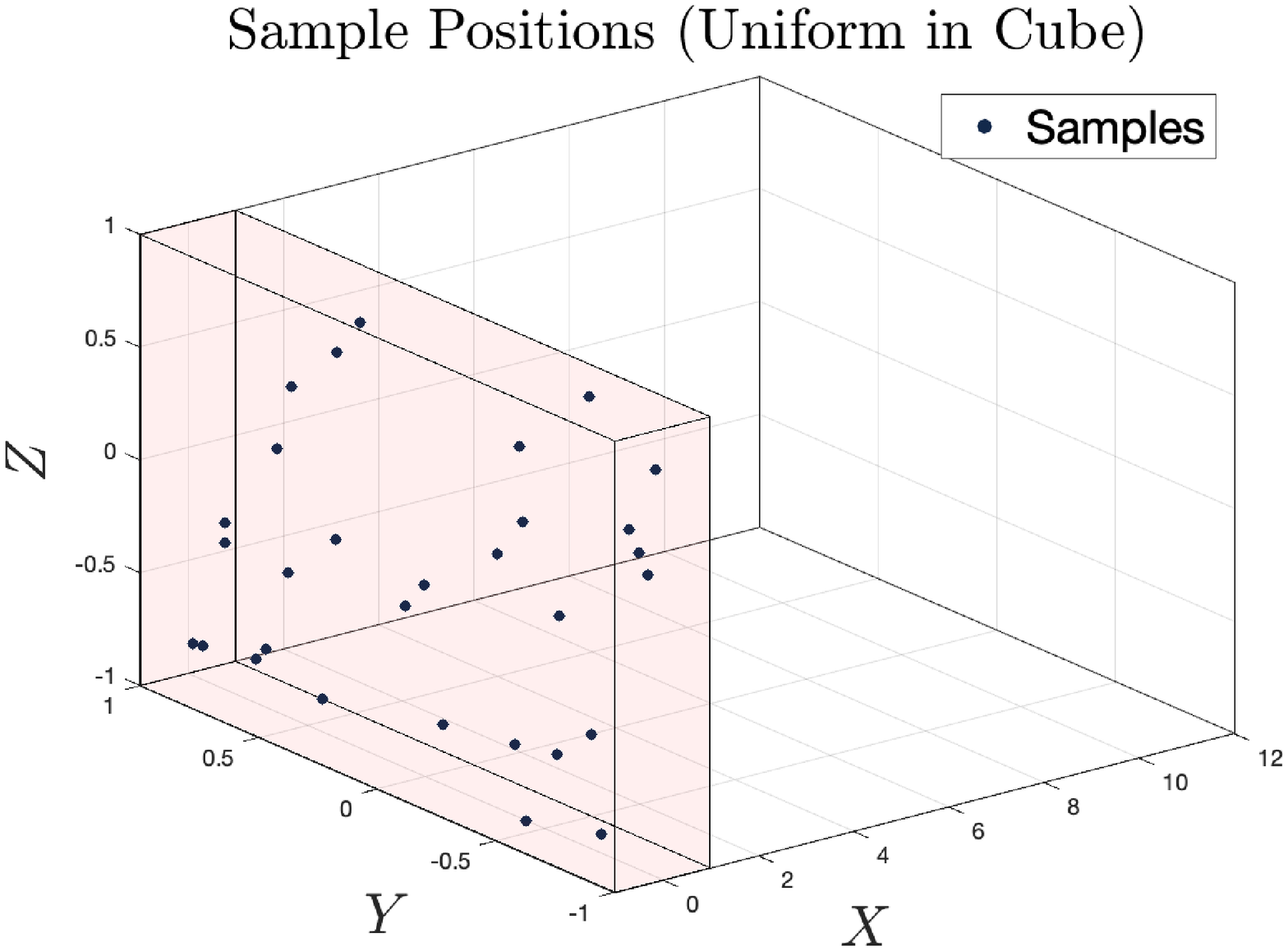}
    \end{subfigure}%
    \begin{subfigure}{0.5\textwidth}
      \centering
      \includegraphics[width=\linewidth]{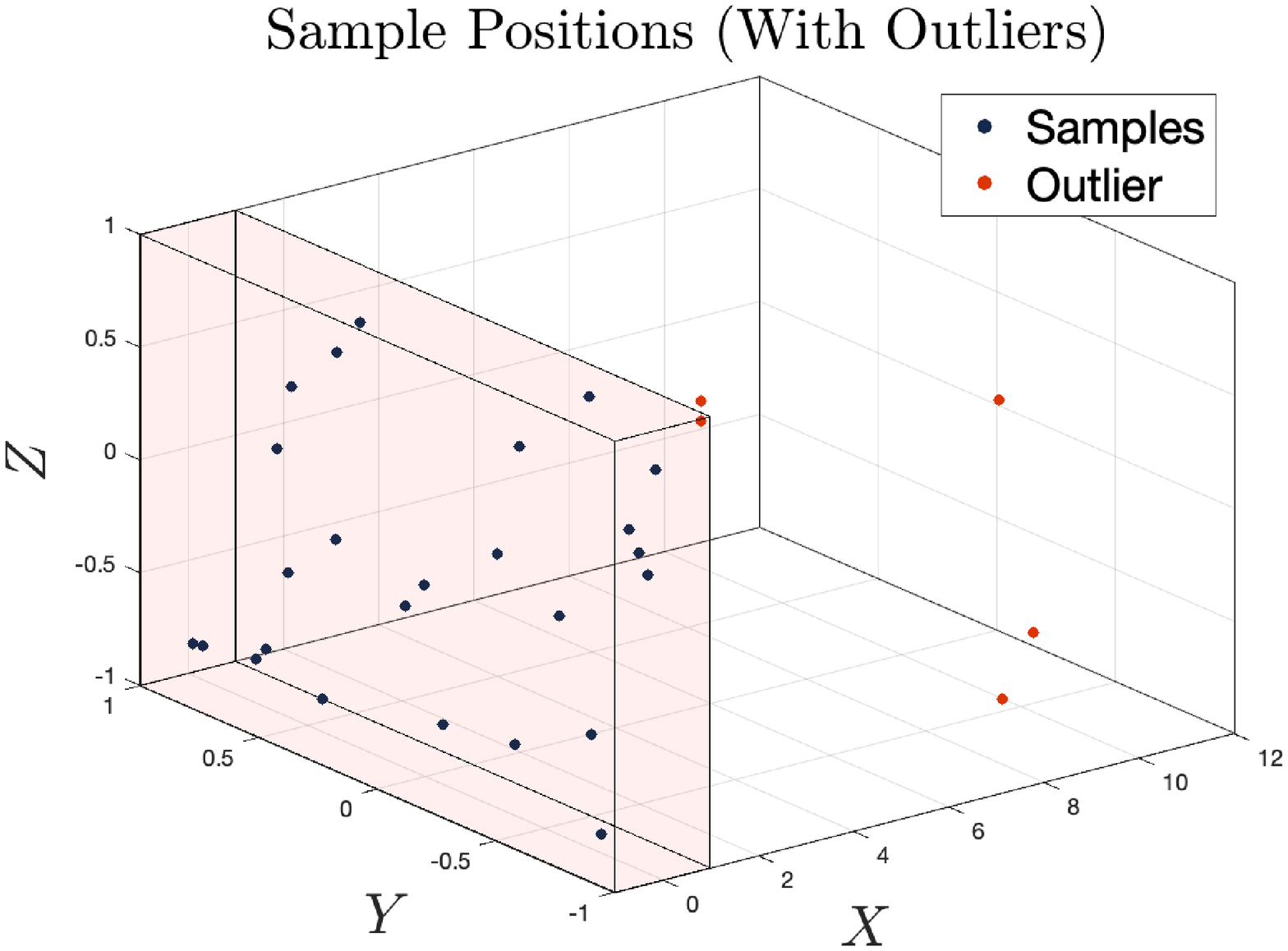}
    \end{subfigure}
    \begin{subfigure}{0.5\textwidth}
      \centering
      \includegraphics[width=\linewidth]{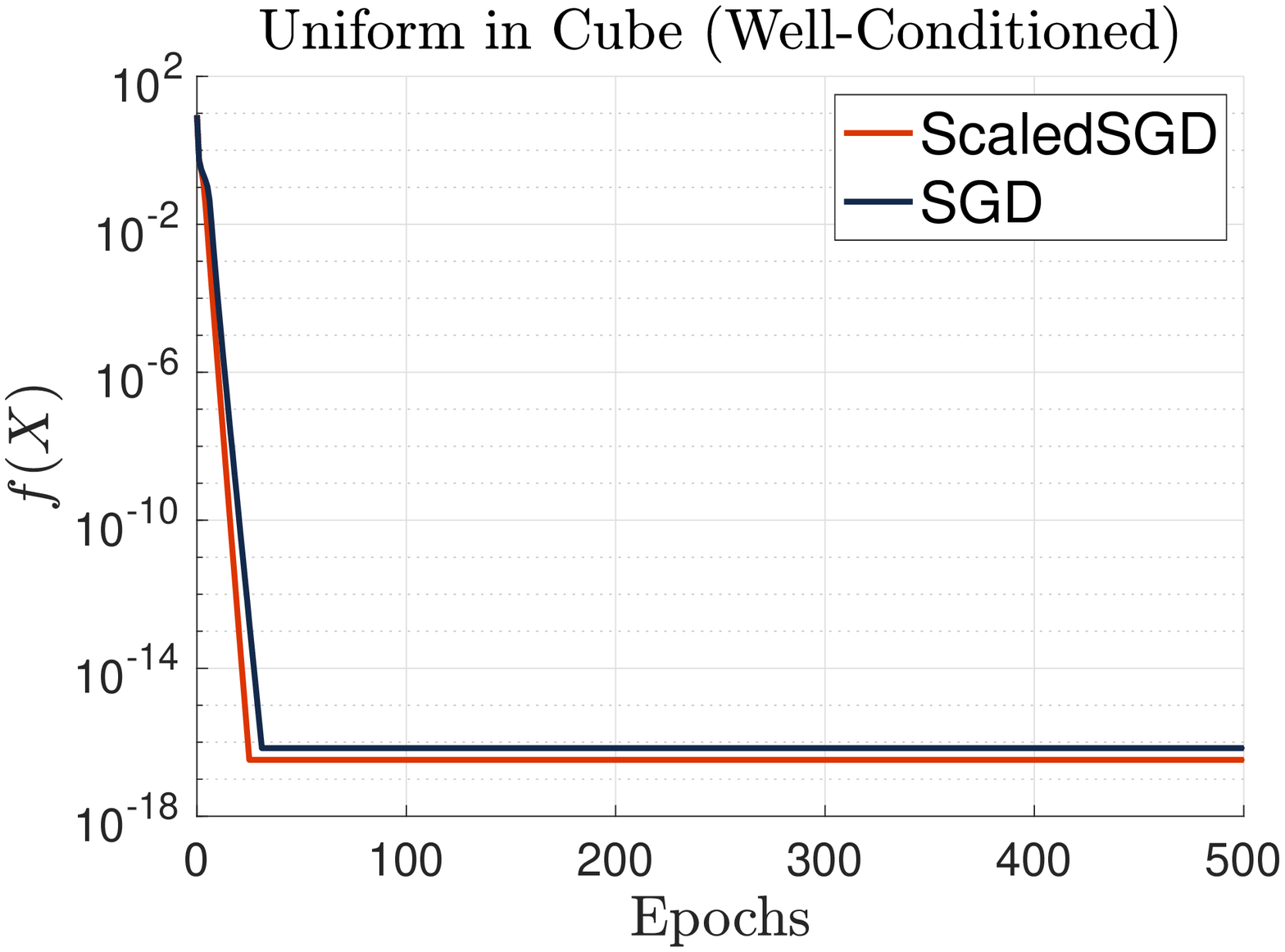}
    \end{subfigure}%
    \begin{subfigure}{0.5\textwidth}
      \centering
      \includegraphics[width=\linewidth]{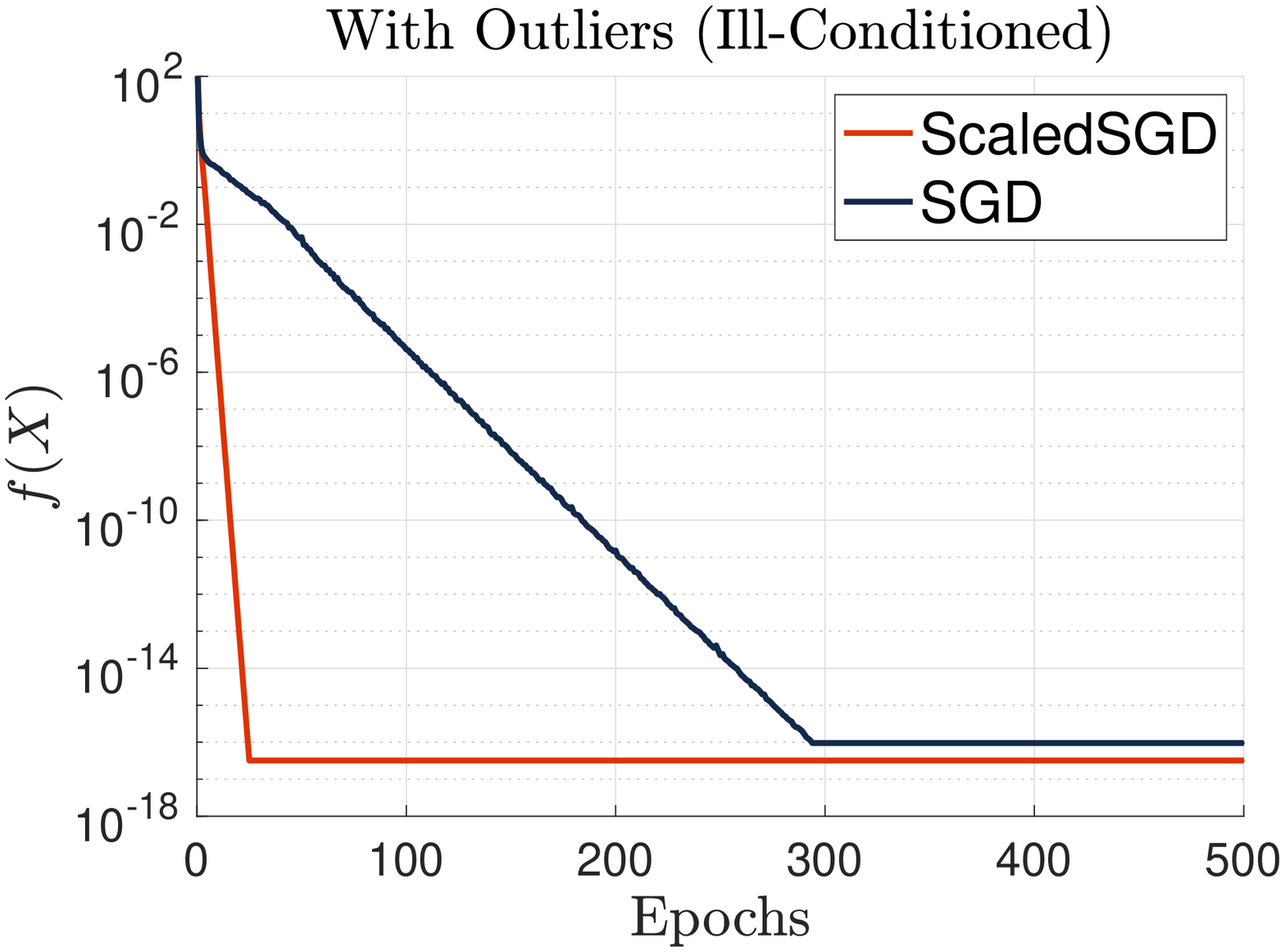}
    \end{subfigure}
    \caption{\textbf{Euclidean distance matrix (EDM) completion.} We compare the convergence rate of \ref{scaledsgd} and \ref{sgd} for EDM completion for two set of samples. (\textbf{Upper right}) 30 samples are uniformly distributed in the pink cube center at origin. (\textbf{Upper left}) 25 samples (in blue) are uniformly distributed in the cube, 5 outlier samples (in orange) are outside of the cube. (\textbf{Lower left}) Sample uniformly in cube. (\textbf{Lower right}) Sample with outliers.}
    \label{fig:edm}
\end{figure}

\paragraph{Item-item collaborative filtering (CF).}
In the task of item-item collaborative filtering (CF), the ground truth $M$ is a $d\times d$ matrix where $d$ is the number of items we wish to rank and the $i,j$-th of $M$ is a similarity measure between the items. Our goal is to learn a low-rank matrix that preserves the ranking of similarity between the items. For instance, given a pairwise sample $(i,j,k)$, if item $i$ is more similar to item $j$ than item $k$, then $M_{ij}>M_{ik}$. We want to learn a low-rank matrix that also has this property, i.e., the $i,j$-th entry is greater than the $i,k$-th entry. 

To gauge the scalability of ScaledSGD on a huge-scale real-world dataset, we perform simulation on item-item collaborative filtering using a $62,000\times 62,000$ item-item matrix $M$ obtained from MovieLens25M dataset. The CF model is trained using Bayesian Personalized Ranking (BRP) loss \cite{rendle2012bpr} on a training set, which consists of 100 million pairwise samples in $M$. The  performance of CF model is evaluated using Area Under the ROC Curve (AUC) score \cite{rendle2012bpr} on a test set, which consists of 10 million pairwise samples in $M$. The BPR loss is a widely used loss function in the context of collaborative filtering for the task of personalized recommendation, and the AUC score is a popular evaluation metric to measure the accuracy of the recommendation system. We defer the detail definition of BPR loss and AUC score to Appendix~\ref{app:cf}.

Figure~\ref{fig:bprhuge} plots the training BPR loss and testing AUC score within the first epoch (filled with red) and the second epoch (filled with blue). In order to measure the efficacy of ScaledSGD, we compare its testing AUC score against a standard baseline called the NP-Maximum \cite{rendle2012bpr}, which is the best possible AUC score by \textit{non-personalized} ranking methods.  For a rigorous definition, see Appendix~\ref{app:cf}.

We emphasize two important points in the Figure~\ref{fig:bprhuge}. First, the percentage of training samples needed for ScaledSGD to achieve the same testing AUC scores as NP-Maximum is roughly 4 times smaller than SGD. Though both ScaledSGD and SGD are able to achieve higher AUC score than NP-Maximum before finishing the first epoch, ScaledSGD achieve the same AUC score as NP-Maximum after training on $11\%$ of training samples while SGD requires $46\%$ of them. We note that in this experiment, the size of the training set is 100 million, this means that SGD would require 35 million more iterations than ScaledSGD before it can reach NP-Maximum. 

Second, the percentage of training samples needed for ScaledSGD to converge after the first epoch is roughly 5 times smaller than SGD. Given that both ScaledSGD and SGD converge to AUC score at around $0.9$ within the second epoch (area filled with blue), we indicate the percentage of training samples when both algorithms reach $0.9$ AUC score in Figure~\ref{fig:bprhuge}. As expected, ScaledSGD is able to converge using fewer samples than SGD, with only $16\%$ of training samples. SGD, on the other hand, requires $81\%$ training samples.

\begin{figure}[h!]
    \centering
    \begin{subfigure}{0.5\textwidth}
      \centering
      \includegraphics[width=\linewidth]{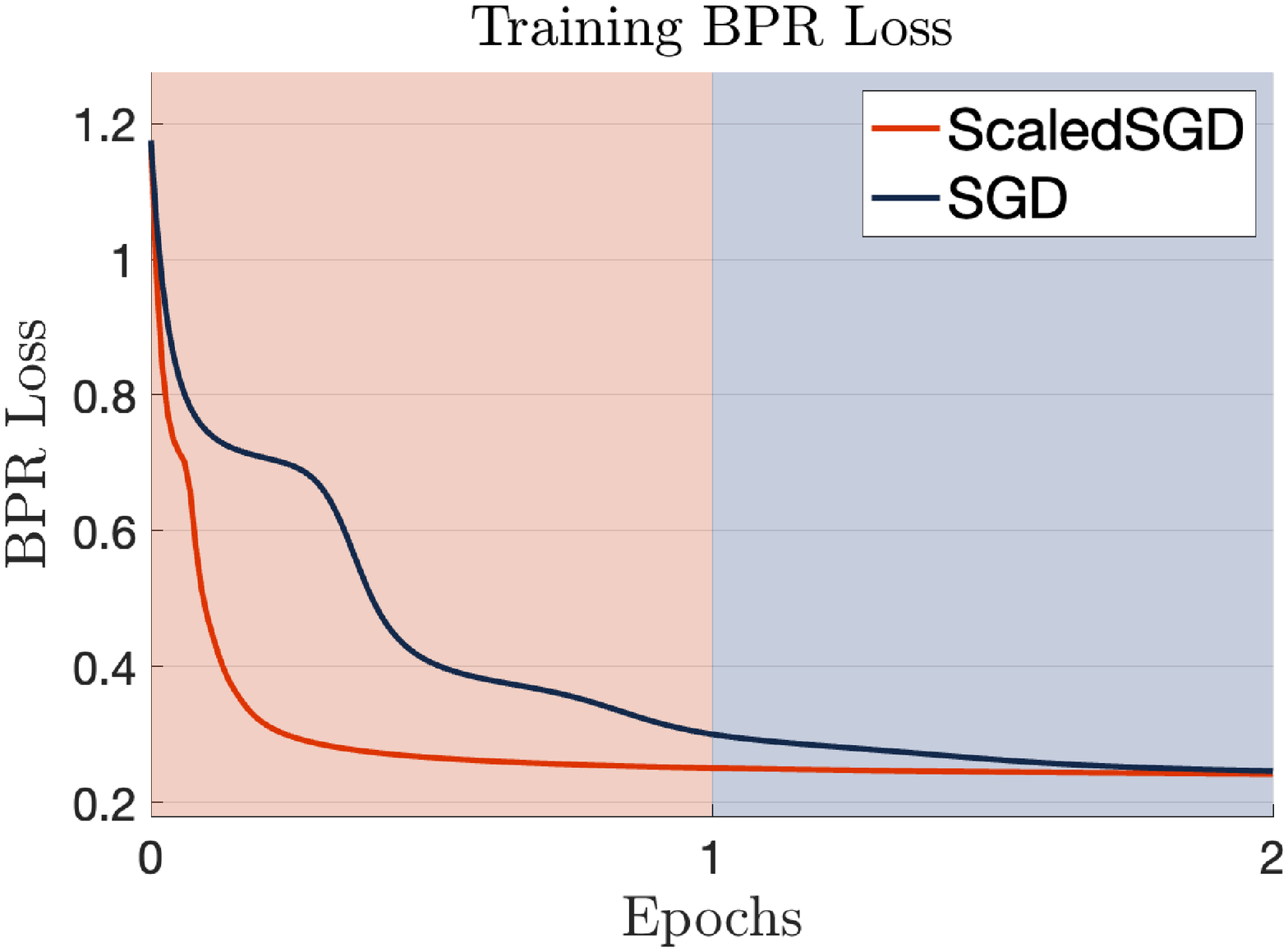}
    \end{subfigure}%
    \begin{subfigure}{0.5\textwidth}
      \centering
      \includegraphics[width=\linewidth]{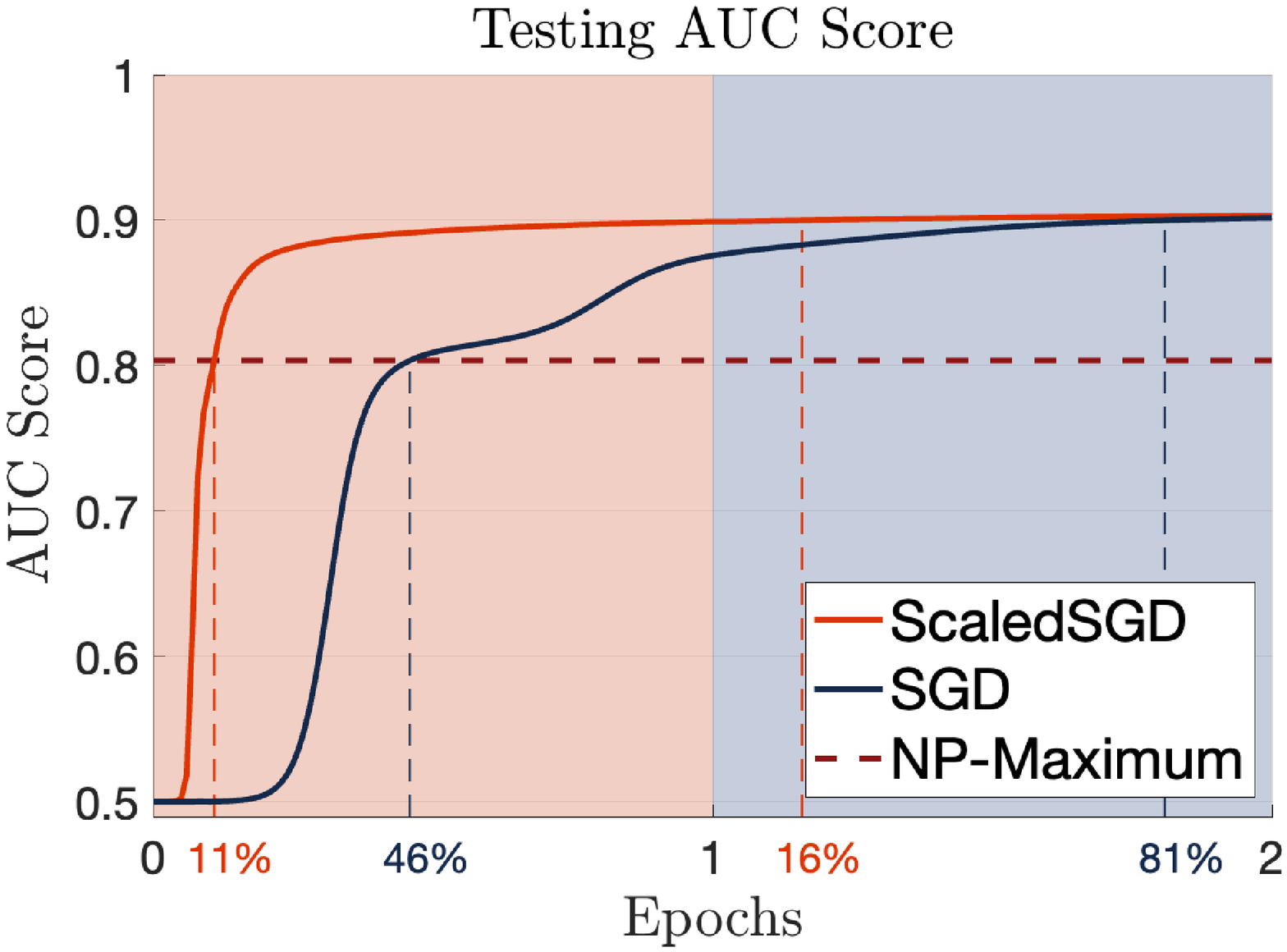}
    \end{subfigure}
    \caption{\textbf{Huge-scale item-item collaborative filtering.} (MovieLens25M dataset with $|\Omega_{\text{train}}|=$ 100 million and $|\Omega_{\text{test}}|=$ 10 million pairwise measurements). We compare the training BPR loss and testing AUC score of \ref{scaledsgd} and \ref{sgd}. (\textbf{Left}) Training BPR loss on the training set $\Omega_{\text{train}}$. (\textbf{Right}) Testing AUC score on the test set $\Omega_{\text{test}}$.}
    \label{fig:bprhuge}
\end{figure}

\section{Conclusions}
We propose an algorithm called \ref{scaledsgd} for huge scale online matrix completion. For the nonconvex approach to solving matrix completion, ill-conditioning in the ground truth causes \ref{sgd} to slow down significantly. ScaledSGD preserves all the favorable qualities of SGD while making it immune to ill-conditioning. For the RMSE loss, we prove that with an initial point close to the ground truth, ScaledSGD converges to an $\epsilon$-accurate solution in $O(\log(1/\epsilon))$ iterations, independent of the condition number $\kappa$. We also run numerical experiments on a wide range of other loss functions commonly used in applications such as collaborative filtering, distance matrix recovery, etc.  We find that ScaledSGD achieves similar acceleration on these losses, which means that it is widely applicable to many real problems. It remains future work to provide rigorous justification for these observations.

\bibliographystyle{unsrtnat}
\bibliography{references}

\newpage
\appendix
\section{Supplemental Details on Experiments in Main Paper}\label{app:exp}

\subsection{Experimental setup and datasets used}\label{app:A1}
\paragraph{Simulation environment.}
We implement both \ref{scaledsgd} and \ref{sgd} in MATLAB (version R2021a). All simulations in this paper are performed on a computer with Apple silicon M1 pro chip with 10-core CPU, 16-core GPU, and 32GB of RAM.

\paragraph{Datasets.} The datasets we use for the experiments in the main paper are described below. 

\begin{itemize}
\item \textbf{Matrix completion with RMSE loss:} In the simulation result shown in Figure~\ref{fig:rmse}, we synthetically generate both the well-conditioned and ill-conditioned ground truth matrix $M$. We fix both $M$ to be a rank-3 matrix of size $30\times 30$. To generate $M$, we sample a random orthonormal matrix $U\in\mathbb{R}^{30\times 3}$ and set $M=USU^T$. For well-conditioned case, we set $S = \mathrm{diag}(2,2,2),$ thus $M$ is perfectly conditioned with $\kappa=1$. For ill-conditioned case, we let $S = \mathrm{diag}(10,10^{-1},10^{-3}),$ so that $M$ is ill-conditioned with $\kappa = 10^4$. 

\item \textbf{Euclidean distance matrix completion:} In this simulation shown in Figure \ref{fig:edm}, the ground truth Euclidean distance matrix $D$ for experiments 1 and 2 are generated with respect to their sample matrix $X$ as $D_{ij}=\|x_i-x_j\|^2_2$. For the sample points in \textbf{Experiment 1}, we randomly sample (without replacement) $30$ points in 3-dimensional cube centered at origin with side length $2$, and the corresponding sample matrix $X$ has conditioned number $\kappa=1.3908$. For the sample points in \textbf{Experiment 2}, we take the first $5$ sample points in experiment 1 and perturb its x-coordinate by $10$, and keep the rest of the $25$ samples intact. The corresponding sample matrix $X$ has conditioned number $\kappa = 8.0828$. 

\item \textbf{Item-item collaborative filtering:} In this simulation shown in Figure \ref{fig:bprhuge}, we use the MovieLens25M dataset \cite{movielens}, which is a standard benchmark for algorithms for recommendation systems.\footnote{The MovieLens25M dataset is accessible at \url{https://grouplens.org/datasets/movielens/25m/}} This dataset consists of 25 million ratings over 62,000 movies by 162,000 users, the ratings are stored in an user-item matrix $G$ whose $(i,j)$-th entry is the rating that the $i$-th user gives the $j$-th movie. The rating is from $1$ to $5$, where a higher score indicates a stronger preference. If the $i,j$-th entry is $0$, then no rating is given. For the simulation of item-item collaborative filtering, the $(i,j)$-th entry of the ground truth item-item matrix $M$ is the \textit{similarity score} between the item $i$ and $j$, which can be computed by measuring cosine similarity between the $i$-th and $j$-th column of $G$.
\end{itemize}

\paragraph{Hyperparameter and initialization.}
We start ScaledSGD and SGD at the same initial point in each simulation. The initial points for each simulation are drawn from the standard Gaussian distribution.
\begin{itemize}
    \item \textbf{Matrix completion with RMSE loss:} The step-size for both ScaledSGD and SGD are set to be $\alpha=0.3$. The search rank for both ScaledSGD and SGD are set to be $r=3$.
    \item \textbf{Euclidean distance matrix completion:} Since SGD is only stable for small step-size in EDM completion problem, while ScaledSGD can tolerance larger step-sizes, we pick the largest possible step-size for ScaledSGD and SGD in both experiments. \textbf{Experiment 1}: step-size for ScaledSGD $\alpha = 0.2$, step-size for SGD $\alpha = 0.02$. \textbf{Experiment 2}: step-size for ScaledSGD $\alpha = 0.2$, step-size for SGD $\alpha = 0.002$. The search rank for both ScaledSGD and SGD are set to be $r=3$.
    \item \textbf{Item-item collaborative filtering:} The step-sizes for this experiment are set as follows: we first pick a small step-size and train the CF model over a sufficient number of epochs, this allows us to estimate the best achievable AUC score; we then set the step-sizes for both ScaledSGD and SGD to the largest possible step-size for which ScaledSGD and SGD is able to converge to the best achievable AUC score, respectively. The step-size for ScaledSGD $\alpha = 1,000$, step-size for SGD $\alpha = 0.01$. The search rank for both ScaledSGD and SGD are set to be $r=3$. 
\end{itemize}

\subsection{Matrix completion with RMSE loss}
We now turn to the practical aspects of implementing ScaledSGD for RMSE loss function. In practical setting, suppose we are given a set $\Omega=\{(i,j)\}$ that contains indices for which we know the value of $M_{ij}$, our goal is to recover the missing elements in $M$ by solving the following nonconvex optimization
\begin{equation*}
\min_{X\in\R^{d\times r}}f(X)=\frac{1}{2|\Omega|}\sum_{(i,j)\in\Omega}\left(x_{i}^{T}x_{j}-M_{ij}\right)^{2}.
\end{equation*}

The gradient of $f(X)$ is
\[
\nabla f(X)= \frac{1}{|\Omega|}\sum_{(i,j)\in\Omega}\left(x_{i}^{T}x_{j}-M_{ij}\right)\left(e_ie_j^T+e_je_i^T\right)X.
\]
\paragraph{ScaledSGD update equations for RMSE loss.}
Each iteration of ScaledSGD samples one element $(i,j)\in\Omega$ uniformly. The resulting iteration updates only two rows of $X$
\begin{align*}
&x_{i,+}=x_{i}-\alpha\cdot\left(x_{i}^{T}x_{j}-M_{ij}\right)Px_{j},\quad
x_{j,+}=x_{j}-\alpha\cdot\left(x_{i}^{T}x_{j}-M_{ij}\right)Px_{i}.
\end{align*}
The update on $P$ is low-rank
\[
P_{+}=(P^{-1}+x_{i,+}x_{i,+}^{T}+x_{j,+}x_{j,+}^{T}-x_{i}x_{i}^{T}-x_{j}x_{j}^{T})^{-1},
\]
and can be computed by calling four times of rank-1 Sherman--Morrison--Woodbury (SMW) update formula in $O(r^{2})$ time
\begin{equation}\label{eq:p_update}
(P^{-1}+uu^{T})^{-1}=P-\frac{Puu^{T}P}{1+u^{T}Pu},\qquad(P^{-1}-uu^{T})^{-1}=P+\frac{Puu^{T}P}{1-u^{T}Pu}.
\end{equation}
In practice, this low-rank update can be ``pushed'' onto a centralized
storage of the preconditioner $P$. Heuristically, independent copies
of $P$ can be maintained by separate, distributed workers, and a
centralized dispatcher can later merge the updates to $P$ by simply
adding the cumulative low-rank updates onto the existing centralized
copy. 

\subsection{Euclidean distance matrix (EDM) completion}
Suppose that we have $d$ points $x_1,\ldots,x_d\in\mathbb R^r$ in $r$ dimensional space, the Euclidean distance matrix $D\in\mathbb R^{d\times d}$ is a matrix of pairwise squared distance between $d$ points in Euclidean space, namely, $D_{ij}=\|x_i-x_j\|^2$. Many applications, such as wireless sensor networks, communication and machine learning, require Euclidean distance matrix to provide necessary services. However, in practical scenario, entries in $D$ that correspond to points far apart are often missing due to high uncertainty or equipment limitations in distance measurement. The task of Euclidean distance matrix completion is to recover the missing entries in $D$ from a set of available measurement, and this problem can be formulated as a rank $r$ matrix completion problem with respect to pairwise square loss function. Specifically, let $X\in\mathbb R^{d\times r}$ be a matrix containing $x_1,\ldots,x_d$ in its row and let $M=XX^T$ be the Grammian of $X$. Each entry of $D$ can be written in terms of three entries in $M$
\[
D_{ij}=\|x_i-x_j\|^2=x_i^Tx_i-2x_i^Tx_j+x_j^Tx_j=M_{ii}-2M_{ij}+M_{jj}.
\]
Hence, given a set of sample $\Omega=\{(i,j)\}$ in $D$, the pairwise square loss function for EDM completion reads
\[
\min_{X\in\R^{d\times r}}f(X)=\frac{1}{4|\Omega|}\sum_{(i,j)\in\Omega} \left(x_i^Tx_i-2x_i^Tx_j+x_j^Tx_j-D_{ij}\right)^2.
\]
The gradient of $f(X)$ is
\[
\nabla f(X)= \frac{1}{|\Omega|}\sum_{(i,j)\in\Omega}\left(x_i^Tx_i-2x_i^Tx_j+x_j^Tx_j-D_{ij}\right)\left[\left(e_ie_i^T+e_je_j^T\right)-\left(e_ie_j^T+e_je_i^T\right)\right]X.
\]
\paragraph{ScaledSGD update equations for EDM completion.}
Each iteration of ScaledSGD samples one element $(i,j)\in\Omega$ uniformly. The resulting iteration updates only two rows of $X$
\begin{align*}
    x_{i,+} &= x_i - \alpha\cdot\left(x_i^Tx_i-2x_i^Tx_j+x_j^Tx_j-D_{ij}\right)P(x_i-x_j),\\
     x_{j,+} &= x_j - \alpha\cdot\left(x_i^Tx_i-2x_i^Tx_j+x_j^Tx_j-D_{ij}\right)P(x_j-x_i)
\end{align*}
Similarly, the update on $P$ is low-rank and can be computed by calling four times of \eqref{p_update}.

\subsection{Item-item collaborative filtering (CF)} \label{app:cf}
In the task of item-item collaborative filtering, the ground truth $M$ is an $d\times d$ matrix where $d$ is the number of items we wish to rank and the $i,j$-th of $M$ is a similarity measure between the items. Our goal is to learn a low-rank matrix that preserves the ranking of similarity between the items. For instance, suppose that item $i$ is more similar to item $j$ than item $k$, then $M_{ij}>M_{ik}$, we want to learn a low-rank matrix $XX^T$ that also has this property, i.e., $x_{i}^Tx_{j} \geq x_{i}^Tx_{k}$ where $x_{i}$ is the $i$-th row of $X$. 
\paragraph{Similarity score.}
An important building block of item-item recommendation systems is the so-called item-item similarity matrix \cite{davidson2010youtube, linden2003amazon, smith2017two}, which we denote by $M$. The $i,j$-th entry of this matrix is the pairwise \textit{similarity scores} of items $i$ and $j$. There are various measures of similarity. In our experiments we adopt a common similarity measure known as cosine similarity \cite{linden2003amazon}. As a result, the item-item matrix can be computed from the user-item matrix. In particular, let $g_i$, $g_j$ denote the $i$-th and $j$-th columns of the user-item matrix $G$, corresponding to the ratings given by all users to the $i$-th and $j$-th items. Then the $(i,j)$-th element of the item-item matrix $M$ is set to
\[M_{ij} = g_i^Tg_j/(\|g_i\|\|g_j\|).\]
In general, the item-item matrix computed this way will be very sparse and not capable of generating good recommendations. Our goal is to complete the missing entries of this matrix, assuming that that $M$ is low-rank. As we will see, we can formulate this completion problem as an optimization problem over the set of $\mathrm{rank}$-$r$ matrices. 

\paragraph{Pairwise entropy loss (BPR loss).}
The Bayesian Personalized Ranking (BRP) loss \cite{rendle2012bpr} is a widely used loss function in the context of collaborative filtering. For the task of predicting a personalized ranking of a set of items (videos, products, etc.), BRP loss often outperforms RMSE loss because it is directly optimized for ranking; most collaborative filtering models that use RMSE loss are essentially scoring each individual item based on user implicit feedbacks, in applications that only positive feedbacks are available, the models will not be able to learn to distinguish between negative feedbacks and missing entries. 

The BPR loss in this context can be defined as follows. Let $\Omega=\{(i,j,k)\}$ denote a set of indices for which we observe the ranking of similarity between items $i,j,k$. Our observations are of the form $Y_{ijk} = 1$ if $M_{ij}>M_{ik}$ and $Y_{ijk} = 0$ otherwise. In other words, $Y_{ijk} = 1$ if item $i$ is more similar to item $j$ than to item $k$. We form a candidate matrix of the form $XX^T$, where $X\in\mathbb{R}^{d\times r}$. Our hope is that $XX^T$ preserves the ranking between the items. The BPR loss function is designed to enforce this property. 

Let $x_i$ denote the $i$-th row of $X$ and set $z_{ijk} = (XX^T)_{ij}-(XX^T)_{ik} = x_i^T(x_j-x_k)$. The BPR loss attempts to preserve the ranking of samples in each row of $M$ by minimizing the logistic loss with respect to $(Y_{ijk},\sigma(z_{ijk}))$, where $\sigma(\cdot)$ is the sigmoid function:
\begin{equation*}
\min_{X\in\R^{n\times r}}f(X)=\frac{1}{|\Omega|}\sum_{(i,j,k)\in\Omega} -Y_{ijk}\log\left(\sigma(z_{ijk})\right)-\left(1-Y_{ijk}\right)\log\left(1-\sigma(z_{ijk})\right).
\end{equation*}
Then the gradient of $f(X)$ is
\[
\nabla f(X)= \frac{1}{|\Omega|}\sum_{(i,j,k)\in\Omega}\left(\sigma(z_{ijk})-Y_{ijk}\right)\left[\left(e_ie_j^T+e_je_i^T\right)-\left(e_ie_k^T+e_ke_i^T\right)\right]X.
\]
\paragraph{ScaledSGD update equations for BPR loss.}
Similarly to the previous section, each iteration of ScaledSGD samples one element $(i,j,k)\in\Omega$ uniformly. The resulting iteration updates only three rows of $X$, as in 
\begin{alignat*}{2}
    &x_{i,+} = x_i - \alpha\cdot\left(\sigma(z_{ijk})-Y_{ijk}\right)P(x_j-x_k),\quad
     x_{j,+} = x_j - \alpha\cdot\left(\sigma(z_{ijk})-Y_{ijk}\right)Px_i,\\
    &x_{k,+} = x_k + \alpha\cdot\left(\sigma(z_{ijk})-Y_{ijk}\right)Px_i
\end{alignat*}
Similar to before, the preconditioner $P$ can be updated via six call of \eqref{p_update} in $O(r^2)$ time. 

\paragraph{The AUC score.} The AUC score \citep{rendle2012bpr} is a popular evaluation metric for recommendation system.  Roughly speaking, the AUC score of a candidate matrix $XX^T$ is the percentage of ranking of the entries of $M$ that is preserved by $XX^T$. Specifically, for each sample $(i,j,k)\in\Omega$, we define a indicator variable $\delta_{ijk}$ as
\[
\delta_{ijk}=\begin{cases}
1 & \text{if $z_{ijk}>0$ and $Y_{ijk}=1$} \\
1 & \text{if $z_{ijk}\leq 0$ and $Y_{ijk}=0$} \\
0 & \text{otherwise},
\end{cases}
\]
where we recall that $Y_{ijk}$ is our observation and $z_{ijk}=(XX^T)_{ij}-(XX^T)_{ik}$.
In other words, $\delta_{ijk}=1$ only if the ranking between $M_{ij}$ and $M_{ik}$ is preserved by $(XX^T)_{ij}$ and $(XX^T)_{ik}$. The AUC score is then defined as the ratio
\[
\text{AUC}=\frac{1}{|\Omega|}\sum_{(i,j,k)\in\Omega}\delta_{ijk}.
\]
Thus, a higher AUC score indicates that the candidate matrix $XX^T$ perserves a larger percentage of the pairwise comparisons in $|\Omega|$.

\paragraph{Training a CF model for Figure~\ref{fig:bprhuge}.}
We precompute a dataset $\Omega$ of $110$ million item-item pairwise comparisons using the user-item ratings from the MovieLens25M dataset, and then run ScaledSGD and SGD over 2 epochs on this dataset. Let $d=62,000$ denote the number of items and let $\Omega \subseteq [d]^3$ denote the set of observations, i.e., the set of entries $(i,j,k)$ where we observe $Y_{ijk} = 1$ if $M_{ij}>M_{ik}$ and $Y_{ijk} = 0$ if $M_{ij}<M_{ik}$. We construct $\Omega$ by sampling 110 million pairwise measurements that have either $M_{ij}>M_{ik}$ or $M_{ij}<M_{ik}$ uniformly at random without replacement from $(i,j,k) \sim [d]^3$. We do this because the item-item matrix $M$ remains highly sparse, and there are many pairs of $(i,j)$ and $(i,k)$ for which $M_{ij}=M_{ik}=0$.

To ensure independence between training and testing, we divide the set $\Omega$ into two disjoint sets $\Omega_{\text{train}}$ and $\Omega_{\text{test}}$. The first set $\Omega_{\text{train}}$ consists of 100 million of all observations, which we use to fit our model. The second set $\Omega_{\text{test}}$ consists of 10 million samples for which we use to calculate the AUC score of our model on new data. 

\paragraph{Upper bounds on the non-personalized ranking AUC score (NP-Maximum).} 
As opposed to personalized ranking methods, non-personalized ranking methods generate the same ranking for every pair of item $j$ and $k$, independent of item $i$. In the context of item-item collaborative filtering, the non-personalized ranking method can be defined as follows. Given a set of pairwise comparisons $\Omega=(\{i,j,k\})$ and observations $Y_{ijk}$, we optimized the ranking between item $j$ and $k$ on a candidate vector $x$, where $x\in\mathbb{R}^{d}$.

Let $x_i$ denote the $i$-th entry of $x$, the non-personalized ranking method attempts to preserve the ranking between the $x_j$ and $x_k$ by minimizing the logistic loss with respect to $(Y_{ijk}, \sigma(x_j-x_k))$ where $\sigma(\cdot)$ is the sigmoid function:
\begin{equation*}
\min_{x\in\R^{r}}f(x)=\frac{1}{|\Omega|}\sum_{(i,j,k)\in\Omega} -Y_{ijk}\log\left(\sigma(x_j-x_k)\right)-\left(1-Y_{ijk}\right)\log\left(1-\sigma(x_j-x_k)\right).
\end{equation*}
The gradient of $f(x)$ is
\[
    \nabla f(x)=\frac{1}{|\Omega|}\sum_{(i,j,k)\in\Omega}\left(\sigma(x_j-x_k)-Y_{ijk} \right)\left[(e_j^T-e_k^T)\right]x,
\]
and the SGD update equations for $x_j$ and $x_k$ are
\[
    x_{j,+}=x_{j}-\alpha\cdot\left(\sigma(x_j-x_k)-Y_{ijk} \right)x_{j},\quad
    x_{k,+}=x_{k}+\alpha\cdot\left(\sigma(x_j-x_k)-Y_{ijk} \right)x_{k}.
\]
Notice that non-personalized ranking method is not a matrix completion problem, the regular SGD is used to minimized $f(x)$. To find the upper bound on the non-personalized ranking AUC score, we directly optimize the non-personalized ranking on the test set $\Omega_{\text{test}}$, and evaluated the corresponding AUC score on  $\Omega_{\text{test}}$. Since we perform both training and evaluation on $\Omega_{\text{test}}$, this corresponding AUC score is the upper bound on the best achievable AUC score on $\Omega_{\text{test}}$.

\section{Additional Experiments on pointwise cross-entropy loss}
This problem is also known as 1-bit matrix completion \cite{davenport20141}. Here our goal is to recover a rank-$r$ matrix $M$ through \textit{binary} measurements. Specifically, we are allowed to take independent measurements on every entry $M_{ij}$, which we denote by $Y_{ij}$. Let $\sigma(\cdot)$ denote the sigmoid function, then $Y_{ij}=1$ with probability $\sigma(M_{ij})$ and $Y_{ij}=0$ otherwise. After a number of measurements are taken on each entry in the set $\Omega$, let $y_{ij}$ denote the percentage of measurements on the $i,j$-th entry that is equal to $1$. The plan is to find the maximum likelihood estimator for $M$ by minimizing a cross-entropy loss defined as follow
\[
\min_{X \in\mathbb{R}^{d\times r}} f(X) = \frac{1}{|\Omega|}\sum_{(i,j)\in\Omega} -y_{ij}\log\left(\sigma(x_i^Tx_j)\right)-\left(1-y_{ij}\right)\log\left(1-\sigma(x_i^Tx_j)\right).
\]
We assume an ideal case where the number of measurements is large enough so that $y_{ij} = \sigma(M_{ij})$ and the entries are fully observed. The gradient of $f(X)$ is
\[
\nabla f(X)= \frac{1}{|\Omega|}\sum_{(i,j)\in\Omega}\left(\sigma(x_i^Tx_j)-y_{ij}\right)\left(e_ie_j^T+e_je_i^T\right)X.
\]
\paragraph{ScaledSGD update equations for pointwise cross-entropy loss.}
Each iteration of ScaledSGD samples one element $(i,j)\in\Omega$ uniformly. The resulting iteration updates only two rows of $X$
\begin{align*}
    x_{i,+} = x_i - \alpha\cdot\left(\sigma(x_i^Tx_j)-y_{ij}\right)Px_j,\quad
    x_{j,+} = x_j - \alpha\cdot\left(\sigma(x_i^Tx_j)-y_{ij}\right)Px_i.
\end{align*}
The preconditioner $P$ can be updated by calling four times of \eqref{p_update} as in RMSE loss.

\paragraph{Matrix completion with pointwise cross-entropy loss.} We apply ScaledSGD to perform matrix completion through minimizing pointwise cross-entropy loss. In this experiment, the well-conditioned and ill-conditioned ground truth matrix $M$ is the same as those in Figure~\ref{fig:rmse}, and the process of data generation are described in \ref{app:A1}. The learning rate for both ScaledSGD and SGD are set to be $\alpha=1$. The search rank for both ScaledSGD and SGD are set to be $r=3$.

Figure \ref{fig:1bit} plots the error $f(X) = \|XX^T-M\|_F^2$ against the number of epochs. Observe that the results shown in Figure \ref{fig:1bit} are almost identical to that of the RMSE loss shown in Figure~\ref{fig:rmse}. Ill-conditioning causes SGD to slow down significantly while ScaledSGD is unaffected. 

\begin{figure}[h!]
    \centering
    \begin{subfigure}{0.5\textwidth}
      \centering
      \includegraphics[width=\linewidth]{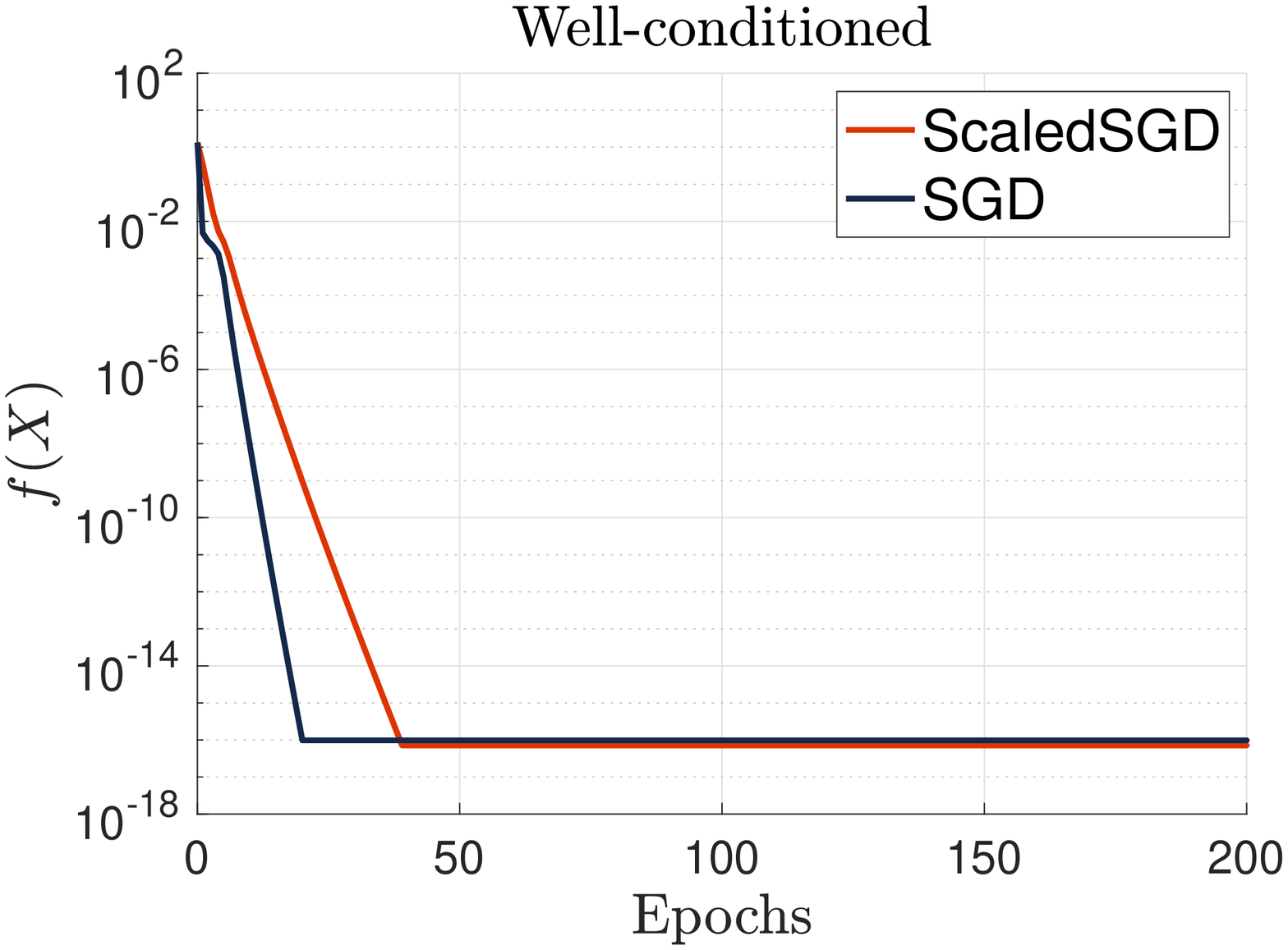}
    \end{subfigure}%
    \begin{subfigure}{0.5\textwidth}
      \centering
      \includegraphics[width=\linewidth]{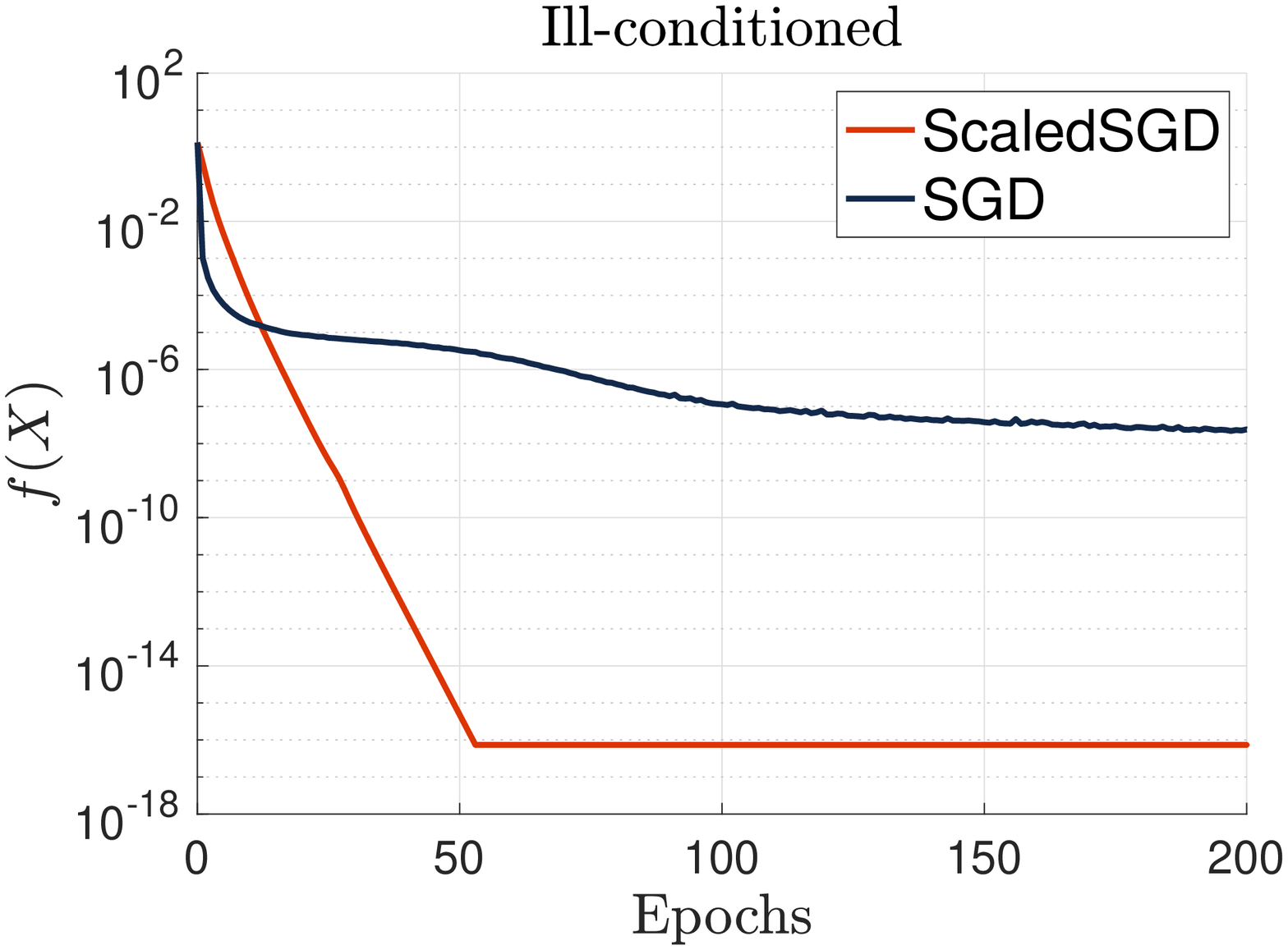}
    \end{subfigure}
    \caption{\textbf{Matrix Completion with pointwise cross entropy loss.} We compare the convergence rate of \ref{scaledsgd} and \ref{sgd} for a well-conditioned and ill-conditioned ground truth matrix $M$ used in Figure \ref{fig:rmse}. (\textbf{Left}) Well-conditoned $M$. (\textbf{Right}) Ill-conditoned $M$.}
    \label{fig:1bit}
\end{figure}

\newpage
\section{Additional Experiments with Noise}
To mimic the real-world datasets, we corrupt each entry of ground truth matrix $M$ by white Gaussian noise. We first generate a noiseless well-conditioned and ill-conditioned matrix $\tilde M$ following same procedure as the one described in \ref{app:A1}. For well-conditioned case, we set the singular value as $S=\mathrm{diag}(10,10,10)$. For the ill-conditioned case, we set $S=\mathrm{diag}(10,10^{-1},10^{-3})$. To obtain a noisy ground truth, we generate a matrix of white Gaussian noise $W$ corresponding to a fixed signal to noise ratio (SNR), which is defined as $\text{SNR}=20\log_{10}(\|\tilde M\|_F / \|W\|_F)$. Finally, we set $M=\tilde M + W$. For the experiments in this section, we set $\text{SNR}=15\text{dB}$. For the case of well-conditioned $\tilde M$, the resulting $M=\tilde M + W$ is full-rank with condition number $\kappa=310.72$. For the case of ill-conditioned $\tilde M$, the resulting $M$ is full-rank with condition number $\kappa = 423.5022$. 

\paragraph{Matrix completion with RMSE loss on noisy datasets.}
We plot the convergence rate of ScaledSGD and SGD under the noisy setting in Figure \ref{fig:rmse_noise}. In this experiment, we pick a larger search rank $r=5$ to accommodate the noisy ground truth. Observe that SGD slows down in both the well-conditioned and ill-conditioned case due to the addition of white Gaussian noise and the larger search rank $r$, while ScaledSGD converge linearly toward the noise floor. 

We also plot the noise floor, which can be computed as follows. First we take the eigendecomposition of $M=Q\Lambda Q^T$, where $Q$ is an orthonormal matrix and $\Lambda$ is a diagonal matrix containing the eigenvalues of $M$ sorted in descending order in its diagonal entries. Let $\Lambda'$ be a diagonal matrix such that $\Lambda'_{ii}=\Lambda_{ii}$ if $i\leq r$, and $\Lambda'_{ii}=0$ otherwise, then the noise floor is defined as the RMSE between $M$ and its best rank-$r$ approximation $M'=Q\Lambda' Q^T$, which is equal to $\frac{1}{2|\Omega|}\sum_{(i,j)\in\Omega}\left(M'_{ij}-M_{ij}\right)^{2}$.

The step-sizes in the simulation are set to be the largest possible step-sizes for which ScaledSGD and SGD can converge to the noise floor. For ScaleSGD, the step-size is set to be $\alpha=0.15$. For SGD, the step-size is set to be $\alpha=0.01$. $\text{SNR}=20\log_{10}(\|\tilde M\|_F / \|W\|_F)=15\text{dB}$.
\begin{figure}[h!]
    \centering
    \begin{subfigure}{0.5\textwidth}
      \centering
      \includegraphics[width=\linewidth]{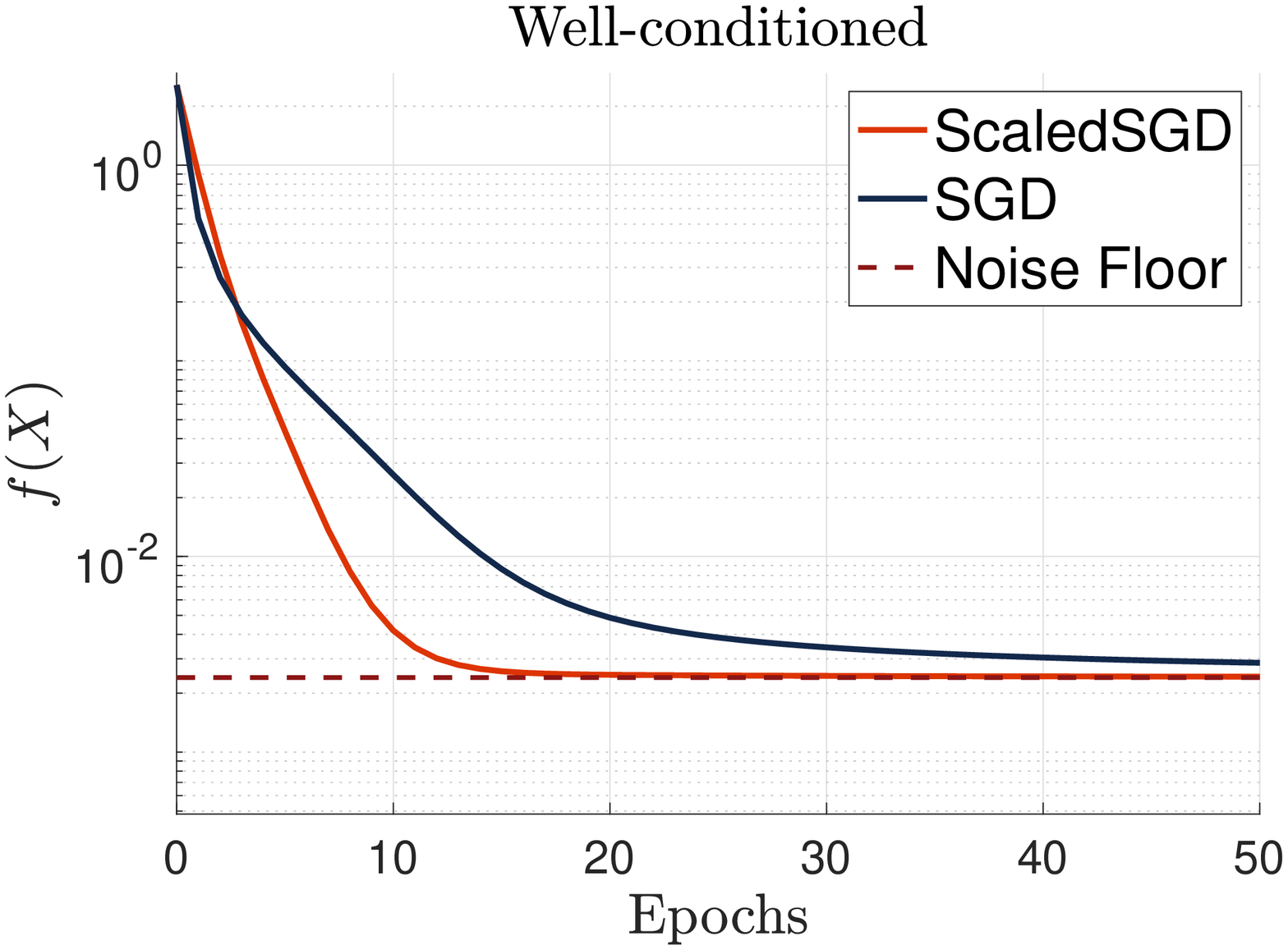}
    \end{subfigure}%
    \begin{subfigure}{0.5\textwidth}
      \centering
      \includegraphics[width=\linewidth]{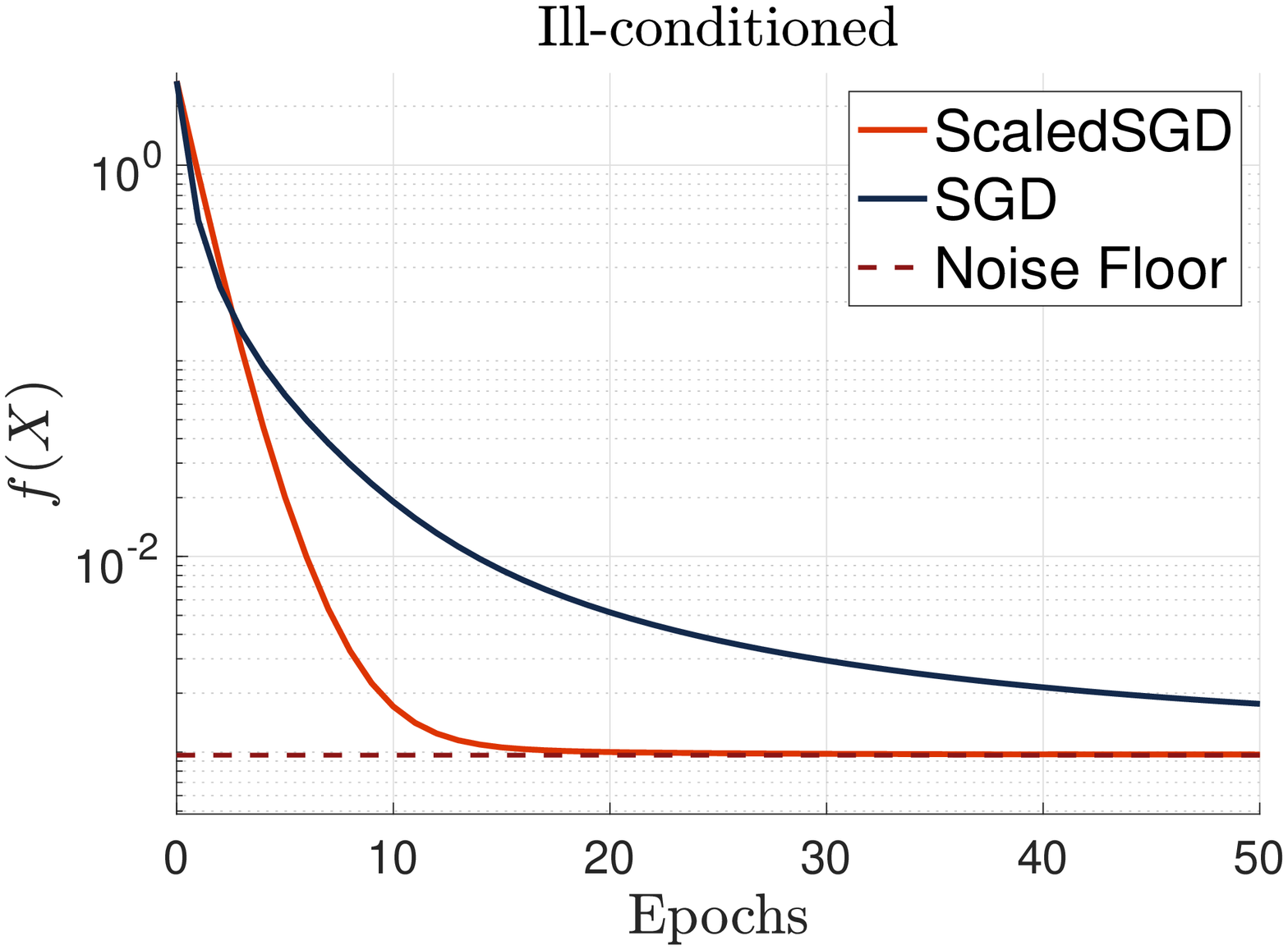}
    \end{subfigure}
    \caption{\textbf{Matrix Completion with RMSE loss in noisy setting.} We compare the convergence rate of \ref{scaledsgd} and \ref{sgd} for noisy ground truth matrix $M=\tilde M + W$ computed with respect to a well-conditioned and ill-conditioned $\tilde M$ and white Gaussian noise $W$. (\textbf{Left}) Well-conditioned $\tilde M$. (\textbf{Right}) Ill-conditioned $\tilde M$.}
    \label{fig:rmse_noise}
\end{figure}

\paragraph{Matrix completion with pointwise cross-entropy loss on noisy datasets.}
We plot the convergence rate of ScaledSGD and SGD under the noisy setting in Figure \ref{fig:1bit_noise}. Similar to RMSE loss in noisy setting, SGD show down in both well-conditioned and ill-conditioned case, while ScaledSGD converge linearly toward the noise floor. In this simulation, the search rank is set to be $r=5$. The step-size are set to be the largest possible step-sizes for which ScaledSGD and SGD can converge to the noise floor. For ScaleSGD, the step-size is set to be $\alpha=0.15$. For SGD, the step-size is set to be $\alpha=0.01$. $\text{SNR}=15\text{dB}$.

\begin{figure}[h!]
    \centering
    \begin{subfigure}{0.5\textwidth}
      \centering
      \includegraphics[width=\linewidth]{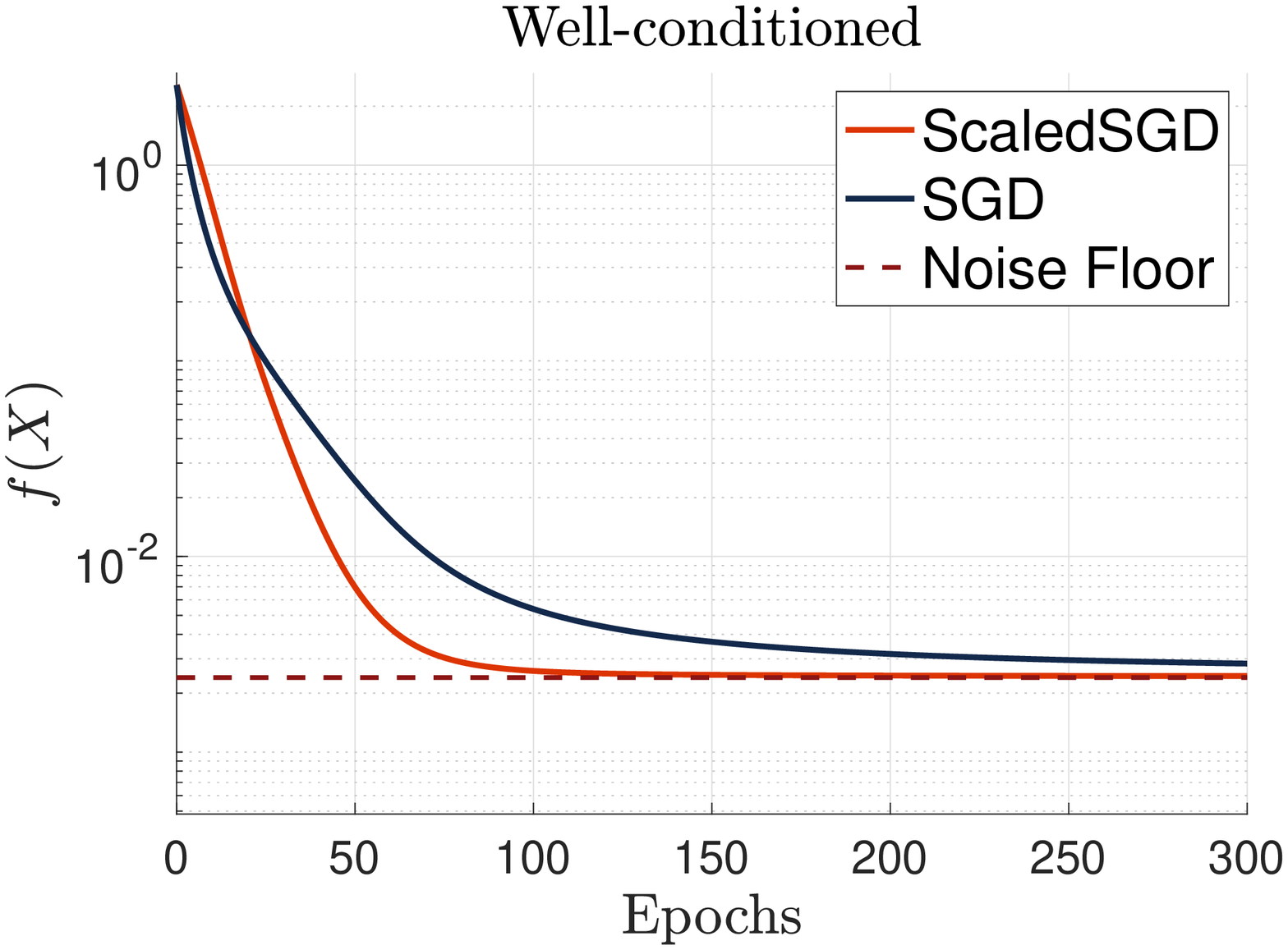}
    \end{subfigure}%
    \begin{subfigure}{0.5\textwidth}
      \centering
      \includegraphics[width=\linewidth]{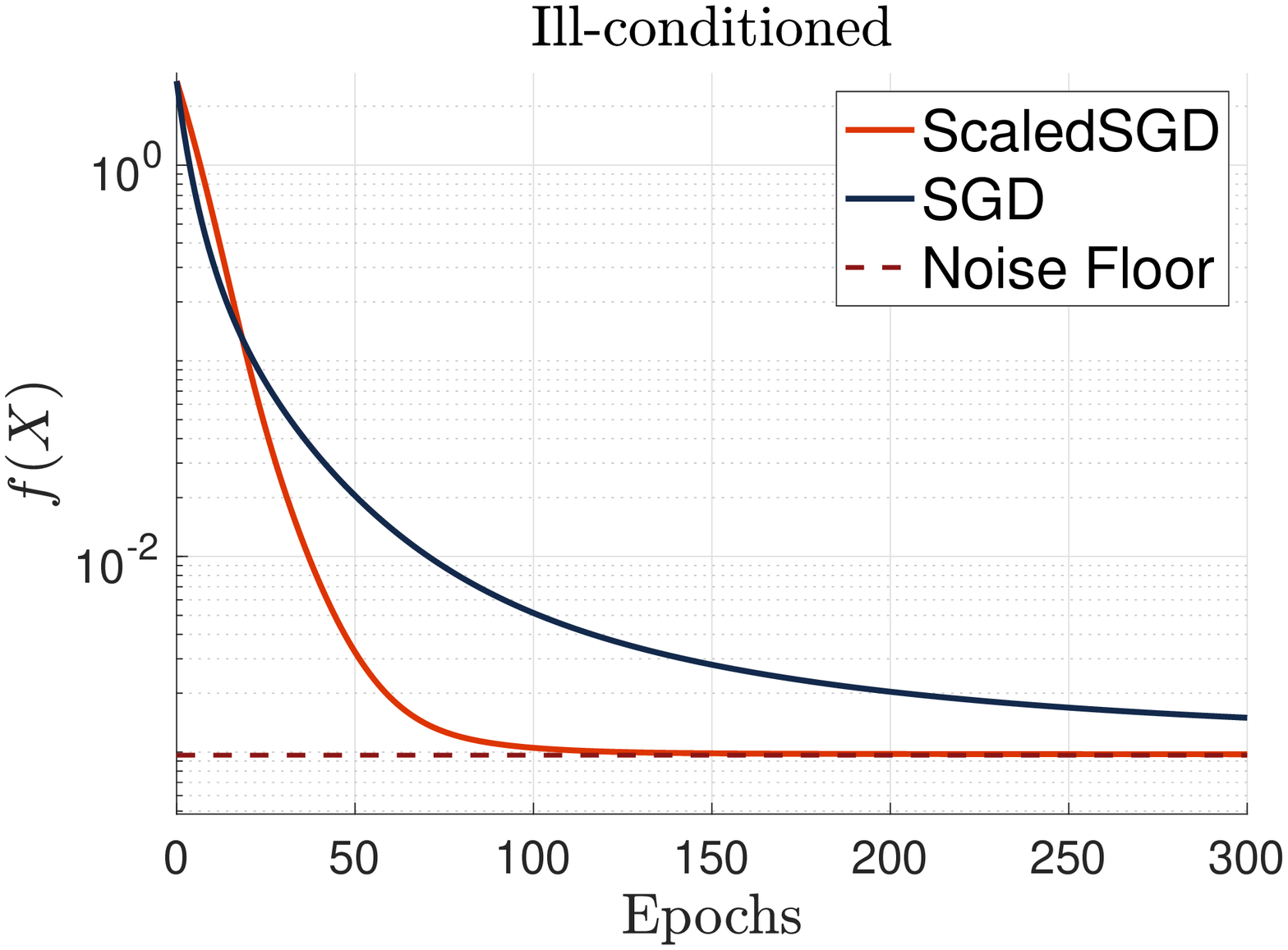}
    \end{subfigure}
    \caption{\textbf{Matrix Completion with pointwise cross-entropy loss in the noisy setting.} We compare the convergence rate of \ref{scaledsgd} and \ref{sgd} for noisy ground truth matrix $M=\tilde M + W$ computed with respect to a well-conditioned and ill-conditioned $\tilde M$ and white Gaussian noise $W$. (\textbf{Left}) Well-conditioned $\tilde M$. (\textbf{Right}) Ill-conditioned $\tilde M$.}
    \label{fig:1bit_noise}
\end{figure}

\section{Additional simulation on item-item collaborative filtering}
Finally, we perform three additional experiments on item-item collaborative filtering in order to compare the ability of \ref{scaledsgd} and \ref{sgd} to generate good recommendations using matrix factorization. 

\paragraph{Dataset.}
For additional simulations on item-time collaborative filtering, we use the MovieLens-Latest-Small and MovieLens-Latest-Full datasets \cite{movielens} in order to gauge the performance of our algorithm on different scales. First, we run a small-scale experiment on the MovieLens-Latest-Small dataset that has 100,000 ratings over 9,000 movies by 600 users. Second, we run a medium-scale and a large-scale experiment on the MovieLens-Latest-Full dataset with 27 million total ratings over 58,000 movies by 280,000 users.\footnote{Both datasets are accessible at \url{https://grouplens.org/datasets/movielens/latest/}}

\paragraph{Experimental Setup.}
The process of training a collaborative filtering model is described in \ref{app:cf}. The hyperparameters for the three experiments in this section are described below.
\begin{itemize}
\item \textbf{MovieLens-Latest-Small dataset:} 
In the small-scale experiment, we sample $|\Omega_{\text{train}}|=$ 1 million and $|\Omega_{\text{test}}|=$ 100,000 pairwise observations for training and testing, respectively. We set our search rank to be $r=3$, so the optimization variable $X$ is of size $9000 \times 3$. Both ScaledSGD and SGD are initialized using a random Gaussian initial point. For ScaledSGD the step-size is $10^{3}$ and for SGD the step-size is $5\times 10^{-2}$. 

\item \textbf{MovieLens-Latest-Full dataset:} In the medium-scale experiment, we sample $|\Omega_{\text{train}}|=$ 10 million and $|\Omega_{\text{test}}|=$ 1 million pairwise observations for training and testing, respectively. In the large-scale experiment, we sample $|\Omega_{\text{train}}|=$ 30 million and $|\Omega_{\text{test}}|=$ 3 million pairwise observations for training and testing, respectively. In both cases, we set our search rank to be $r=3$, so the optimization variable $X$ is of size $58000 \times 3$. For ScaledSGD the step-size is $5\times 10^{3}$ and for SGD the step-size is $5\times 10^{-2}$. 
\end{itemize}



\paragraph{Results.}
The results of our experiments for ScaledSGD and SGD are plotted in Figures \ref{fig:bprsmall}, \ref{fig:bprmedium}, and \ref{fig:bprlarge}. 
In all three cases, ScaledSGD reaches the AUC scores that are greater than NP-Maximum's within the first epoch, while SGD requires more than one epoch to achieve the same AUC score as NP-Maximum's in the small-scale (Figure~\ref{fig:bprsmall}) and medium-scale (Figure~\ref{fig:bprmedium}) setting. In addition, of all three cases, ScaledSGD is able to converge to the asymptote of AUC score within the second epoch, while SGD needs more than 2 epochs to converge to the asymptote in the small-scale (Figure~\ref{fig:bprsmall}) and medium-scale (Figure~\ref{fig:bprmedium}) setting. These results demonstrates that ScaledSGD remain highly efficient across small-scale (Figure~\ref{fig:bprsmall}), medium-scale (Figure~\ref{fig:bprmedium}), large-scale(Figure~\ref{fig:bprlarge}) and huge-scale (Figure~\ref{fig:bprhuge}) settings. 

\newpage
\begin{figure}[h!]
    \centering
    \begin{subfigure}{0.5\textwidth}
      \centering
      \includegraphics[width=\linewidth]{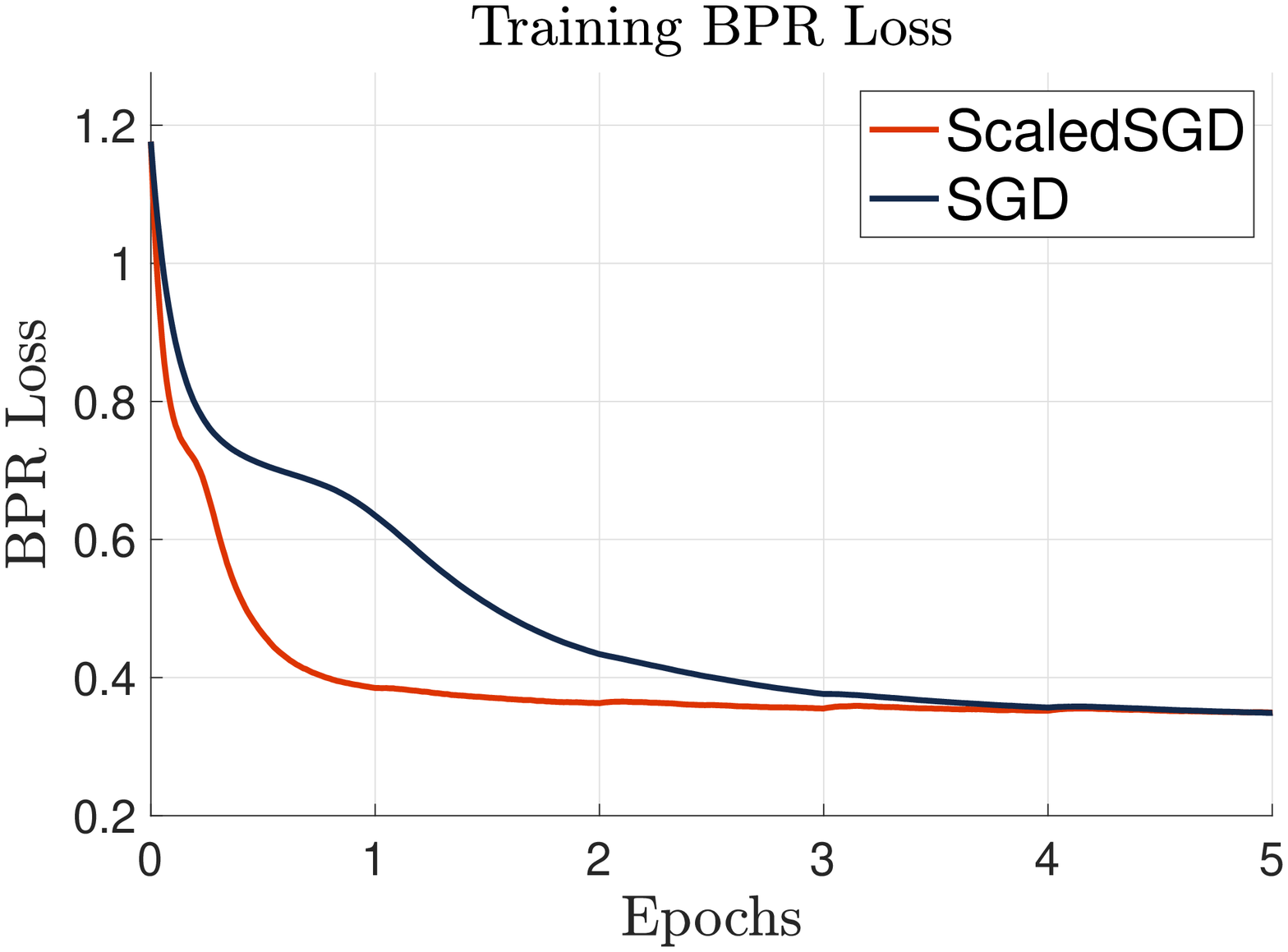}
    \end{subfigure}%
    \begin{subfigure}{0.5\textwidth}
      \centering
      \includegraphics[width=\linewidth]{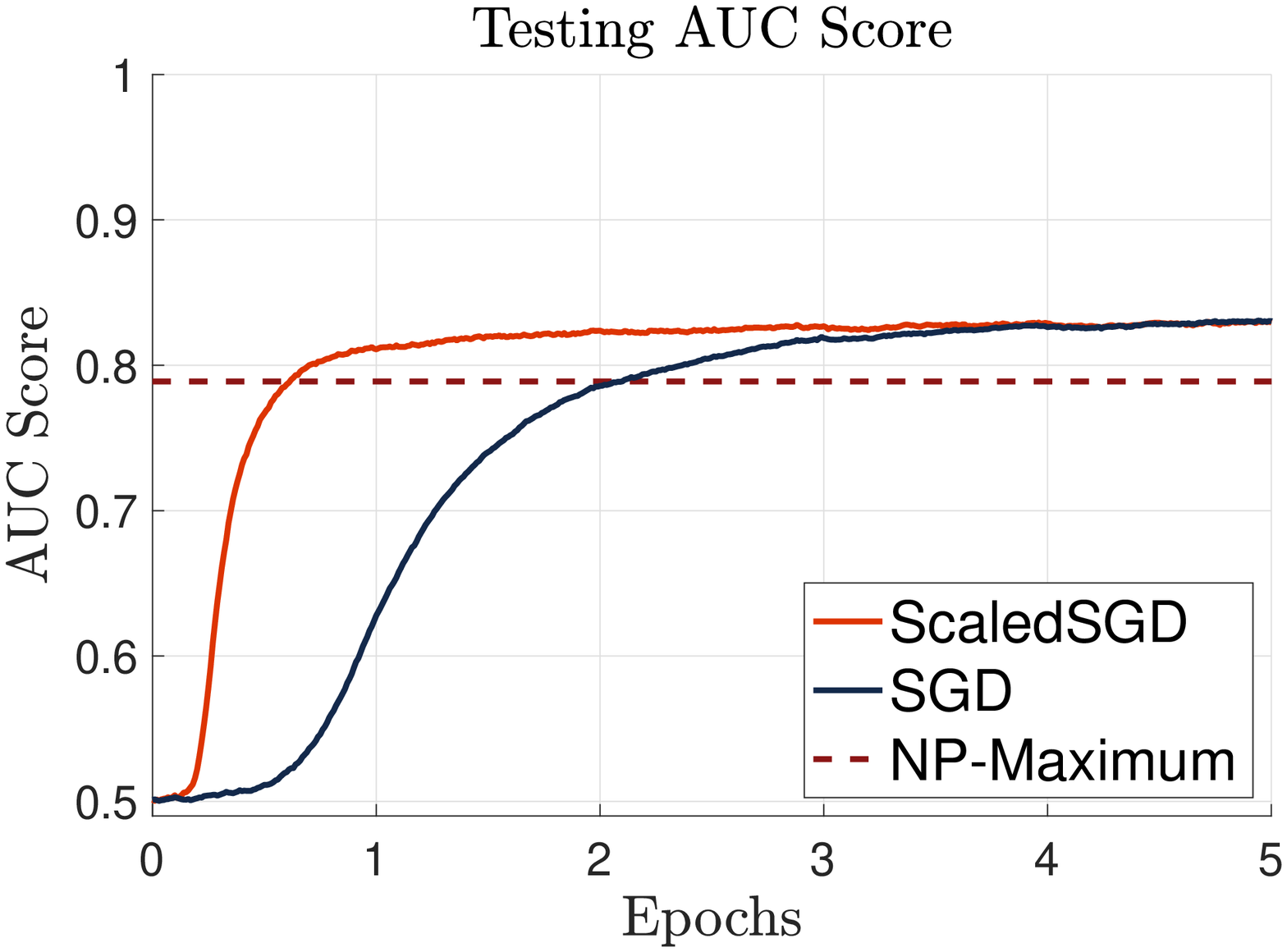}
    \end{subfigure}
    \caption{\textbf{Small-scale item-item collaborative filtering.} (MovieLens-Latest-Small dataset with $|\Omega_{\text{train}}|=$ 1 million and $|\Omega_{\text{test}}|=$ 100,000 pairwise measurements). We compare the training BPR loss and testing AUC score of \ref{scaledsgd} and \ref{sgd}. (\textbf{Left}) Training BPR loss on the training set $\Omega_{\text{train}}$. (\textbf{Right}) Testing AUC score on the test set $\Omega_{\text{test}}$.}
    \label{fig:bprsmall}
\end{figure}

\begin{figure}[h!]
    \vspace{-0.5em}
    \centering
    \begin{subfigure}{0.5\textwidth}
      \centering
      \includegraphics[width=\linewidth]{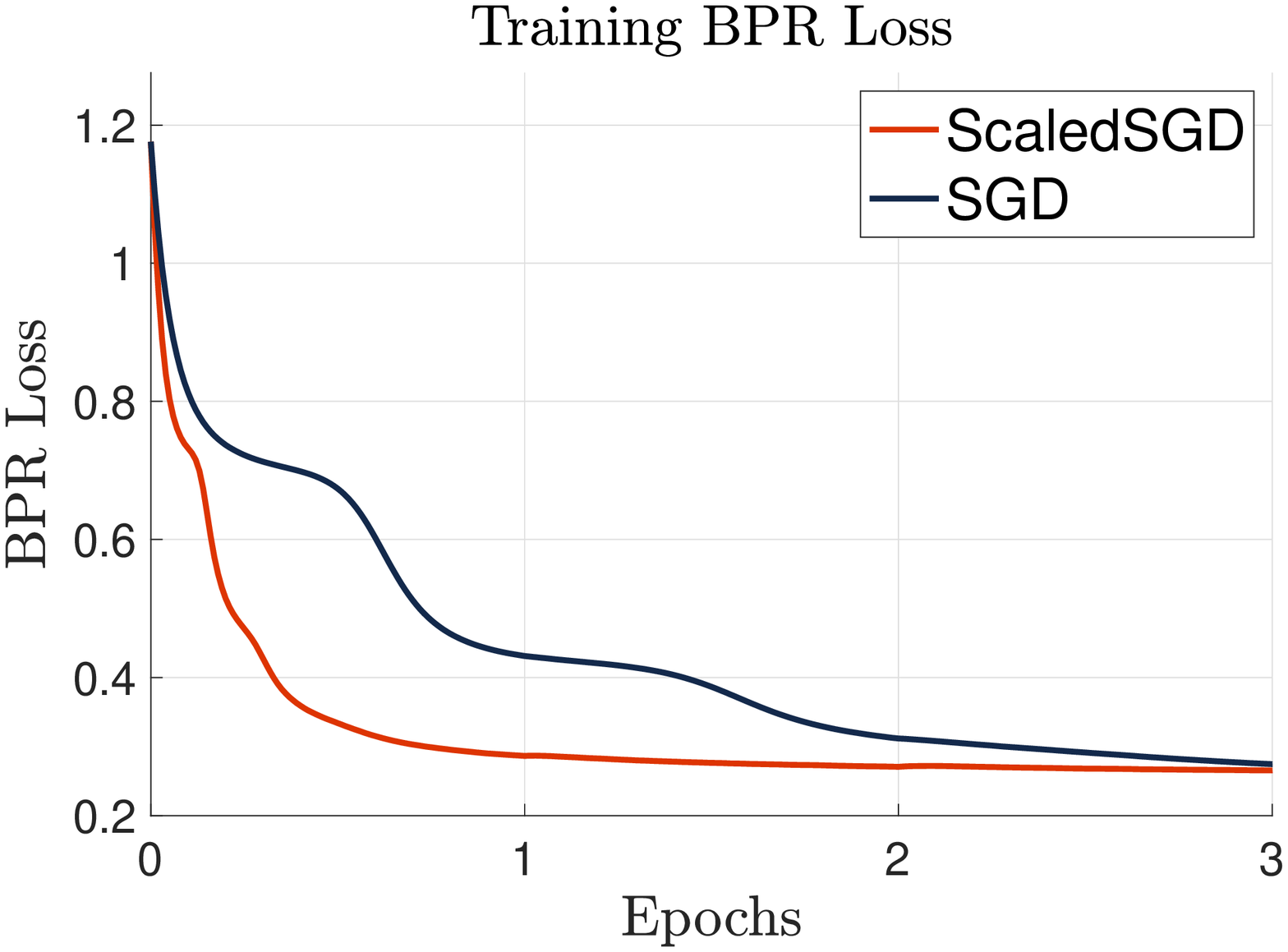}
    \end{subfigure}%
    \begin{subfigure}{0.5\textwidth}
      \centering
      \includegraphics[width=\linewidth]{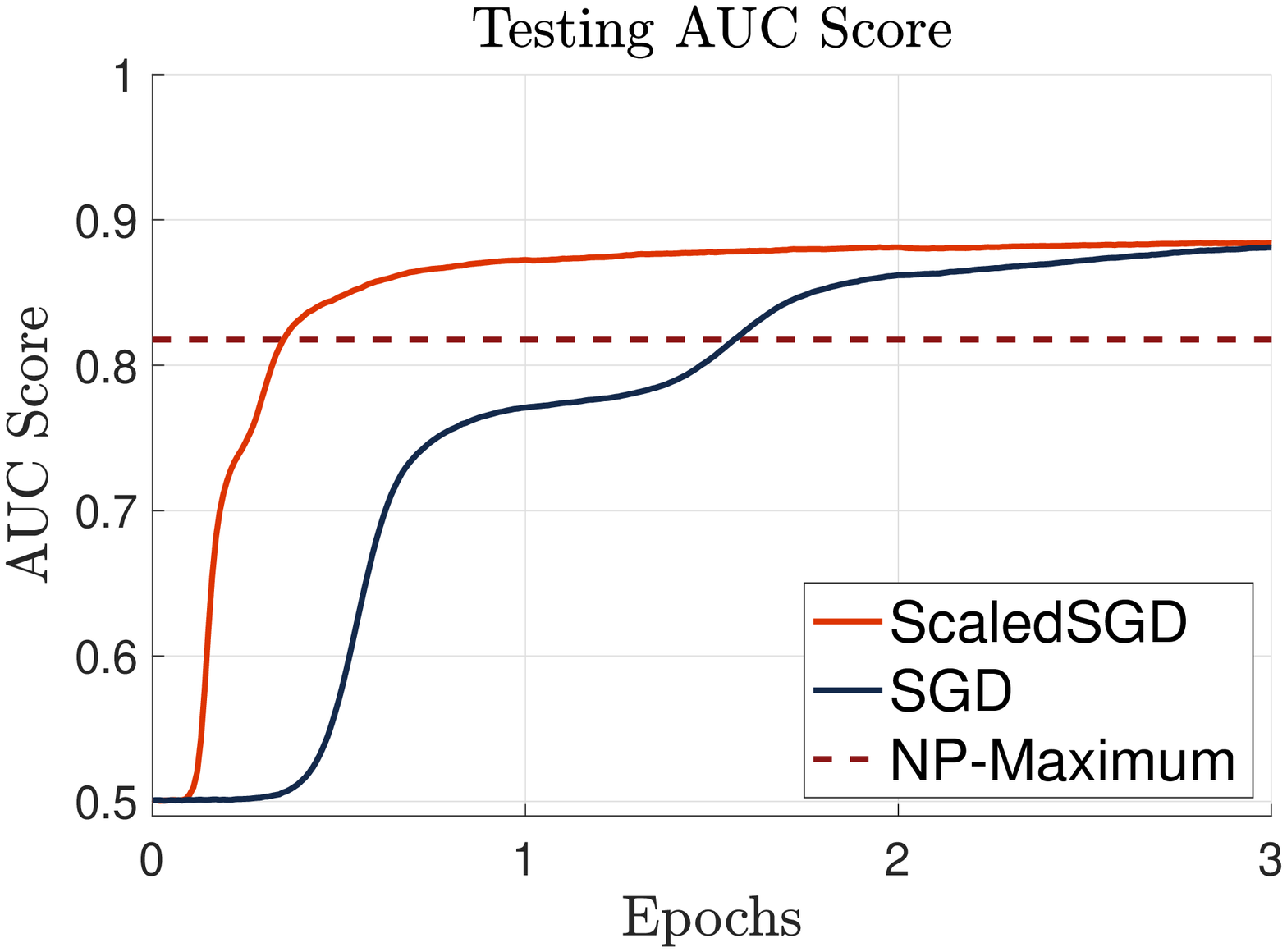}
    \end{subfigure}
    \caption{\textbf{Medium-scale item-item collaborative filtering.} (MovieLens-Latest-Full dataset with $|\Omega_{\text{train}}|=$ 10 million and $|\Omega_{\text{test}}|=$ 1 million pairwise measurements). We compare the BPR loss and AUC score of \ref{scaledsgd} and \ref{sgd} (\textbf{Left}) Training BPR loss on the training set $\Omega_{\text{train}}$. (\textbf{Right}) Testing AUC score on the test set $\Omega_{\text{test}}$.}
    \label{fig:bprmedium}
\end{figure}

\begin{figure}[h!]
    \vspace{-0.5em}
    \centering
    \begin{subfigure}{0.5\textwidth}
      \centering
      \includegraphics[width=\linewidth]{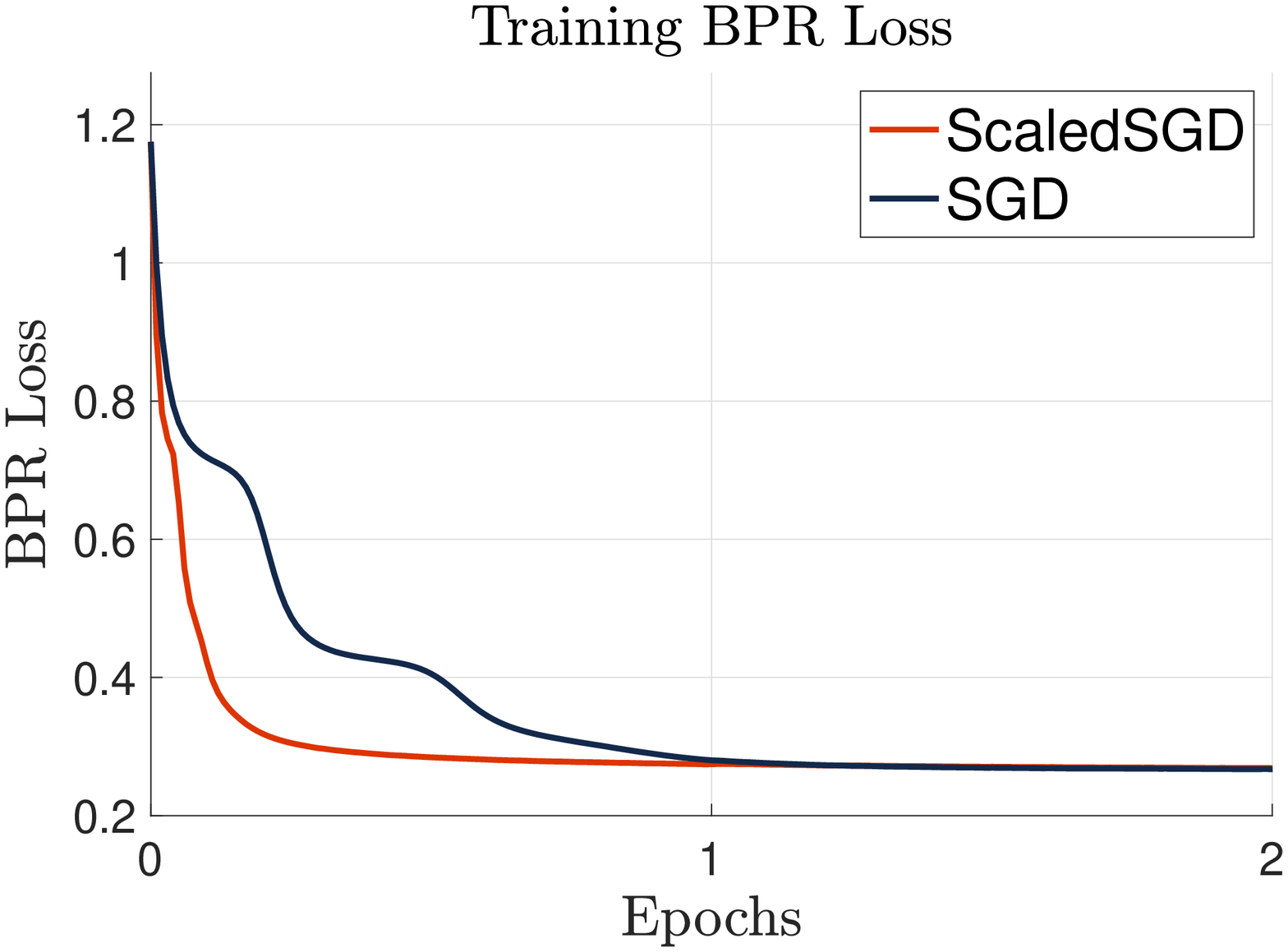}
    \end{subfigure}%
    \begin{subfigure}{0.5\textwidth}
      \centering
      \includegraphics[width=\linewidth]{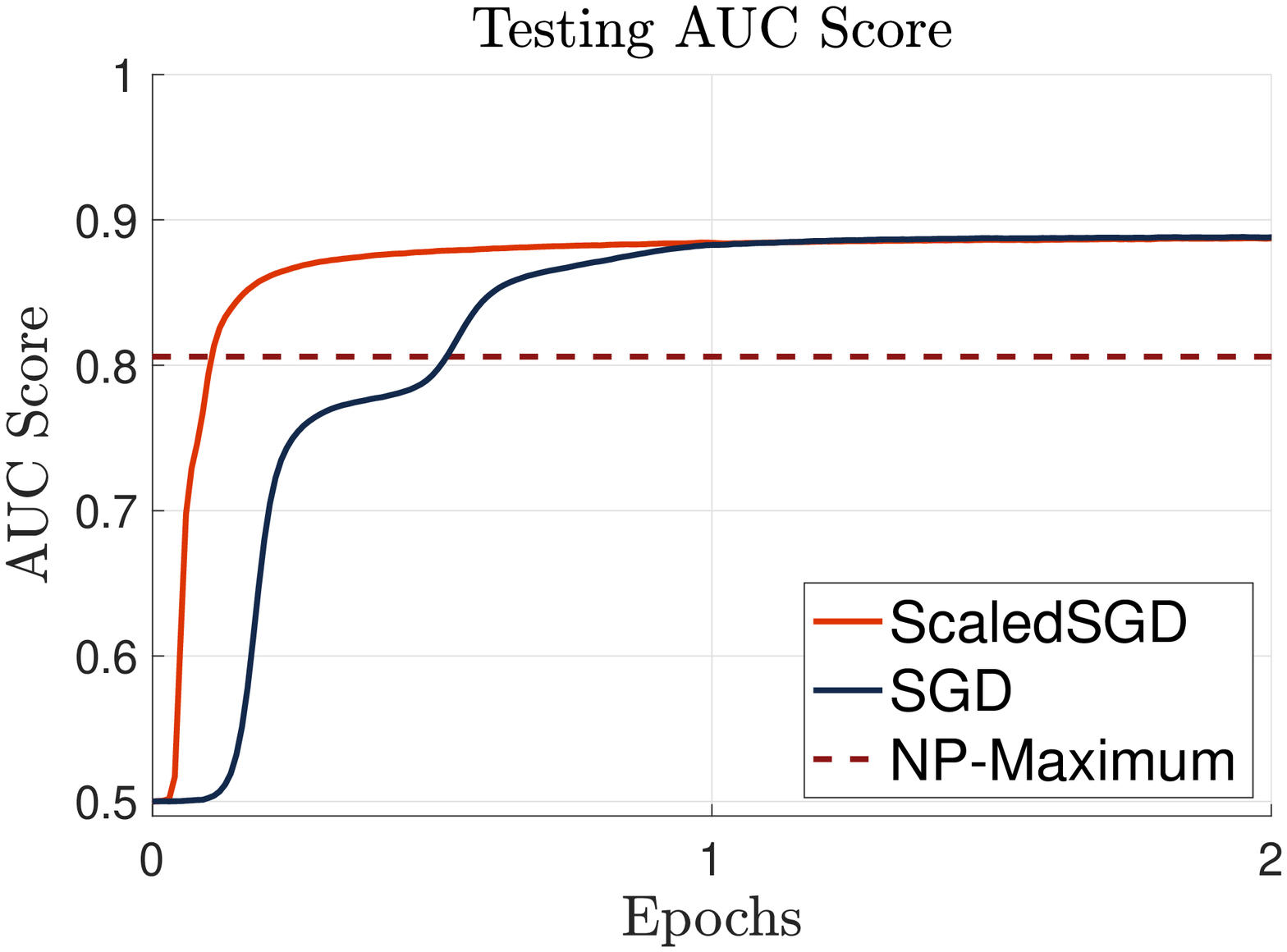}
    \end{subfigure}
    \caption{\textbf{Large-scale item-item collaborative filtering.} (MovieLens-Latest-Full dataset with $|\Omega_{\text{train}}|=$ 30 million and $|\Omega_{\text{test}}|=$ 3 million pairwise measurements). We compare the BPR loss and AUC score of \ref{scaledsgd} and \ref{sgd} (\textbf{Left}) Training BPR loss on the training set $\Omega_{\text{train}}$. (\textbf{Right}) Testing AUC score on the test set $\Omega_{\text{test}}$.}
    \label{fig:bprlarge}
\end{figure}


\section{Proof of the theoretical results} \label{app:proof}

In this section, we show that, in expectation, the search direction
$V=SG(X)(X^{T}X^{-1})$ makes a geometric decrement to both the function
value $f$ and the incoherence $g$. A key idea is to show that the
size of the decrement in $f$ is controlled by the coherence $g_{\max}\ge g_{k}(X)$
of the current iterate, and this motivates the need to decrement $g_{k}$
in order to keep the iterates incoherent. Our key result is that both
decrements are independent of the condition number $\kappa$.

\subsection{Preliminaries}

\global\long\def\vect{\mathrm{vec}}%
 
\global\long\def\eqdef{\overset{\text{def}}{=}}%
 
\global\long\def\tr{\operatorname{tr}}%

We define the inner product between two matrices as $\inner XY\eqdef\tr(X^{T}Y)$,
which induces the Frobenius norm as $\|X\|_{F}=\sqrt{\inner XX}$.
The vectorization $\vect(X)$ is the column-stacking operation that
turns an $m\times n$ matrix into a length-$mn$ vector; it preserves
the matrix inner product $\inner XY=\vect(X)^{T}\vect(Y)$ and the
Frobenius norm $\|\vect(X)\|=\|X\|_{F}$.

We denote $\lambda_{i}(M)$ and $\sigma_{i}(M)$ as the $i$-th eigenvalue
and singular value of a symmetric matrix $M=M^{T}$, ordered from
the most positive to the most negative. We will often write $\lambda_{\max}(M)$
and $\lambda_{\min}(M)$ to index the most positive and most negative
eigenvalues, and $\sigma_{\max}(M)$ and $\sigma_{\min}(M)$ for the
largest and smallest singular values.

Recall for any matrix $V\in\mathbb{R}^{d\times r}$, we define its
local norm with respect to $X\in\mathbb{R}^{d\times r}$ as 
\[
\|V\|_{X}=\|V(X^{T}X)^{1/2}\|_{F},\qquad\|V\|_{X}^{*}=\|V(X^{T}X)^{-1/2}\|_{F}.
\]
Also recall that we have defined the stochastic gradient operator
\begin{equation}
SG(X)=2d^{2}\cdot(XX^{T}-ZZ^{T})_{i,j}\cdot(e_{i}e_{j}^{T}+e_{j}e_{i}^{T})X
\end{equation}
where $(i,j)\sim[d]^{2}$ is selected uniformly at random. This way
SGD is written $X_{+}=X-\alpha SG(X)$ and ScaledSGD is written $X_{+}=X-\alpha SG(X)(X^{T}X)^{-1}$
for step-size $\alpha>0$. 

\subsection{\label{subsec:fconv}Function value convergence}

Recall that \thmref{jin}, due to \citet{jin2016provable}, says that SGD
converges to $\epsilon$ accuracy in $O(\kappa^{4}\log(1/\epsilon))$
iterations, with a four-orders-of-magnitude dependence on the condition
number $\kappa$. By comparison, our main result \thmref{main}
says that ScaledSGD converges to $\epsilon$ accuracy in $O(\log(1/\epsilon))$
iterations, completely independence on the condition number $\kappa$.
In this section, we explain how the first two factors of $\kappa$
are eliminated, by considering the full-batch counterparts of these
two algorithms.

First, consider full-batch gradient descent on the function $f(X)=\|XX^{T}-ZZ^{T}\|_{F}^{2}$.
It follows from the local Lipschitz continuity of $f$ that
\begin{align}
f(X_{t+1}) & \le f(X_{t})-\alpha\inner{\nabla f(X_{t})}{\nabla f(X_{t})}+\alpha\cdot(L/2)\cdot\|\nabla f(X_{t})\|_{F}^{2}\nonumber \\
 & =f(X_{t})-\underbrace{\alpha\|\nabla f(X_{t})\|_{F}^{2}}_{\text{linear progress}}+\alpha^{2}\cdot(L/2)\cdot\underbrace{\|\nabla f(X_{t})\|_{F}^{2}}_{\text{inverse step-size}}\label{eq:gd_fexp}
\end{align}
where $X_{t+1}=X_{t}-\alpha\nabla f(X_{t})$. Here, the linear progress
term determines the amount of progress that can proportionally be
made with a sufficiently small step-size $\alpha$, whereas the inverse
step-size term basically controls how large the step-size can be.
In the case of full-batch gradient descent, it is long known that
an $X_{t}$ that is sufficiently close to $Z$ will satisfy the following
\[
8\lambda_{\min}(Z^{T}Z)\cdot f(X_{t})\le\|\nabla f(X_{t})\|_{F}^{2}\le16\lambda_{\max}(Z^{T}Z)\cdot f(X_{t})
\]
and therefore, taking $\lambda_{\max}(Z^{T}Z)=1$ and $\lambda_{\min}(Z^{T}Z)=\kappa^{-1}$
where $\kappa$ is the condition number, we have linear convergence
\[
f(X_{t+1})\le\left(1-\alpha\cdot8\kappa^{-1}+\alpha^{2}\cdot8L\right)f(X_{t})\le\left(1-\alpha\cdot4\kappa^{-1}\right)f(X_{t})
\]
for step-sizes of $\alpha\le2/(\kappa L)$. Therefore, it follows
from this analysis that full-batch gradient descent takes $T=O(\kappa^{2}\log(1/\epsilon))$
iterations to converge to $\epsilon$-accuracy. In this iteration
count, one factor of $\kappa$ arises from the linear progress term,
which shrinks as $O(\kappa^{-1})$ as $\kappa$ grows large. The second
factor of $\kappa$ arises because the inverse step-size term is a
factor of $\kappa$ larger than the linear progress term, which restricts
the maximum step-size to be no more than $O(\kappa^{-1})$.

The following lemma, restated from the main text, shows that an analogous
analysis for full-batch ScaledGD proves an iteration count of $T=O(\log(1/\epsilon))$
with no dependence on the condition number $\kappa$. In fact, it
proves that full-batch ScaledGD converges like full-batch gradient
descent with a perfect condition number $\kappa=1$.
\begin{lem}[Function descent, \lemref{fdecr} restated]
\label{lem:fdescr2}Let $X,Z\in\R^{n\times r}$ satisfy $\|XX^{T}-ZZ^{T}\|_{F}\le\rho\lambda_{\min}(Z^{T}Z)$
where $\rho<1/2$. Then, the function $f(X)=\|XX^{T}-ZZ^{T}\|_{F}^{2}$
satisfies 
\begin{gather}
|f(X+V)-f(X)-\inner{\nabla f(X)}V|\le\frac{L_{X}}{2}\cdot\|V\|_{X}^{2},\label{eq:Liplike}\\
13\cdot f(X)\le(\|\nabla f(X)\|_{X}^{*})^{2}\le16\cdot f(X),\label{eq:graddom}
\end{gather}
for all $\|V\|_{X}\le C\cdot\sqrt{f(X)}$ with $L_{X}=6+8C+2C^{2}=O(1+C^{2})$. 
\end{lem}

It follows that the iteration $X_{t+1}=X_{t}-\alpha\nabla f(X_{t})(X^{T}_{t}X_{t})^{-1}$
yields
\begin{align}
f(X_{t+1}) & \le f(X_{t})-\alpha\inner{\nabla f(X_{t})}{\nabla f(X_{t})(X^{T}_{t}X_{t})^{-1}}+\alpha\cdot(L_{X}/2)\cdot\|\nabla f(X_{t})(X^{T}_{t}X_{t})^{-1}\|_{X}^{2}\nonumber \\
 & =f(X_{t})-\alpha(\|\nabla f(X_{t})\|_{X}^{*})^{2}+\alpha^{2}\cdot(L_{X}/2)\cdot(\|\nabla f(X_{t})\|_{X}^{*})^{2}\label{eq:gd_fexp-1}\\
 & \le\left(1-\alpha\cdot8+\alpha^{2}\cdot8L_{X}\right)f(X_{t})\le\left(1-\alpha\cdot4\right)f(X_{t})
\end{align}
for step-sizes of $\alpha\le2/L_{X},$ where $L_{X}=6+8(4)+2(4)^{2}$.
Therefore, we conclude that full-batch ScaledGD takes $T=O(\log(1/\epsilon))$
iterations to converge to $\epsilon$-accuracy, as if the condition
number were perfectly $\kappa=1$.

Note that \lemref{fdescr2} has been proved in both \citet{tong2021accelerating}
and \citet{zhang2021preconditioned}. For completeness, we give a
proof inspired by \citet{zhang2021preconditioned}.
\begin{proof}[Proof of \lemref{fdescr2}]
We prove (\ref{eq:Liplike}) via a direct expansion of the quadratic
\begin{gather*}
\underbrace{\|(X+V)(X+V)^{T}-ZZ^{T}\|_{F}^{2}}_{f(X+V)}=\underbrace{\|XX^{T}-ZZ^{T}\|_{F}^{2}}_{f(X)}+\underbrace{2\inner{XX^{T}-ZZ^{T}}{XV^{T}+VX^{T}}}_{\inner{\nabla f(X)}V}\\
+\underbrace{2\inner{XX^{T}-ZZ^{T}}{VV^{T}}+\|XV^{T}+VX^{T}\|_{F}^{2}}_{\frac{1}{2}\inner{\nabla^{2}f(X)[V]}V}+\underbrace{2\inner{VX^{T}+XV^{T}}{VV^{T}}}_{\frac{1}{6}\inner{\nabla^{3}f(X)[V,V]}V}+\underbrace{\|VV^{T}\|_{F}^{2}}_{\frac{1}{24}\inner{\nabla^{4}f(X)[V,V,V]}V}
\end{gather*}
and it then follows by simple counting that
\begin{gather*}
|f(X+V)-f(X)-\inner{\nabla f(X)}V|\le\frac{L_{2}}{2}\|V\|_{X}^{2}+\frac{L_{3}}{6}\|V\|_{X}^{3}+\frac{L_{4}}{24}\|V\|_{X}^{4},\\
L_{2}=4+2\frac{\|XX^{T}-ZZ^{T}\|_{F}}{\lambda_{\min}(X^{T}X)},\quad L_{3}=\frac{24}{\lambda_{\min}(X^{T}X)},\quad L_{4}=\frac{24}{\lambda_{\min}^{2}(X^{T}X)}.
\end{gather*}
Now, from Weyl's inequality that
\[
\lambda_{\min}(X^{T}X)=\lambda_{r}(XX^{T})\ge\lambda_{r}(ZZ^{T})-\|XX^{T}-ZZ^{T}\|_{F}\ge(1-\rho)\cdot\lambda_{\min}(Z^{T}Z)
\]
and therefore $\|XX^{T}-ZZ^{T}\|_{F}/\lambda_{\min}(X^{T}X)\le\rho/(1-\rho)\le1$
because $\rho<1/2$. It follows that $L_{2}\le6$. If $\|V\|_{X}\le C\cdot\sqrt{f(X),}$
then $\frac{L_{3}}{3}\|V\|_{X}\le8C$ and $\frac{L_{4}}{12}\|V\|_{X}^{2}\le2C^{2}$. 

For the upper-bound in (\ref{eq:graddom}), we have simply 
\[
\|\nabla f(X)\|_{X}^{*}=4\|(XX^{T}-ZZ^{T})X(X^{T}X)^{-1/2}\|_{F}\le4\|XX^{T}-ZZ^{T}\|_{F}.
\]
For the lower-bound in (\ref{eq:graddom}), we evoke \citet[Lemma~12]{zhang2021preconditioned}
with RIP constant $\delta=0$ and regularization parameter $\eta=0$
to yield\footnote{Here, we correct for a factor-of-two error in \citet[Lemma~12]{zhang2021preconditioned}. }
\[
\|\nabla f(X)\|_{X}^{*}=\max_{\|Y\|_{X}=1}2\inner{XY^{T}+YX^{T}}{XX^{T}-ZZ^{T}}=4\|XX^{T}-ZZ^{T}\|_{F}\cdot\cos\theta
\]
in which $\cos\theta$ is defined between $XX^{T}-ZZ^{T}$ and the
set $\{XY^{T}+YX^{T}:Y\in\R^{d\times r}\}$, as in
\[
\cos\theta=\max_{Y\in\R^{d\times r}}\frac{\inner{XY^{T}+YX^{T}}{XX^{T}-ZZ^{T}}}{\|XY^{T}+YX^{T}\|_{F}\cdot\|XX^{T}-ZZ^{T}\|_{F}}.
\]
It follows from \citet[Lemma~13]{zhang2021preconditioned} that
\[
\sin\theta=\frac{\|(I-XX^{\dagger})(XX^{T}-ZZ^{T})(I-XX^{\dagger})\|_{F}}{\|XX^{T}-ZZ^{T}\|}\le\frac{1}{\sqrt{2}}\frac{\rho}{\sqrt{1-\rho^{2}}}.
\]
Hence, for $\rho<1/2$, we have 
\[
\frac{(\|\nabla f(X)\|_{X}^{*})^{2}}{\|XX-ZZ^{T}\|_{F}^{2}}=16\cos^{2}\theta\ge16\left(1-\frac{1}{2}\frac{\rho^{2}}{1-\rho^{2}}\right)\ge16\left(1-\frac{1}{6}\right)=\frac{40}{3}>13.
\]
\end{proof}

\subsection{Coherence convergence}

We now explain that ScaledSGD eliminates the last two
factors of $\kappa$ from SGD because it is able to keep its iterates
a factor of $\kappa^{2}$ more incoherent. First, consider regular SGD on the function $f(X)=\|XX^{T}-ZZ^{T}\|_{F}^{2}$.
Conditioning on the current iterate, we have via the local Lipschitz
continuity of $f$:
\begin{align}
\E[f(X_{+})] & \le f(X)-\alpha\inner{\nabla f(X)}{\E[SG(X)]}+\alpha\cdot(L/2)\cdot\E[\|SG(X)\|_{F}^{2}]\nonumber \\
 & =f(X)-\underbrace{\alpha\|\nabla f(X)\|_{F}^{2}}_{\text{linear progress}}+\alpha^{2}\cdot(L/2)\cdot\underbrace{\E[\|SG(X)\|_{F}^{2}]}_{\text{inverse step-size}}\label{eq:sgd_fexp}
\end{align}
where $X_{+}=X-\alpha SG(X)$. In expectation, the linear
progress term of SGD coincides with that of full-batch gradient descent
in (\ref{eq:gd_fexp}). The inverse step-size term, however, is up
to a factor of $d^{2}$ times larger. To see this, observe that
\begin{align*}
\E[\|SG(X)\|_{F}^{2}] & =\frac{1}{d^{2}}\sum_{i=1}^{d}\sum_{j=1}^{d}\|2d^{2}\cdot(XX^{T}-ZZ^{T})_{i,j}\cdot(e_{i}e_{j}^{T}+e_{j}e_{i}^{T})X\|_{F}^{2}\\
 & =4d^{2}\cdot\sum_{i=1}^{d}\sum_{j=1}^{d}(XX^{T}-ZZ^{T})_{i,j}^{2}\cdot\|(e_{i}e_{j}^{T}+e_{j}e_{i}^{T})X\|_{F}^{2}\\
 & =4d^{2}\cdot f(X)\cdot\|(e_{i}e_{j}^{T}+e_{j}e_{i}^{T})X\|_{F}^{2}\le16d^{2}\cdot f(X)\cdot\max_{i}\|e_{i}^{T}X\|_{F}^{2}.
\end{align*}
In a coarse analysis, we can simply bound $\max_{i}\|e_{i}^{T}X\|_{F}^{2}\le\lambda_{\max}(X^{T}X)=O(1)$
to yield
\[
f(X_{+})\le\left(1-\alpha\cdot8\kappa^{-1}+16d^{2}\cdot\alpha^{2}\right)f(X)\le\left(1-\alpha\cdot4\kappa^{-1}\right)f(X)
\]
for step-sizes of $\alpha\le4/(\kappa d^{2})$. Hence, we conclude
that it takes $T=O(\kappa^{2}d^{2}\log(1/\epsilon))$ iterations to
converge to $\epsilon$-accuracy, with an epoch of $d^{2}$ iterations
of SGD essentially recreating a single iteration of full-batch gradient
descent. Unfortunately, the matrix is already fully observed after $d^2$ iterations, and so this result is essentially vacuous. 

Here, \citet{jin2016provable} pointed out that the term $h_{\max}=\max_{i}\|e_{i}^{T}X\|_{F}^{2}$
measures the \emph{coherence} of the $d\times r$ iterate $X$, and
can be as small as $O(1/d)$ for small values of rank $r=O(1)$. Conditioned
on the current iterate $X$, they observed that the function $h_{i}(X)=\|e_{i}^{T}X\|_{F}^{2}$
converges towards a finite value in expectation
\begin{align*}
\E[h_{i}(X_{+})] & \le\left(1-\alpha\cdot8\kappa^{-1}\right)h_{i}(X)+\alpha\cdot8\sqrt{h_{i}(X)h_{i}(Z)}+\alpha^{2}/2\cdot\E[\|e_{i}^{T}SG(X)\|_{F}^{2}]\\
 & \le\left(1-\alpha\cdot8\kappa^{-1}\right)h_{i}(X)+\alpha\cdot8\sqrt{h_{i}(X)h_{i}(Z)}+\alpha^{2}\cdot O(d^{2}h_{\max}^{2}).
\end{align*}
Let us define $\gamma$ as the \emph{ratio} between the coherences
of the ground truth $Z$ and the iterate $X$:
\[
\gamma=\frac{\max_{i}\|e_{i}^{T}X\|_{F}^{2}}{\max_{j}\|e_{j}^{T}Z\|_{F}^{2}}=\frac{\max_{i}h_{i}(Z)}{\max_{j}h_{j}(Z)}\quad\iff\quad\max_{j}\|e_{j}^{T}Z\|_{F}^{2}\le\gamma^{-1}\cdot h_{\max}.
\]
Crucially, we require $\gamma=\kappa^{2}$ in order for $h_{i}(X)$ to converge towards $\frac{1}{2}h_{\max}$
in expectation:
\begin{align*}
\E[h_{i}(X_{+})-\frac{1}{2}h_{\max}] & \le\left(1-\alpha\cdot8\kappa^{-1}\right)h_{i}(X)+\alpha\cdot4\gamma^{-1/2}h_{\max}-\frac{1}{2}h_{\max}\\
 & \le\left(1-\alpha\cdot8\kappa^{-1}\right)\left[h_{i}(X)-\left(\frac{1-\alpha\cdot8\gamma^{-1/2}}{1-\alpha\cdot8\kappa^{-1}}\right)\frac{1}{2}h_{\max}\right].
\end{align*}
As a consequence, we conclude that, while SGD is able to keep its
iterates $X$ incoherent, their actual coherence $h_{\max}=\max_{i}\|e_{i}^{T}X\|_{F}^{2}$
is up to a factor of $\kappa^{2}$ worse than the coherence $\max_{j}\|e_{j}^{T}Z\|_{F}^{2}$
of the ground truth $Z$.

Using a standard supermartingale argument, \citet{jin2016provable}
extended the analysis above to prove that if the ground truth $Z$
has coherence $\max_{j}\|e_{j}^{T}Z\|_{F}^{2}=O(1/d)$, then the SGD
generates iterates $X$ that have coherence $\max_{i}\|e_{i}^{T}X\|_{F}^{2}\le h_{\max}=O((\kappa^{2}/d)\log d)$,
which is two factors worse in $\kappa$ as expected. Combined, this
proves that SGD converges to $\epsilon$ accuracy in $T=O(\kappa^{4}dr\log(d/\epsilon))$
iterations with the step-size of $\alpha=O(\kappa^{-1}d^{-1}h_{\max}^{-1})$
and iterate coherence $h_{\max}=O((\kappa^{2}/d)\log d)$, which is
another two factors of $\kappa$ worse than full-batch gradient descent. 

The following lemma, restated from the main text, shows that an analogous
analysis for ScaledSGD proves that the algorithm maintains iterates
$X$ whose coherences have no dependence on $\kappa$. Here, we
need to define a different incoherence function $g_{i}(X)=\|e_{i}X(X^{T}X)^{-1/2}\|^{2}\equiv(\|e_{i}X\|_{X}^{*})^{2}$
in order to ``stochastify'' our previous analysis for full-batch
ScaledGD. Surprisingly, the factors of $(X^{T}X)^{-1}$ in both the
new definition of $g_{i}(X)$ and the search direction $SG(X)(X^{T}X)^{-1}$
do not hurt incoherence, but in fact improves it. 
\begin{lem}[Coherence descent, \lemref{gdecr} restated]
\label{lem:gdecr2}Let $X,Z\in\R^{n\times r}$ satisfy $\|XX^{T}-ZZ^{T}\|_{F}\le\rho\lambda_{\min}(Z^{T}Z)$
where $\rho<1/2$. Then, the functions $f(X)=\|XX^{T}-ZZ^{T}\|_{F}^{2}$
and $g_{k}(X)=e_{k}^{T}X(X^{T}X)^{-1}X^Te_{k}$ satisfy
\begin{gather*}
|g_{k}(X+V)-g_{k}(X)-\inner V{\nabla g_{k}(X)}|\le\frac{5(\|V\|_{X}^{*})^{2}}{1-2\|V\|_{X}^{*}},\\
\inner{\nabla g_{k}(X)}{\nabla f(X)(X^{T}X)^{-1}}\ge\left[\frac{1-2\rho}{1-\rho}g_{k}(X)-\frac{1}{1-\rho}\sqrt{g_{k}(X)g_{k}(Z)}\right].
\end{gather*}
\end{lem}

Conditioning on $X$, we have for the search direction
$V=SG(X)(X^{T}X)^{-1}$ and $X_{+}=X+V$
\begin{align}
\E g_{k}(X_{+}) & \le g_{k}(X)-\alpha\inner{\nabla g_{k}(X)}{\E[V]}+\alpha^{2}\cdot\E\left[\frac{(\|V\|_{X}^{*})^{2}}{1-2\|V\|_{X}^{*}}\right]\nonumber \\
 & \le\left(1-\zeta\alpha\right)g_{k}(X)+\alpha\cdot\frac{\zeta}{2}g_{\max}\qquad\text{ for }\alpha=O(\rho^{-1}d^{-2})\label{eq:gdecr-1}
\end{align}
where $\zeta=\frac{1-2\rho}{1-\rho}$. It then follows that $g_{k}(X_{+})$
converges geometrically towards $\frac{1}{2}g_{\max}$ in expectation,
with a convergence rate $(1-\zeta\alpha)$ that is independent of
the condition number $\kappa$:
\[
\E\left[g_{k}(X_{+})-\frac{1}{2}g_{\max}\right]\le\left[\left(1-\zeta\alpha\right)g_{k}(X)+\alpha\cdot\frac{\zeta}{2}g_{\max}\right]-\frac{1}{2}g_{\max}\le\left(1-\zeta\alpha\right)\left[g_{k}(X)-\frac{1}{2}g_{\max}\right].
\]
Before we prove \lemref{gdecr2}, we first need to prove a simple
claim.
\begin{lem}[Change of norm]
The local norm $\|V\|_{X}^{*}=\|V(X^{T}X)^{-1/2}\|$ satisfies
\[
\frac{(\|V\|_{X}^{*})^{2}}{1+2\|Y-X\|_{X}^{*}+(\|Y-X\|_{X}^{*})^{2}}\le(\|V\|_{Y}^{*})^{2}\le\frac{(\|V\|_{X}^{*})^{2}}{1-2\|Y-X\|_{X}^{*}}.
\]
\end{lem}

\begin{proof}
The upper-bound follows because 
\[
\tr(VP_{Y}V^{T})=\tr(VP_{X}^{1/2}[P_{X}^{-1/2}P_{Y}P_{X}^{-1/2}]P_{X}^{1/2}V^{T})\le\tr(VP_{X}V^{T})/\lambda_{\min}[P_{X}^{1/2}P_{Y}^{-1}P_{X}^{1/2}]
\]
where $P_{Y}=(Y^{T}Y)^{-1}$ and $P_{X}=(X^{T}X)^{-1}$ and therefore
\[
P_{X}^{1/2}P_{Y}^{-1}P_{X}^{1/2}\succeq I+P_{X}^{1/2}[X^{T}(Y-X)+(Y-X)^{T}X]P_{X}^{1/2}
\]
and $\sigma_{\max}[P_{X}^{1/2}[X^{T}(Y-X)P_{X}^{1/2}]\le\|Y-X\|_{X}$
because $XP_{X}^{1/2}$ is orthonormal. The lower-bound follows similarly.
\end{proof}
We are ready to prove \lemref{gdecr2}.
\begin{proof}[Proof of \lemref{gdecr2}]
It follows from the intermediate value version of Taylor's theorem
that there exists some $\tilde{X}=X+tV$ with $t\in[0,1]$ that
\[
g_{i}(X+V)-g_{i}(X)-\inner{\nabla g_{i}(X)}V=\frac{1}{2}\inner{\nabla^{2}g_{i}(\tilde{X})[V]}V.
\]
Let $P=(X^{T}X)^{-1}$ and $U=e_{i}e_{i}^{T}(I-XPX^{T})V$ and $G=VPX^{T}e_{i}e_{i}^{T}$.
By direct computation, we have 
\begin{align*}
\frac{1}{2}\inner{\nabla g_{i}(X)}V & =\inner{(I-XPX^{T})e_{i}e_{i}^{T}XP}V=\inner U{XP}=\inner{I-XPX^{T}}G,\\
\frac{1}{2}\inner{\nabla^{2}g_{i}(X)[V]}V & =\inner{UP-XP(U^{T}X+X^{T}U)P}V-\inner{(I-XPX^{T})(G+G^{T})XP}V,
\end{align*}
by differentiating $XP$ and $XPX^{T}$ respectively. A coarse count
yields $\frac{1}{2}\inner{\nabla^{2}g_{i}(X)[V]}V\le5(\|V\|_{X}^{*})^{2}$
and therefore
\[
|g_{i}(X+V)-g_{i}(X)-\inner{\nabla g_{i}(X)}V|\le5\|V\|_{X+tV}^{2}\le\frac{5(\|V\|_{X}^{*})^{2}}{1-2t\|V\|_{X}^{*}}\le\frac{5(\|V\|_{X}^{*})^{2}}{1-2\|V\|_{X}^{*}},
\]
which is the first claim. Now,  observe that the two functions have gradient
\[
\nabla g_{i}(X)=2[I-X(X^{T}X)^{-1}X^{T}]e_{i}e_{i}^{T}X(X^{T}X)^{-1},\qquad\nabla f(X)=4(XX^{T}-ZZ^{T})X.
\]
Directly substituting yields
\begin{align*}
\frac{1}{8}\inner{\nabla g_{i}(X)}{\nabla f(X)(X^{T}X)^{-1}} & =\inner{[I-X(X^{T}X)^{-1}X^{T}]e_{i}e_{i}^{T}X(X^{T}X)^{-1}}{(XX^{T}-ZZ^{T})X(X^{T}X)^{-1}}\\
 & =e_{i}^{T}X(X^{T}X)^{-2}X^{T}ZZ^{T}X(X^{T}X)^{-1}X^{T}e_{i}-e_{i}^{T}ZZ^{T}X(X^{T}X)^{-2}X^{T}e_{i}
\end{align*}
where the second line follows from the fact that
\[
\inner{[I-X(X^{T}X)^{-1}X^{T}]e_{i}e_{i}^{T}X(X^{T}X)^{-1}}{XX^{T}X(X^{T}X)^{-1}}=0.
\]
The second claim follows from the following three identities
\begin{gather}
\lambda_{\min}(X^{T}X)\ge(1-\rho)\lambda_{\min}(Z^{T}Z)\label{eq:incoh1}\\
e_{i}^{T}ZZ^{T}X(X^{T}X)^{-2}X^{T}e_{i}\le\frac{1}{1-\rho}\|e_{i}^{T}X\|_{X}^{*}\cdot\|e_{i}^{T}Z\|_{Z}^{*}\label{eq:incoh2}\\
e_{i}^{T}X(X^{T}X)^{-2}X^{T}ZZ^{T}X(X^{T}X)^{-1}X^{T}e_{i}\ge\frac{1-2\rho}{1-\rho}\cdot(\|e_{i}^{T}X\|_{X}^{*})^{2}\label{eq:incoh3}
\end{gather}
We have (\ref{eq:incoh1}) via Weyl's inequality:
\[
\lambda_{\min}(X^{T}X)=\lambda_{r}(XX^{T})=\lambda_{r}(ZZ^{T}+XX^{T}-ZZ^{T})\ge\lambda_{r}(ZZ^{T})-\|XX^{T}-ZZ^{T}\|_{F}.
\]
We have (\ref{eq:incoh2}) by rewriting 
\begin{align*}
    e_{i}^{T}ZZ^{T}X(X^{T}X)^{-2}X^{T}e_{i}&=(e_{i}^{T}P)(P^{T}ZZ^{T}X(X^{T}X)^{-2}X^{T}Q)(Q^{T}e_{i})\\
    &\le\|e_{i}^{T}P\|\|ZZ^{T}X(X^{T}X)^{-2}X^{T}\|\|e_{i}^{T}Q\|
\end{align*}
and rewriting $ZZ^{T}=XX^{T}-E$ where $E=XX^{T}-ZZ^{T}$ and evoking
(\ref{eq:incoh1}) as in 
\[
\|ZZ^{T}X(X^{T}X)^{-2}X^{T}\|\le\underbrace{\|XX^{T}X(X^{T}X)^{-2}X^{T}\|}_{=1}+\underbrace{\|E\|\cdot\|X(X^{T}X)^{-2}X^{T}\|}_{\le\rho/(1-\rho)}
\]
and noting that $1+\frac{\rho}{1-\rho}=\frac{1}{1-\rho}$. We have
(\ref{eq:incoh3}) again by substituting $ZZ^{T}=XX^{T}-E$ 
\begin{align*}
e_{i}^{T}X(X^{T}X)^{-2}X^{T}ZZ^{T}X(X^{T}X)^{-1}X^{T}e_{i} & =e_{i}^{T}X(X^{T}X)^{-1}X^{T}e_{i}-e_{i}^{T}X(X^{T}X)^{-2}X^{T}EX(X^{T}X)^{-1}X^{T}e_{i}\\
 & \ge e_{i}^{T}X(X^{T}X)^{-1}X^{T}e_{i}\cdot(1-\underbrace{\|E\|\cdot\|X(X^{T}X)^{-2}X^{T}\|}_{\le\rho/(1-\rho)})
\end{align*}
and then noting that $1-\frac{\rho}{1-\rho}=\frac{1-2\rho}{1-\rho}$. 
\end{proof}

\subsection{Proof of the main result}

In the previous two subsections, we showed that when conditioned on the current iterate $X_t$,  a single step of ScaledSGD $X_{t+1}=X_{t}-\alpha SG(X_{t})(X_{t}^{T}X_{t})^{-1}$
is expected to geometrically
converge both the loss function $f$ and each of the incoherence functions
$g_{i}$, as in 
\[
\E[f(X_{t+1})]\le(1-\alpha)f(X_{t}),\quad\E[g_{i}(X_{t+1})-\frac{1}{2}g_{\max}]\le\left(1-\frac{1-2\rho}{1-\rho}\alpha\right)\left[g_{i}(X_{t})-\frac{1}{2}g_{\max}\right].
\]
In this section, we will extend this geometric convergence to $T$
iterations of ScaledSGD. Our key challenge is to verify that the \emph{variances}
and \emph{maximum deviations} of the sequences $f(X_{0}),f(X_{1}),\dots,f(X_{T})$
and $g_{i}(X_{0}),g_{i}(X_{1}),\dots,g_{i}(X_{T})$ have the right
dependence on the dimension $d$, the radius $\rho$, the condition
number $\kappa$, the maximum coherence $g_{\max}$, and the iteration
count $t$, so that $T$ iterations of ScaledSGD with a step-size
of $\alpha\le c/[(g_{\max}+\rho)d^{2}\log d]$ results in no more
than a multiplicative factor of 2 deviation from expectation. Crucially,
we must check that the cumulated deviation over $T$ iterations does
not grow with the iteration count $T$, and that the convergence rate
is independent of the condition number $\kappa$. We emphasize that
the actual approach of our proof via the Azuma--Bernstein inequality
is textbook; to facilitate a direct comparison with SGD, we organize
this section to closely mirror \citet{jin2016provable}'s proof of
\thmref{jin}. 

\global\long\def\dev{\mathrm{dev}}%
\global\long\def\var{\mathrm{var}}%
\global\long\def\fE{\mathfrak{E}}%
\global\long\def\Pr{\mathbf{Pr}}%
\global\long\def\Var{\mathbf{Var}}%
Let $f_{\max}=\rho^{2}\cdot\lambda_{\min}^{2}(Z^{T}Z)$ and $g_{\max}=\frac{16}{(1-2\rho)^{2}}\max_{i}g_{i}(Z)$.
Our goal is to show that the following event happens with probability
$1-T/d^{10}$: 
\begin{equation}
\fE_{t}\equiv\left\{ f(X)\le\left(1-\frac{\alpha}{2}\right)^{t}\cdot f_{\max},\quad\max_{i}g_{i}(X_{\tau})\le g_{\max}\quad\text{for all }\tau\le t\right\} ,\label{eq:event-1-1}
\end{equation}
Equivalently, conditioned on event $\fE_{t}$, we want to prove that
the probability of failure at time $t+1$ is $\delta\equiv1/d^{10}$.
We split this failure event into a probability of $\frac{\delta}{2}$
that the function value clause fails to hold, as in $f(X_{t+1})>\left(1-\alpha/2\right)^{t}\cdot f_{\max}$,
and a probability of $\frac{\delta}{2d}$ that any one of the $d$
incoherence caluses fails to hold, as in $g_{i}(X_{t+1})>g_{\max}$.
Then, cumulated over $T$ steps, the total probability of failure
would be $T\cdot\delta=T/d^{10}$ as desired. 

We begin by setting up a supermartingale on the loss function $f$.
Our goal is to show that the variance and the maximum deviation of
this supermartingale have the right dependence on $\alpha,d,\rho,\kappa,g_{\max}$,
so that a step-size of $\alpha\le c/(g_{\max}d^{2}\log d)$ with a
sufficiently small $c>0$ will keep the cumulative deviations over
$T$ iterations within a factor of 2. Note that, by our careful choice
of the coherence function $g_{i}$, the following statement for ScaledSGD
match the equivalent statements for SGD with a perfect condition number
$\kappa=1$; see \citet[Section~B.2]{jin2016provable}.
\begin{lem}[Function value supermartingale]
\label{lem:fmartin}Let $f(X)=\|XX^{T}-ZZ^{T}\|_{F}^{2}.$ Define
$f_{\max}=\rho^{2}\cdot\lambda_{\min}^{2}(Z^{T}Z)$ and $g_{\max}=\frac{16}{(1-2\rho)^{2}}\max_{i}g_{i}(Z)$.
For a sufficiently small $c>0$, the following with learning rate
$\alpha\le c/(g_{\max}d^{2}\log d)$ is a supermartingale 
\[
F_{t}=\left(1-\alpha\right)^{-t}f(X_{t})\cdot1_{\fE_{t}},
\]
meaning that $\E[F_{t+1}|X_{t},\dots,X_{0}]\le F_{t}$ holds for all
$t\in\{0,1,2,\dots\}$. Moreover, there exist sufficiently large constants
$C_{\dev},C_{\var}>0$ such that the following holds with probability
one:
\begin{gather*}
\E[F_{t}|X_{t-1},\dots,X_{0}]-F_{t}\le C_{\dev}\cdot\alpha\cdot d^{2}\cdot g_{\max}\cdot\left(1-\alpha\right)^{-t}\left(1-\frac{\alpha}{2}\right)^{t}f_{\max},\\
\Var[F_{t}|X_{t-1},\dots,X_{0}]\le C_{\var}\cdot\alpha^{2}\cdot d^{2}\cdot g_{\max}\cdot\left(1-\alpha\right)^{-2t}\left(1-\frac{\alpha}{2}\right)^{2t}f_{\max}^{2}.
\end{gather*}
\end{lem}

\begin{proof}
The proof is technical but straightforward; it is deferred to \secref{martingales}.
\end{proof}
\begin{lem}[Function value concentration]
\label{lem:fconc}Let the initial point satisfy $f(X_{0})\le\frac{1}{2}f_{\max}$.
Then, there exists a sufficiently small constant $c>0$ such that
for all learning rates $\alpha<c/(g_{\max}d^{2}\log d)$, we have
\[
\Pr\left(f_{i}(X_{t+1})1_{\fE_{t}}>\left(1-\frac{\alpha}{2}\right)^{t}f_{\max}\right)=\Pr\left(\fE_{t}\cap\left\{ f_{i}(X_{t+1})>\left(1-\frac{\alpha}{2}\right)^{t}f_{\max}\right\} \right)\le\frac{1}{2d^{10}}.
\]
\end{lem}

\begin{proof}
Let $\sigma^{2}=\sum_{\tau=1}^{t}\Var[F_{\tau}|X_{\tau-1},\dots,X_{0}]$
and let $R$ satisfy $\E[F_{\tau}|X_{\tau-1},\dots,X_{0}]\le X_{\tau}+R$
almost surely for all $\tau\in\{1,2,\dots,t\}$. Recall via the standard
Azuma--Bernstein concetration inequality for supermartingales that
$\Pr\left(F_{t}\ge F_{0}+s\right)\le\exp\left(-\frac{s^{2}/2}{\sigma^{2}+Rs/3}\right)$.
Equivalently, there exists a large enough constant $C>0$ in $s=C\cdot\left(1-\alpha\right)^{t}\left[\sqrt{\sigma^{2}\log d}+R\log d\right]$
such that the following is true
\[
\Pr\left(f(X_{t+1})1_{\fE_{t}}\ge\left(1-\alpha\right)^{t}f(X_{0})+s\right)\le\frac{1}{2d^{10}}.
\]
Given that $f(X_{0})\le\frac{1}{2}f_{\max}$ and therefore $\left(1-\alpha\right)^{t}f(X_{0})\le\frac{1}{2}\left(1-\frac{\alpha}{2}\right)^{t}\cdot f_{\max}$
holds by hypothesis, the desired claim is true if we can show that
$s\le\frac{1}{2}\left(1-\frac{\alpha}{2}\right)^{t}\cdot f_{\max}$.
Crucially, we observe that the variance term in $s$ does not blow-up
with time $t$
\begin{align*}
\left(1-\alpha\right)^{2t}\cdot\sigma^{2} & \le f_{\max}^{2}\cdot C_{\var}\cdot d^{2}\cdot g_{\max}\cdot\alpha^{2}\cdot\sum_{\tau=1}^{t}\left(1-\alpha\right)^{2t-2\tau}\left(1-\frac{\alpha}{2}\right)^{2\tau}\\
 & =\left(1-\frac{\alpha}{2}\right)^{2t}f_{\max}^{2}\cdot C_{\var}\cdot d^{2}\cdot g_{\max}\cdot\alpha^{2}\cdot\sum_{\tau=1}^{t}\left(\frac{1-\alpha}{1-\alpha/2}\right)^{2t-2\tau}\\
 & \le\left(1-\frac{\alpha}{2}\right)^{2t}f_{\max}^{2}\cdot C_{\var}\cdot d^{2}\cdot g_{\max}\cdot\alpha
\end{align*}
due to the geometric series expansion $\sum_{\tau=0}^{t}\beta^{t-\tau}=(1-\beta^{t+1})/(1-\beta).$
Substituting the deviations term, choosing a step-size $\alpha\le c/(\rho d^{2}\log d)$
for sufficiently small $c$ yields
\begin{align*}
s & =\left(1-\frac{\alpha}{2}\right)^{t}\cdot\left[\sqrt{C_{\var}\cdot d^{2}\cdot g_{\max}\cdot\alpha\cdot\log d}+C_{\dev}\cdot d^{2}\cdot g_{\max}\cdot\alpha\cdot\log d\right]\cdot f_{\max}\\
 & \le\frac{1}{2}\left(1-\frac{\alpha}{2}\right)^{t}f_{\max}.
\end{align*}
\end{proof}
We now set up a supermartingale on each of the incoherence functions
$g_{i}$. Again, our goal is to show that the variance and the maximum
deviation of this supermartingale have the right dependence on $\alpha,d,\rho,\kappa,g_{\max}$,
so that a step-size of $\alpha\le c/(\rho d^{2}\log d)$ with a sufficiently
small $c>0$ will keep the cumulative deviations over $T$ iterations
within a factor of 2. Note that \citet[Section~B.2]{jin2016provable}'s
proof tracks a different function $h_{i}(X)=e_{i}^{T}XX^{T}e_{i}$
that is substantially simpler, but pays a penalty of two to three
factors of the condition number $\kappa$.
\begin{lem}[Incoherence supermartingale]
\label{lem:gmartin}Let $g_{i}(X)=e_{i}^{T}X(X^{T}X)^{-1}X^{T}e_{i}$.
Define $g_{\max}=\frac{16}{(1-2\rho)^{2}}\max_{i}g_{i}(Z)$. For a
fixed $i\in[n]$ with sufficiently small $c>0$, the following with
learning rate $\alpha<c/(\rho d^{2}\log d)$ is a supermartingale
\[
G_{it}=\left(1-\zeta\cdot\alpha\right)^{-t}\left(g(X_{t})\cdot1_{\fE_{t-1}}-\frac{\zeta}{2}g_{\max}\right)\text{ where }\zeta=\frac{1-2\rho}{1-\rho}<1,
\]
meaning that $\E[G_{i(t+1)}|X_{t},\dots,X_{0}]\le G_{it}$ holds for
all $t\in\{0,1,2,\dots\}$. Moreover, there exist sufficiently large
constants $C_{dev},C_{var}>0$ with no dependence on $g_{\max},n,t$
such that 
\begin{gather*}
\E[G_{it}|X_{t-1},\dots,X_{0}]-G_{it}\le C_{\dev}\cdot\alpha\cdot d^{2}\cdot\rho\cdot\left(1-\zeta\cdot\alpha\right)^{-t}g_{\max},\\
\Var[G_{it}|X_{t-1},\dots,X_{0}]\le C_{\var}\cdot\alpha^{2}\cdot d^{2}\cdot\rho^{2}\cdot\left(1-\zeta\cdot\alpha\right)^{-2t}g_{\max}^{2}.
\end{gather*}
\end{lem}

\begin{proof}
The proof is long but straightforward; it is deferred to \secref{martingales}.
\end{proof}
\begin{lem}[Incoherence concentration]
\label{lem:gconc}Let the initial point satisfy $\max_{i}g_{i}(X_{0})\le\frac{1}{2}g_{\max}$.
Then, there exists a sufficiently small constant $c>0$ such that
for all learning rates $\alpha<c/(\rho d^{2}\log d)$, we have 
\begin{equation}
\Pr(g_{i}(X_{t+1})1_{\fE_{t}}>g_{\max})=\Pr(\fE_{t}\cap\{g_{i}(X_{t+1})>g_{\max}\})\le\frac{1}{2d^{11}}.\label{eq:PrG}
\end{equation}
\end{lem}

\begin{proof}
Let $\sigma^{2}=\sum_{\tau=1}^{t}\Var[G_{i\tau}|X_{\tau-1},\dots,X_{0}]$
and let $R$ satisfy $\E[G_{i\tau}|X_{\tau-1},\dots,X_{0}]\le X_{\tau}+R$
almost surely for all $\tau\in\{1,2,\dots,t\}$. Recall via the standard
Azuma--Bernstein concetration inequality for supermartingales that
\[
\Pr\left(G_{it}\ge G_{i0}+s\right)\le\exp\left(-\frac{s^{2}/2}{\sigma^{2}+Rs/3}\right).
\]
Equivalently, there exists a large enough constant $C>0$ such that
the following is true
\begin{gather*}
\Pr\left(g_{i}(X_{t+1})1_{\fE_{t}}\ge\frac{1}{2}g_{\max}+\left(1-\zeta\cdot\alpha\right)^{t}\left(g(X_{0})-\frac{1}{2}g_{\max}\right)+s'\right)\le\frac{1}{2d^{11}}\\
\text{where }s'=C\cdot\left(1-\zeta\cdot\alpha\right)^{t}\cdot\left[\sqrt{\sigma^{2}\log d}+R\log d\right].
\end{gather*}
Given that $g(X_{0})\le\frac{1}{2}g_{\max}$ holds by hypothesis,
the desired claim is true if we can show that $s'\le\frac{1}{2}g_{\max}$.
Crucially, we observe that the variance term in $s'$ does not blow-up
with time $t$
\begin{align*}
\left(1-\zeta\cdot\alpha\right)^{2t}\sigma^{2} & \le g_{\max}^{2}\cdot C_{\var}\cdot d^{2}\cdot\rho^{2}\cdot\alpha^{2}\sum_{\tau=1}^{t}\left(1-\zeta\cdot\alpha\right)^{2t-2\tau}\\
 & \le g_{\max}^{2}\cdot C_{\var}\cdot d^{2}\cdot\rho^{2}\cdot\alpha
\end{align*}
due to the geometric series expansion $\sum_{\tau=0}^{t}\beta^{t-\tau}=(1-\beta^{t+1})/(1-\beta).$
Substituting the deviations term, choosing a step-size $\alpha\le c/(\rho d^{2}\log d)$
for sufficiently small $c$ yields
\[
s'=O\left(\sqrt{\alpha\cdot d^{2}\cdot\rho^{2}\cdot g_{\max}^{2}\cdot\log d}\right)+O\left(\alpha\cdot d^{2}\cdot\rho\cdot g_{\max}\cdot\log d\right)=\frac{g_{\max}}{2}.
\]
\end{proof}
In summary, \lemref{fconc} requires a step-size of $\alpha\le c/(g_{\max}d^{2}\log d)$
to keep deviations on $f$ small, while \lemref{gconc} requires a
step-size of $\alpha\le c/(\rho d^{2}\log d)$ to keep deviations
on $g_{i}$ small. Therefore, it follows that a step-size $\alpha\le c/((g_{\max}+\rho)d^{2}\log d)$
will keep both deviations small.
\begin{proof}[Proof of \thmref{main}]
For a step-size $\alpha\le c/((g_{\max}+\rho)d^{2}\log d)$ with
sufficiently small $c>0$, both concentration bounds \lemref{fconc}
and \lemref{gconc} are valid. Combined, we take the trivial union
bound to determine the probability of \emph{failure} at the $(t+1)$-th
step, after succeeding after $t$ steps:
\begin{align*}
\Pr(\fE_{t}\cap\overline{\fE}_{t+1}) & =\sum_{i=1}^{d}\Pr(\fE_{t}\cap\{g_{i}(X_{t+1})\ge g_{\max}\})+\Pr(\fE_{t}\cap\{f(X_{t+1})\ge(1-\frac{\alpha}{2})^{t+1}f_{\max}).\\
 & \le d\cdot\frac{1}{2d^{11}}+\frac{1}{2d^{10}}=\frac{1}{d^{10}}.
\end{align*}
Here, $\overline{\fE}_{t+1}$ denotes the complement of $\fE_{t+1}$.
The probability of failure at the $T$-th step is then the cummulative
probability of failing at the $(t+1)$-th step, after succeeding after
$t$ steps, over all $t\le T$:
\[
\Pr(\overline{\fE}_{T})\le\sum_{t=1}^{T}\Pr(\fE_{t-1}\cap\overline{\fE}_{t})\le\frac{T}{d^{10}}
\]
and this proves that $\fE_{T}$ happens with probability $1-T/d^{10}$
as desired. 
\end{proof}

\subsection{\label{sec:martingales}Proofs of supermartingale deviations and
variances}

We will now verify the supermartingales and their deviations and variances
in detail. We first begin by proving the following bounds on the size
of the stochastic gradient.
\begin{lem}
Let $X,Z\in\R^{n\times r}$ satisfy $\|XX^{T}-ZZ^{T}\|_{F}\le\rho\cdot\lambda_{\min}(Z^{T}Z)$
with $\rho<1/2$ and $\max_{i}e_{i}^{T}X(X^{T}X)^{-1}Xe_{i}\le g_{\max}$
and $\max_{i}e_{i}^{T}Z(Z^{T}Z)^{-1}Ze_{i}\le g_{\max}$. Then, with
respect to the randomness of the following
\[
SG(X)=2d^{2}\cdot(XX^{T}-ZZ^{T})_{i,j}\cdot(e_{i}e_{j}^{T}+e_{j}e_{i}^{T})X
\]
where $(i,j)\sim[d]^{2}$ is selected uniformly at random, we have:
\begin{enumerate}
\item $\|SG(X)\|_{X}^{*}\le8d^{2}\cdot g_{\max}^{1/2}\cdot\|XX^{T}-ZZ^{T}\|_{F}$.
\item $\|SG(X)(X^{T}X)^{-1}\|_{X}^{*}\le16d^{2}\cdot g_{\max}^{1/2}\cdot\rho$.
\item $\E(\|SG(X)\|_{X}^{*})^{p}\le2^{2p}\cdot d^{2(p-1)}\cdot g_{\max}^{p/2}\cdot\|XX^{T}-ZZ^{T}\|_{F}^{p}$.
\item $\E(\|SG(X)(X^{T}X)^{-1}\|_{X}^{*})^{p}\le2^{3p}\cdot d^{2(p-1)}\cdot g_{\max}^{p/2}\cdot\rho^{p}$.
\end{enumerate}
\end{lem}

\begin{proof}
Let us write $E=XX^{T}-ZZ^{T}$. To prove (i) we have
\begin{gather*}
\|SG(X)\|_{X}^{*}=2d^{2}\cdot E_{i,j}\cdot(\|(e_{i}e_{j}^{T}+e_{j}e_{i}^{T})X\|_{X}^{*})\le4d^{2}\cdot E_{i,j}\cdot\max_{i}\|e_{i}^{T}X\|_{X}^{*}
\end{gather*}
and if we write $Q_{X}=X(X^{T}X)^{-1/2}$ and $Q_{Z}=Z(Z^{T}Z)^{-1/2}$
we have 
\begin{align*}
E_{i,j}=e_{i}^{T}Ee_{j} & =e_{i}^{T}Q_{X}Q_{X}^{T}(XX^{T}-ZZ^{T})e_{j}-e_{i}^{T}(I-Q_{X}Q_{X}^{T})ZZe_{j}\\
 & \le\|e_{k}^{T}Q_{X}\|\|Q_{X}^{T}(XX^{T}-ZZ^{T})e_{j}\|+\|e_{i}^{T}(I-Q_{X}Q_{X}^{T})ZZ^{T}Q_{Z}\|\|Q_{Z}^{T}e_{j}\|,\\
 & \le g_{\max}^{1/2}\cdot\|XX^{T}-ZZ^{T}\|_{F}+\|XX^{T}-ZZ^{T}\|_{F}\cdot g_{\max}^{1/2}.
\end{align*}
we use the fact that $(I-Q_{X}Q_{X}^{T})(XX^{T}-ZZ^{T})=-(I-Q_{X}Q_{X}^{T})ZZ^{T}$
in the first and last lines. To prove (ii) we have
\[
\|SG(X)(X^{T}X)^{-1}\|_{X}^{*}=\frac{\|SG(X)\|_{X}^{*}}{\lambda_{\min}(X^{T}X)}\le\frac{8d^{2}\cdot g_{\max}^{1/2}\cdot\rho\cdot\lambda_{\min}(Z^{T}Z)}{(1-\rho)\cdot\lambda_{\min}(Z^{T}Z)}\le16d^{2}\cdot g_{\max}^{1/2}\cdot\rho
\]
where we used Weyl's inequality $\lambda_{r}(XX^{T})\ge\lambda_{r}(ZZ^{T})-\|XX^{T}-ZZ^{T}\|_{F}$.
To prove (iii) we have
\begin{align*}
\E(\|SG(X)\|_{X}^{*})^{p} & =\frac{1}{d^{2}}\sum_{i,j}d^{2p}\cdot(2E_{i,j})^{p}\cdot(\|(e_{i}e_{j}^{T}+e_{j}e_{i}^{T})X\|_{X}^{*})^{p}\le2^{2p}\cdot d^{2(p-1)}\cdot\|E\|_{F}^{p}\cdot\max_{i}(\|e_{i}^{T}X\|_{X}^{*})^{p}
\end{align*}
where we used $(\sum_{i}x_{i}^{2})^{1/2}\ge(\sum_{i}x_{i}^{p})^{1/p}$
for any $p\ge2$. The proof of (iv) follows identically by applying
the proof of (ii) to the proof of (iii).
\end{proof}
We now prove the properties of the function value supermartingale
$F_{t}$.
\begin{proof}[Proof of \lemref{fmartin}]
Conditioning on the current iterate $X_{t}$ and the event $\fE_{t}$,
the new iterate $X_{t+1}=X_{t}-\alpha SG(X_{t})(X_{t}^{T}X_{t})^{-1}$
has expectation
\[
\E[f(X_{t+1})]\le f(X_{t})-\alpha\inner{\nabla f(X_{t})}{\E[SG(X_{t})(X_{t}^{T}X_{t})^{-1}]}+\frac{L_{X}}{2}\alpha^{2}\cdot\E[(\|SG(X_{t})\|_{X}^{*})^{2}]
\]
with $L_{X}=O(1)$ by evoking \lemref{fdecr} noting that $\|\alpha SG(X_{t})\|_{X}^{*}=O(1)\cdot\sqrt{f(X_{t})}$
for the step-size $\alpha\le c/(g_{\max}d^{2}\log d)$, since 
\begin{align*}
\|\alpha SG(X_{t})\|_{X}^{*} & =\alpha\cdot2d^{2}\cdot(X_{t}X_{t}^{T}-ZZ^{T})_{i,j}\|(e_{i}e_{j}^{T}+e_{j}e_{i}^{T})X_{t}\|_{X}^{*}\\
 & \le\alpha\cdot2d^{2}\cdot|X_{t}X_{t}^{T}-ZZ^{T}|_{\infty}\cdot2\sqrt{g_{\max}}\\
 & \le\frac{c}{g_{\max}d^{2}\log d}\cdot2d^{2}\cdot\sqrt{f(X_{t})g_{\max}}\cdot2\sqrt{g_{\max}}=\frac{4c}{\log d}\sqrt{f(X_{t})}.
\end{align*}
The linear term evaluates simply as $\E[SG(X)(X^{T}X)^{-1}]=\nabla f(X)(X^{T}X)^{-1}$,
while the quadratic term evaluates
\begin{align*}
\E[(\|SG(X)\|_{X}^{*})^{2}] & =\frac{1}{d^{2}}\sum_{i,j}4d^{4}\cdot(XX^{T}-ZZ^{T})_{i,j}^{2}(\|(e_{i}e_{j}^{T}+e_{j}e_{i}^{T})X\|_{X}^{*})^{2}\\
 & \le\sum_{i,j}4d^{2}\cdot(XX^{T}-ZZ^{T})_{i,j}^{2}\cdot4g_{\max}=16\cdot g_{\max}\cdot d^{2}\cdot f(X)
\end{align*}
Combined, substituting $(\|\nabla f(X)\|_{X}^{*})^{2}\ge13\cdot f(X)$,
it follows that we have geometric convergence
\begin{align*}
\E[f(X_{t+1})] & \le f(X_{t})-\alpha\inner{\nabla f(X_{t})}{\E[SG(X_{t})(X_{t}^{T}X_{t})^{-1}]}+\frac{L_{X}}{2}\alpha^{2}\cdot\E[(\|SG(X_{t})\|_{X}^{*})^{2}]\\
 & \le\left(1-2\alpha\right)f(X_{t})+L_{X}\cdot\alpha^{2}\cdot8\cdot g_{\max}\cdot d^{2}\cdot f(X_{t})\le\left(1-\alpha\right)f(X_{t})
\end{align*}
where we observe that we can pick a small enough constant $c$ in
the step-size $\alpha\le c/(g_{\max}d^{2}\log d)$ so that
\[
L_{X}\cdot\alpha^{2}\cdot8\cdot g_{\max}\cdot d^{2}\cdot f(X_{t})=\frac{c\cdot L_{X}\cdot8\cdot g_{\max}\cdot d^{2}}{g_{\max}d^{2}\log d}\alpha f(X_{t})\le\alpha f(X_{t}).
\]
Now, to confirm that $F_{t}$ is a martingale, it remains to see that
\begin{align*}
\E[F_{t+1}|X_{t}] & =\left(1-\alpha\right)^{-(t+1)}\E[f(X_{t+1})\cdot1_{\fE_{t}}|X_{t}]\leq\left(1-\alpha\right)^{-(t+1)}\left(1-\alpha\right)f(X_{t})1_{\fE_{t}}\\
 & \leq\left(1-\alpha\right)^{-t}f(X_{t})1_{\fE_{t-1}}=F_{t},
\end{align*}
where the last inequality follows from $1_{\fE_{t}}\leq1_{\fE_{t-1}}$. 

We now bound the deviations on $F_{t}$. Conditioning on the previous
iterates $X_{t},\dots,X_{0}$, we obseve that the $f(X_{t})$ terms
cancel:
\begin{gather}
f(X_{t+1})\cdot1_{\fE_{t}}-\E[f(X_{t+1})\cdot1_{\fE_{t}}]\le\left[-\alpha\langle\nabla f(X_{t}),[SG(X_{t})-\E SG(X_{t})](X_{t}^{T}X_{t})^{-1}\rangle\right.\nonumber \\
\left.+\frac{\alpha^{2}\cdot L_{X}}{2}(\|SG(X_{t})\|_{X}^{*})^{2}+\E(\|SG(X_{t})\|_{X}^{*})^{2}\right]\cdot1_{\fE_{t}}.\label{eq:expanF}
\end{gather}
Here we have for the linear term
\begin{align*}
\inner{\nabla f(X_{t})}{SG(X_{t})(X_{t}^{T}X_{t})^{-1}}\cdot1_{\fE_{t}} & \le\|\nabla f(X_{t})\|_{X}^{*}\|SG(X_{t})\|_{X}^{*}\cdot1_{\fE_{t}}\\
 & \le4\sqrt{f(X_{t})}\cdot4d^{2}\sqrt{f(X_{t})}\cdot g_{\max}\cdot1_{\fE_{t}}=O(d^{2}g_{\max})f(X_{t})\cdot1_{\fE_{t}}
\end{align*}
and the quadratic term
\[
(\|SG(X_{t})\|_{X}^{*})^{2}\cdot1_{\fE_{t}}\le d^{4}\cdot f(X_{t})\cdot g_{\max}^{2}\cdot1_{\fE_{t}}=O(d^{4}g_{\max}^{2})f(X_{t})\cdot1_{\fE_{t}}.
\]
Therefore, using the maximum value to bound the expectation, we have
\begin{align*}
F_{t+1}-\E[F_{t+1}|X_{t},\dots,X_{0}] & \le\alpha\left(1-\alpha\right)^{-t}\cdot\left[O(d^{2}g_{\max})f(X_{t})+\alpha O(d^{4}g_{\max}^{2})f(X_{t})\right]\cdot1_{\fE_{t}}\\
 & \le\alpha\left(1-\alpha\right)^{-t}\cdot O(d^{2}g_{\max})f(X_{t})\cdot1_{\fE_{t}}\\
 & \le C_{\dev}\cdot\alpha\left(1-\alpha\right)^{-t}\left(1-\frac{\alpha}{2}\right)^{t}f_{\max}\cdot d^{2}\cdot g_{\max}\cdot1_{\fE_{t}}
\end{align*}
where again we observe that a step-size like $\alpha\le c/(g_{\max}d^{2}\log d)=O(d^{-2}g_{\max}^{-1})$
yields the cancellation of exponents $\alpha\cdot O(d^{4}g_{\max}^{2})=O(d^{2}g_{\max})$. 

Finally, we bound the variance. Conditioned on all previous iterates
$X_{t},\dots,X_{0}$ we have 
\begin{gather*}
\Var(\langle\nabla f(X_{t}),SG(X_{t})(X_{t}^{T}X_{t})^{-1}\rangle\cdot1_{\fE_{t}})\le\E[\langle\nabla f(X_{t}),SG(X_{t})(X^{T}_{t}X_{t})^{-1}\rangle^{2}\cdot1_{\fE_{t}}]\\
\leq~(\|\nabla f(X_{t})\|_{X_{t}}^{*})^{2}\cdot\E[(\|SG(X_{t})\|_{X}^{*})^{2}]\cdot1_{\fE_{t}}\leq O(d^{2}g_{\max})\cdot f(X_{t})^{2}\cdot1_{\fE_{t}},
\end{gather*}
and also 
\begin{align*}
 & \Var((\|SG(X_{t})\|_{X}^{*})^{2}\cdot1_{\fE_{t}})\le\E[(\|SG(X_{t})\|_{X}^{*})^{4}\cdot1_{\fE_{t}}]=O(d^{6}g_{\max}^{3})f(X_{t})^{2}\cdot1_{\fE_{t}}.
\end{align*}
By the same expansion in (\ref{eq:expanF}) we have 
\begin{align*}
\Var(F_{t+1}|X_{t},\dots,X_{0}) & \leq\alpha^{2}\left(1-\alpha\right)^{-2t}\cdot\left[O(d^{2}g_{\max})\cdot f(X_{t})^{2}+\alpha^{2}O(d^{6}g_{\max}^{3})f(X_{t})^{2}\right]\cdot1_{\fE_{t}}\\
 & \leq\alpha^{2}\left(1-\alpha\right)^{-2t}\cdot O(d^{2}g_{\max})\cdot f(X_{t})^{2}\cdot1_{\fE_{t}}\\
 & \le C_{\var}\cdot\alpha^{2}\left(1-\alpha\right)^{-2t}\left(1-\frac{\alpha}{2}\right)^{2t}f_{\max}^{2}\cdot d^{2}\cdot g_{\max}\cdot1_{\fE_{t}}
\end{align*}
where again we observe that a step-size like $\alpha\le c/(g_{\max}d^{2}\log d)=O(d^{-2}g_{\max}^{-1})$
yields the cancellation of exponents $\alpha^{2}\cdot(d^{6}g_{\max}^{3})=O(d^{2}g_{\max})$. 
\end{proof}
We now prove properties of the incoherence martingale.
\begin{proof}[Proof of \lemref{gmartin}]
Conditioning on $X_{t}$ and the event $\fE_{t}$, we have for $V=SG(X_{t})(X^{T}_{t}X_{t})^{-1}$
\begin{align*}
\E[g_{i}(X_{t+1})] & \le g_{k}(X_{t})-\alpha\inner{\nabla g_{i}(X_{t})}{\E[V]}+\alpha^{2}\cdot\E\left[\frac{(\|V\|_{X}^{*})^{2}}{1-2\|V\|_{X}^{*}}\right]\\
 & \le\left(1-\frac{1-2\rho}{1-\rho}\alpha\right)g_{i}(X_{t})+\alpha\cdot\frac{1}{1-\rho}\cdot\sqrt{g_{i}(X)g_{i}(Z)}+\alpha^{2}\cdot\frac{\E\left[(\|V\|_{X}^{*})^{2}\right]}{1-2\|V\|_{X}^{*}}\\
 & \le\left(1-\frac{1-2\rho}{1-\rho}\alpha\right)g_{i}(X_{t})+\alpha\cdot\frac{\sqrt{g_{i}(Z)/g_{\max}}}{1-\rho}\cdot g_{\max}+\alpha^{2}\cdot\frac{O(d^{2}\cdot g_{\max}\cdot\rho^{2})}{1-O(g_{\max}^{1/2}\cdot\rho)}\\
 & \le\left(1-\zeta\alpha\right)g_{i}(X_{t})+\alpha\cdot\frac{\zeta}{2}g_{\max}\qquad\text{ for }\alpha=O(\rho^{-1}d^{-2}).
\end{align*}
Here we note that we have carefully chosen $g_{\max}$ so that the
ratio $\max_{i}g_{i}(Z)/g_{\max}=[(1-\rho)\cdot\zeta/4]^{2}.$ It
then follows that the following is a supermartingale
\[
G_{it}=\left(1-\zeta\alpha\right)^{-t}\left(g_{i}(X_{t})\cdot1_{\fE_{t-1}}-\frac{\zeta}{2}g_{\max}\right).
\]
Indeed, we have 
\begin{align*}
\E[G_{i(t+1)}|X_{t},\dots,X_{0}] & =\left(1-\zeta\alpha\right)^{-(t+1)}\left(\E[g_{i}(X_{t+1})\cdot1_{\fE_{t}}|X_{t},\dots,X_{0}]-\frac{\zeta}{2}g_{\max}\right)\\
 & \le\left(1-\zeta\alpha\right)^{-(t+1)}\left[\left(1-\zeta\alpha\right)g_{i}(X_{t})\cdot1_{\fE_{t}}+\alpha\cdot\frac{\zeta}{2}g_{\max}\cdot1_{\fE_{t}}-\frac{\zeta}{2}g_{\max}\right]\\
 & \le\left(1-\zeta\alpha\right)^{-t}\left[g_{i}(X_{t})\cdot1_{\fE_{t-1}}-\left(\frac{1-\alpha}{1-\alpha}\right)\cdot\frac{\zeta}{2}g_{\max}\right]=G_{it}
\end{align*}
where the final line uses $1_{\fE_{t}}\le1_{\fE_{t-1}}\le1$. 

We now bound the deviations on $G_{it}$. Conditioning on the previous
iterates $X_{t},\dots,X_{0}$, we obseve that the $g_{i}(X_{t})$
terms cancel:
\begin{multline}
g_{i}(X_{t+1})\cdot1_{\fE_{t}}-\E[g_{i}(X_{t+1})\cdot1_{\fE_{t}}]\le\left[-\alpha\langle\nabla g_{i}(X_{t}),[SG(X_{t})-\E SG(X_{t})](X_{t}^{T}X_{t})^{-1}\rangle\right.\\
\left.+5\alpha^{2}\cdot\frac{(\|SG(X_{t})(X^{T}_{t}X_{t})^{-1}\|_{X}^{*})^{2}+\E(\|SG(X_{t})(X^{T}_{t}X_{t})^{-1}\|_{X}^{*})^{2}}{1-\alpha\|SG(X_{t})(X^{T}_{t}X_{t})^{-1}\|_{X}^{*}}\right]\cdot1_{\fE_{t}}.\label{eq:expanG}
\end{multline}
Here we have for the linear term
\[
\|\nabla g(X)\|\cdot1_{\fE_{t}}\le\|2[I-X(X^{T}X)^{-1}X^{T}]e_{i}e_{i}^{T}X(X^{T}X)^{-1}\|\cdot1_{\fE_{t}}\le O(\kappa^{1/2}\cdot g_{\max}^{1/2})
\]
and $\|SG(X)\|_{X}^{*}\cdot1_{\fE_{t}}\le O(d^{2}\cdot\sqrt{g_{\max}f(X)})=O(d^{2}\sqrt{g_{\max}}\rho/\kappa)$
noting that $f_{\max}=\rho^{2}/\kappa^{2}$ and hence 
\begin{align*}
|\inner{\nabla g(X_{t})}{SG(X_{t})(X^{T}_{t}X_{t})^{-1}}| & \le\|\nabla g(X_{t})\|\cdot\|SG(X_{t})\|_{X}^{*}\\
 & \le O(\sqrt{\kappa g_{\max}})\cdot O(d^{2}\sqrt{g_{\max}}\rho/\kappa)=O(g_{\max}\cdot d^{2}\rho)\cdot1_{\fE_{t}}.
\end{align*}
We have for the quadratic term 
\[
\frac{5(\|SG(X_{t})(X^{T}_{t}X_{t})^{-1}\|_{X}^{*})^{2}}{1-\alpha\|SG(X_{t})(X^{T}X)^{-1}\|_{X}^{*}}\cdot1_{\fE_{t}}\le\frac{O(d^{4}\cdot g_{\max}\cdot\rho^{2})}{1-\alpha\cdot O(d^{2}\cdot g_{\max}^{1/2}\cdot\rho)}\cdot1_{\fE_{t}}=O(g_{\max}\cdot d^{4}\rho^{2})\cdot1_{\fE_{t}}.
\]
Therefore, using the maximum value to bound the expectation, we have
\begin{align*}
G_{i(t+1)}-\E[G_{i(t+1)}|X_{t},\dots,X_{0}] & \le\alpha\left(1-\alpha\right)^{-t}g_{\max}\cdot\left[O(d^{2}\rho)+\alpha O(d^{4}\rho^{2})f(X_{t})\right]\cdot1_{\fE_{t}}\\
 & \le\alpha\left(1-\alpha\right)^{-t}g_{\max}\cdot O(d^{2}\rho)\cdot1_{\fE_{t}}
\end{align*}
where again we observe that a step-size like $\alpha\le c/(\rho d^{2}\log d)=O(d^{-2}\rho^{-1})$
yields the cancellation of exponents $\alpha\cdot O(d^{4}\rho^{2})=O(d^{2}\rho)$. 

Finally, we bound the variance. Conditioned on all previous iterates
$X_{t},\dots,X_{0}$ we have 
\begin{gather*}
\Var(\langle\nabla g_{i}(X_{t}),SG(X_{t})(X^{T}_{t}X_{t})^{-1}\rangle\cdot1_{\fE_{t}})\le\E[\langle\nabla g_{i}(X_{t}),SG(X_{t})(X^{T}_{t}X_{t})^{-1}\rangle^{2}\cdot1_{\fE_{t}}]\\
\leq~(\|\nabla g_{i}(X_{t})\|)^{2}\cdot\E[(\|SG(X_{t})\|_{X}^{*})^{2}]\cdot1_{\fE_{t}}\leq O(g_{\max}^{2}\cdot d^{2}\rho^{2})\cdot1_{\fE_{t}},
\end{gather*}
and also 
\begin{gather*}
\Var\left(\frac{5(\|SG(X_{t})(X^{T}_{t}X_{t})^{-1}\|_{X}^{*})^{2}}{1-\alpha\|SG(X_{t})(X^{T}_{t}X_{t})^{-1}\|_{X}^{*}}\cdot1_{\fE_{t}}\right)\le\frac{25\cdot\E[(\|SG(X_{t})(X^{T}_{t}X_{t})^{-1}\|_{X}^{*})^{4}]}{(1-\alpha\|SG(X_{t})(X^{T}_{t}X_{t})^{-1}\|_{X}^{*})^{2}}\cdot1_{\fE_{t}}\\
\le\frac{O(g_{\max}^{2}\cdot d^{6}\rho^{4})}{1-\alpha\cdot O(d^{2}\cdot g_{\max}^{1/2}\cdot\rho)}\cdot1_{\fE_{t}}=O(g_{\max}^{2}\cdot d^{6}\rho^{4})\cdot1_{\fE_{t}}.
\end{gather*}
By the same expansion in (\ref{eq:expanG}) we have 
\begin{align*}
\Var(G_{i(t+1)}|X_{t},\dots,X_{0}) & \leq\alpha^{2}\left(1-\alpha\right)^{-2t}\cdot\left[O(g_{\max}^{2}\cdot d^{2}\rho^{2})\cdot1_{\fE_{t}}+\alpha^{2}O(g_{\max}^{2}\cdot d^{6}\rho^{4})\cdot1_{\fE_{t}},\right]\cdot1_{\fE_{t}}\\
 & \leq\alpha^{2}\left(1-\alpha\right)^{-2t}g_{\max}^{2}\cdot O(d^{2}\rho^{2})\cdot1_{\fE_{t}}
\end{align*}
where again we observe that a step-size like $\alpha\le c/(\rho d^{2}\log d)=O(d^{-2}\rho^{-1})$
yields the cancellation of exponents $\alpha^{2}\cdot(d^{6}\rho^{4})=O(d^{2}\rho^{2})$. 
\end{proof}

\end{document}